\documentclass[]{fairmeta}

\usepackage{xcolor}
\usepackage{graphicx}
\usepackage{natbib}
\usepackage{appendix}
\usepackage{booktabs}
\usepackage{url}
\usepackage{xcolor}
\usepackage{tcolorbox}
\usepackage{hyperref}
\definecolor{RoyalBlue}{RGB}{65,105,225}
\definecolor{Green}{RGB}{34,139,34}
\hypersetup{
	colorlinks=true,
	citecolor=RoyalBlue,
	linkcolor=RoyalBlue,
	filecolor=magenta,      
	urlcolor=RoyalBlue
}
\usepackage{cleveref}
\usepackage{tabularx}
\usepackage{comment}
\usepackage{amsthm}
\usepackage{tikz}
\usepackage[frozencache,cachedir=minted-cache]{minted}
\setminted{fontsize=\scriptsize}
\usepackage{amssymb}
\usepackage{wrapfig}

\usepackage{fontawesome}
\usepackage{pifont}

\usepackage{thmtools,thm-restate}

\usepackage{cleveref}
\usepackage{caption}
\usepackage{subcaption}
\usepackage{multicol}
\usepackage{multirow}

\newtheorem{theorem}{Theorem}
\newtheorem{lemma}{Lemma}
\newtheorem{assumption}{Assumption}
\newtheorem{definition}{Definition}

\newtheorem{constraint}{Constraint}

\DeclareMathAlphabet{\mathbsf}{OT1}{cmss}{bx}{n}
\DeclareMathAlphabet{\mathssf}{OT1}{cmss}{m}{sl}
\crefname{lemma}{lemma}{lemmas}
\Crefname{lemma}{Lemma}{Lemmas}
\crefname{thm}{theorem}{theorems}
\Crefname{thm}{Theorem}{Theorems}
\crefname{prop}{proposition}{propositions}
\Crefname{prop}{Proposition}{Propositions}
\crefname{defn}{definition}{definitions}
\Crefname{defn}{Definition}{Definitions}
\creflabelformat{equation}{#1#2#3}
\crefname{equation}{equation}{equations}
\Crefname{equation}{Equation}{Equations}
\Crefname{section}{Section}{Sections}
\Crefname{appendix}{Appendix}{Appendices}
\crefname{figure}{figure}{figures}
\Crefname{figure}{Figure}{Figures}
\crefname{algorithm}{algorithm}{algorithms}
\Crefname{algorithm}{Algorithm}{Algorithms}
\crefname{assumption}{assumption}{assumptions}
\Crefname{assumption}{Assumption}{Assumptions}

\title{Multi-Domain Causal Representation Learning via Weak Distributional Invariances}

\author[1,\dagger]{Kartik Ahuja}
\author[2, \dagger]{Amin Mansouri}
\author[3]{Yixin Wang}

\affiliation[1]{FAIR at Meta}
\affiliation[2]{Mila-Quebec AI Institute}
\affiliation[3]{University of Michigan}
\contribution[\dagger]{Equal contribution}

\abstract{Causal representation learning has emerged as the center of action in causal machine learning research. In particular, multi-domain datasets present a natural opportunity for showcasing the advantages of causal representation learning over standard unsupervised representation learning. While recent works have taken crucial steps towards learning causal representations, they often lack applicability to multi-domain datasets due to over-simplifying assumptions about the data; e.g. each domain comes from a different single-node perfect intervention. In this work, we relax these assumptions and capitalize on the following observation: there often exists a subset of latents whose certain distributional properties (e.g., support, variance) remain stable across domains; this property holds when, for example, each domain comes from a multi-node imperfect intervention. Leveraging this observation, we show that autoencoders that incorporate such invariances can provably identify the stable set of latents from the rest across different settings.}

\correspondence{Kartik Ahuja at \email{kartikahuja@meta.com}}

\begin{document}

\maketitle

\section{Introduction}

Despite the incredible success of modern AI systems, they possess limited reasoning and planning skills \citep{bubeck2023sparks} and often lack controllability \citep{leivada2023dall}. Towards alleviating these concerns, causal representation learning \citep{scholkopf2021toward} aims to build models with a better causal understanding of the world.  

The theory of causal representation learning to date has largely focused on developing algorithms that are capable of identifying the underlying causal structure of the data-generating process under minimal supervision. This capability is enabled by endowing these learners with inductive biases that capture natural properties of the data \citep{locatello2020weakly, brehmer2022weakly}. Despite the advances, existing causal representation learners remain far from readily applicable to the increasingly prevalent multi-domain datasets in practice \citep{gulrajani2020search, koh2021wilds}.
One wonders why? An important reason is that existing approaches rely on strong assumptions about the data-generating process. For example, many assume that the data in different domains is gathered under perfect interventions. 
Moreover, many also require that the relationships between the latents can be described by the same fixed directed acyclic graph (DAG) across all data points. This assumption is often violated: e.g. the causal relationships between the latents can have different causal directions in two images, where a cat chases a dog in one image and the dog chases the cat in the other. In this work, we relax these assumptions, making progress towards causal representation learning for complex multi-domain datasets.

\paragraph{Contributions.}  The invariance principle considered in this paper is reminiscent of the invariance principle in \cite{peters2016causal,arjovsky2019invariant}, though we focus on unlabelled multi-domain data. At a high-level, the principle requires that a fixed subset of latents is not intervened across domains, and their distributions remain invariant. We study different forms of distributional invariance, ranging from weak invariance on the support to strong invariance on the marginal distribution of the latents. We divide our analysis into two parts. We first focus on standard settings where the latents in the entire data are governed by a fixed acyclic structural causal model; we then relax this assumption. We also consider different assumptions on the mixing function that generates the observations. In our theoretical and empirical analysis, the identification results take the form ``\emph{latents with invariant distributional properties can be disentangled from the rest}.'' 

\section{Related Works}

The field of causal representation learning bears a deep connection to the field of independent component analysis (ICA) \citep{hyvarinen2023nonlinear}. The seminal work of \citet{comon1994independent} on linear independent component analysis studied linear mixing of independent non-Gaussian latents and proposed a method that identifies the true latents up to permutation and scaling. Since then, much progress has taken place.   Existing works in the area of representation identification can be categorized into the following categories based on the assumptions: i) assumptions on the distribution of latent factors, and ii) assumptions on the mixing functions.  In the pivotal work of \citet{khemakhem2020variational}, the authors studied general diffeomorphisms mixing but made additional assumptions such as the availability of auxiliary information that renders latents conditionally independent. Recently, \citet{kivva2022identifiability} considered a setup similar to \citet{khemakhem2020variational}; they relaxed the crucial assumption that auxiliary information is observed but restricted the family of mixing maps to piecewise linear diffeomorphisms, in order to obtain a similar level of identification as \citet{khemakhem2020variational}. The recent work of \citet{liang2023causal} takes the connection between causal representation learning and ICA one step further. They study the question of identifiability under the supposition that the underlying causal graph is known, much in the same spirit that ICA supposes the graph is known and all latent variables are independent. 

More recently, the problem of interventional causal representation learning has come to attention in \cite{ahuja2022interventional, seigal2022linear,varici2023score}. \cite{ahuja2022interventional} study a) polynomial mixing with interventions that induce independent support; b) general diffeomorphisms with hard do interventions. \cite{seigal2022linear} study linear mixing with perfect interventions,  and \cite{varici2023score} study linear mixing with perfect and imperfect interventions.  The relatively recent work of \cite{von2023nonparametric} studied general diffeomorphism mixing with perfect interventions, and \cite{buchholz2023learning} studied general diffeomorphisms with latents that follow linear Gaussian structural causal model (SCM)  under both perfect and imperfect interventions. The different identification guarantees in these works are summarized in Table \ref{table_synth}, where we also contrast our results. There are a few aspects that separate us from existing works. Firstly, these works study single-node interventions and we study multi-node imperfect interventions. We also study the setting where a fixed DAG does not explain the relationships between the latents for the entire observational dataset.  Another close line of work focuses on the intermediate goal of learning the underlying latent causal graph. Some examples in this line of work include \citet{cai2019triad, xie2020generalized, jiang2023learning}
and a concurrent work \citep{zhang2023identifiability}.

Aside from the above works, other causal representation learning settings that have been studied include settings where the learner has access to i) paired observations (e.g., data generated pre- and post-intervention) \citep{locatello2020weakly, lachapelle2022disentanglement, ahuja2022weakly, lippe2022citris, lippe2022icitris, von2021self}, ii) temporal data \citep{hyvarinen2019nonlinear, yao2022learning1, lachapelle2022partial, ahuja2021properties}, iii) multi-view data \citep{gresele2020incomplete} iv) other forms of auxiliary information \citep{khemakhem2020variational, khemakhem2020ice, hyvarinen2019nonlinear}, v) object-centric inductive biases \citep{mansouri2022object, lachapelle2023additive, brady2023provably}. These settings are qualitatively different from ours.

Lastly, the distributional invariances used in our work may remind the readers of the seminal works of~\citet{ganin2016domain, muandet2013domain}. There are a few notable differences: i) these works focus on domain generalization in the presence of labeled data, while we focus on the unsupervised setting, ii) these works enforce invariance of the joint distribution of all the latents, while we enforce a weaker invariance on a subset of the latents.

\begin{table*}
\footnotesize
	\caption{\small{Our results compared with related works. Existing works assume that the relationship between latents can be described by a fixed DAG across domains. We relax this assumption to work with general multi-domain settings.} }
	\label{table_synth}
	\renewcommand{\arraystretch}{1.1}
	\centering
	\begin{tabular}{llll}
		\toprule
		\textbf{Input data}  & \textbf{Assm.  on}  $p_Z$   &  \textbf{Assm. on} $g$     &   \textbf{Identification}                     \\  \midrule
                 \textcolor{gray}{Observational} & \textcolor{gray}{$z_i \perp z_j | u$, $u$ aux info.}  & \textcolor{gray}{Diffeomorphism} & \textcolor{gray}{Perm \& scale (Khemakhem et al.) } \\ 
               Multi $do$ intvn/node  & Non-parametric & Diffeomorphism & $\approx$ Comp-wise (Ahuja et al.)\\
               Perfect ($1$-node)  &  Linear &  Linear & Comp-wise  (Seigal et al.) \\
               Perfect ($1$-node)  &  Non-parametric &  Polynomial & Comp-wise  (Ahuja et al.) \\ 
               Perfect ($1$-node)  & Non-parametric & Diffeomorphism & Comp-wise (Kugelgen et al.)  \\  
               Imperfect ($1$-node) & Non-parametric & Linear & Mix  consistency (Varici et al.)\\
               Imperfect ($1$-node) &  Non-parametric + ind support & Polynomial & Block affine (Ahuja et al.) \\ 
               Imperfect ($1$-node) & Linear Gaussian & Diffeomorphism & Affine (Buchholz et al.) \\ 
               Imperfect (multi-node) & Non-linear & Polynomial & Block affine (Theorem \ref{thm3}) \\  
             General multi-domain   & Non-param, sup inv $\mathcal{S}$ & Polynomial & Block affine (Theorem \ref{thm4}) \\
             General multi-domain & Non-param, sup inv $\mathcal{S}$& Diffeomorphism & $\Gamma^{c}$ identification (Theorem \ref{thm5}) \\
                \textcolor{gray}{Counterfactual } & \textcolor{gray}{Non-parametric} & \textcolor{gray}{Diffeomorphism} & \textcolor{gray}{Comp-wise (Brehmer et al.)} \\ 
		\bottomrule
	\end{tabular}
 \label{table1}
\end{table*}

\section{Unsupervised Multi-Domain Causal Representation Learning}

\paragraph{Problem statement.}

We are given unlabelled data--- a set of $x$'s (e.g., images)---from multiple domains. Consider a domain $j\in [k]$, where  $k$ is the number of domains, $[k]$ is shorthand for $\{1, \cdots, k\}$. The latent variables $z \in \mathbb{R}^{d}$ in domain $j$ are sampled from a distribution $p_Z^{(j)}$ whose support is denoted as $\mathcal{Z}^{(j)}$. These sampled latents $z$ are then rendered by an injective mixing function $g:\mathbb{R}^{d}\rightarrow \mathbb{R}^{n}$ to generate $x \in \mathbb{R}^{n}$. The support of the corresponding $x$'s in domain $j$ is denoted as $\mathcal{X}^{(j)}$.  Define the union of the support of the latents across domains as $\mathcal{Z} = \cup_{j \in [k]} \mathcal{Z}^{(j)}$ and correspondingly for the observations $x$'s as $\mathcal{X}  = \cup_{j \in [k]}\mathcal{X}^{(j)}$.  The data-generating process (DGP) is formally stated below. In each domain $j \in [k],$
\begin{equation} 
   z \sim p_Z^{(j)},  \;  x \leftarrow g(z) \qquad 
    \label{eqn: dgp}
\end{equation}
The goal of causal representation learning is \textit{provable
representation identification}, i.e. to learn an encoder function that can take in the observation $x$ and provably output its underlying true latent $z$ (or its desirable approximation). In practice, such an encoder is often learned via solving a reconstruction identity,  $h \circ f(x) = x, \forall x\in\mathcal{X},$ where $f:\mathbb{R}^{n} \rightarrow
\mathbb{R}^{d}$ and $h:\mathbb{R}^{d}\rightarrow
\mathbb{R}^{n}$ are a pair of encoder and decoder that
jointly satisfy the reconstruction identity. 
The pair $(f,h)$ together is referred to as the autoencoder. Given the
learned encoder $f$, the resulting representation is $\hat{z}
\triangleq f(x)$, which holds the encoder's estimate of the latents.   A common goal in causal representation learning is to achieve \textit{component-wise} disentanglement, i.e., each $\hat{z}_i$ is a scalar and invertible function of some $z_j$, where $\hat{z}_i$ and $z_j$ are $i^{th}$ and $j^{th}$ components of $\hat{z}$ and $z$.

\paragraph{Invariance principle for causal representations.}  The invariance principle we consider here is inspired by the folklore cow-on-the-beach example \citep{beery2018recognition}. The distributional properties of a certain set of latents (e.g., the alphabets across domains as shown in Figure~\ref{fig1}, or the cow characteristics across domains) are stable. In contrast, the distribution properties of the other latents (e.g. color characteristics in Figure~\ref{fig1}) are unstable; they vary across domains. More concretely, we divide the different components of latent $z$ into two sets, $\mathcal{S}$ and $\mathcal{U}$, where $\mathcal{S}$ corresponds to the stable set of latents and $\mathcal{U}$ corresponds to the unstable set of latents, and without loss of generality we write $z=[z_{\mathcal{S}},z_{\mathcal{U}}]$.  We require that some aspect of the joint distribution of $\mathcal{S}$---denoted as $p_{z_{\mathcal{S}}}^{(j)}$---does not vary across domains. Formally, there exists a functional $F$ such that $F\big[p_{z_{\mathcal{S}}}^{(j)}\big]$ is invariant across $j$. If $F[\cdot]$ is the identity functional, then the distribution itself is invariant. Other examples of $F[\cdot]$ include the support of the latents' distributions, the mean of the latents, the variance of the latents, etc. To realize this invariance principle in causal representation learning, we study autoencoders that enforce similar invariance on a certain subset $\hat{\mathcal{S}}\subseteq [d] $ of its estimated latents $\hat{z}$: 
    \begin{align}
    h \circ f(x) = x, \qquad &\forall x\in\mathcal{X}; \label{eqn: recons}\\
    F\big[p_{\hat{z}_{\hat{\mathcal{S}}}}^{(p)}\big] = F\big[p_{\hat{z}_{\hat{\mathcal{S}}}}^{(q)}\big], \qquad&\forall p \not= q, p,q \in [k]. \label{eqn: const}
    \end{align}
In what follows, we will show how autoencoders equipped with this class of invariance constraints can learn to disentangle the stable latents from the unstable latents: they return representations $\hat{z}$ that can provably satisfy $\hat{z}_{\hat{\mathcal{S}}} = u(z_{\mathcal{S}})$, where $u(\cdot)$ is an injective map.   For some choice of $\hat{\mathcal{S}}$, a solution to the reconstruction identity under invariance constraint may not exist. The learner can select $\hat{\mathcal{S}}$ as follows. It can start with the largest possible $\hat{\mathcal{S}}$, i.e. a set of size $d$. It reduces the size of the set by one until a solution to the reconstruction identity under invariance constraint is found, which is guaranteed to occur when $|\hat{\mathcal{S}|} = |\mathcal{S}|$. 

\begin{figure}
    \centering
    \includegraphics[ width=2.5in]{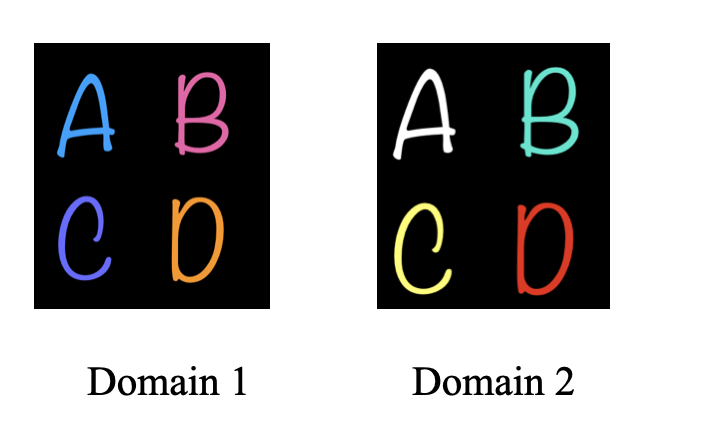}
    \caption{The distribution of the alphabet styles is stable across the domains but the distribution of color is unstable.}
    \label{fig1}
\end{figure}

\subsection{Acyclic Structural Causal Models $p_z$}

We start with the setting where the distribution of the latents $p_z$ comes from an acyclic causal model. To identify the stable latents, we first leverage previous results to achieve affine identification of all latents. We then use distributional invariance to achieve the identification of the stable latents. Let us now revisit a result from \citet{ahuja2022interventional} for affine identification under a polynomial mixing $g$.

\begin{assumption} (Polynomial mixing)
\label{assm1: dgp1} The interior of the support of $z$, denoted as $\mathcal{Z}$, is a non-empty subset of $\mathbb{R}^{d}$. The mixing map $g$ is a polynomial of finite degree $p$ whose corresponding coefficient matrix $G$ has full column rank. Specifically,  $g$ is determined by the coefficient matrix $G$ as follows,
 $$ g(z) = G[1, z, z\bar{\otimes} z, \cdots, \underbrace{z \bar{\otimes} \cdots \bar{\otimes}\;z}_{p \;\text{times}}]^{\top} \qquad \forall z\in \mathbb{R}^d,$$
where $ \bar{\otimes} $ represents the Kronecker product  with all distinct entries; for example, if $z = [z_1, z_2]$, then~$z\bar{\otimes} z = [z_1^2, z_1z_2, z_2^2]$. 
\end{assumption}

\begin{constraint} (Polynomial decoder)
\label{assm3: h_poly_new} 
The learned decoder $h$ is a polynomial of degree $p$ that is determined by its
corresponding coefficient matrix $H$ as follows, $$  h(z) = H[1, z, z\bar{\otimes} z, \cdots, \underbrace{z \bar{\otimes} \cdots \bar{\otimes}\;z}_{p \;\text{times}}]^{\top}\qquad \forall z\in \mathbb{R}^d.$$
Moreover, the interior of the image of the encoder $ f(\mathcal{X})$ is a non-empty subset of $\mathbb{R}^d$.
\end{constraint}

\begin{theorem}[\citet{ahuja2022interventional}]
\label{thm1}
Suppose the multi-domain data is gathered from the DGP in equation~\eqref{eqn: dgp}  under Assumptions~\ref{assm1: dgp1}. Then the autoencoder that solves the reconstruction
identity (equation~\eqref{eqn: recons}) under Constraint~\ref{assm3: h_poly_new} achieves affine identification, i.e.,
$\forall z \in \mathcal{Z}, \hat{z} = Az + c$,
where $\hat{z}$ is the encoder $f$'s output, $z$ is the true
latent, $A \in \mathbb{R}^{d\times d}$ is  invertible  and $c\in
\mathbb{R}^{d}$.
\end{theorem}

We now strengthen the above affine identification by using the distributional invariance of the stable set of latents. In what follows, we focus on the latents $p_z$ that follow an acyclic structural causal model as follows. In each domain $j\in [k]$,
\begin{equation}
\begin{split}    
  & z_i^{(j)} \leftarrow q_i\big(z_{\mathrm{Pa}(i)}^{(j)}\big)  + \varrho_i^{(j)}, z_{\mathrm{Pa}(i)}^{(j)} \perp \varrho_i^{(j)},  \forall i \in [d]; \\ 
  &  x \leftarrow g(z),
\end{split}
    \label{eqn: dgp_scm_additive_nlinear}
\end{equation}
where $q_i(\cdot)$ refers to the map  that  generates $z_i^{(j)}$, namely the $i^{th}$ component of $z^{(j)}$; $\mathrm{Pa}(i)$ is the set of parents of $z_i^{(j)}$;  $\varrho_i^{(j)}$ is noise in domain $j$. Each sampled latent is mixed by $g$ to generate $x$. We drop the domain index $j$ from $z^{(j)}$  in $x\leftarrow g(z)$ and wherever else it is not needed.
We use domain index $1$ to denote the observational dataset. The domains from index $2$ and onwards correspond to interventional datasets. The interventions considered in this section correspond to imperfect interventions, where the mapping $q_i(\cdot)$ remains unchanged but the distribution of the noise variables changes across domains. We assume that the nodes in $\mathcal{U}$ undergo imperfect interventions, but the nodes in $\mathcal{S}$ are~never~intervened.

\begin{assumption}[Single-node imperfect interventions] \label{assm: imp_int_structure} 
In interventional domain $j$ ($j\geq 2$), exactly one node in $\mathcal{U}$ undergoes an imperfect intervention on the noise term. Moreover, across all domains, each node in $\mathcal{U}$ undergoes intervention at least once. Further, the children of any node in $\mathcal{U}$ must also belong to $\mathcal{U}$.
\end{assumption}

Assumption~\ref{assm: imp_int_structure} implies that the distribution of $z_{\mathcal{S}}$ remains invariant across domains. To identify $z_{\mathcal{S}}$, we thus impose the following invariance constraint: the marginal distribution of components in subset $\hat{\mathcal{S}} \subseteq [d]$ of the estimated latents must remain invariant across domains.

\begin{constraint} (Marginal invariance)
For each $i \in \hat{\mathcal{S}}$, $p_{\hat{z}_{i}^{(p)}} = p_{\hat{z}_{i}^{(q)}}, \forall p \not= q, p,q \in [k].$
\label{assm: dist_inv}
\end{constraint}

\begin{restatable}[Single-node imperfect interventions]{theorem}{singlenode}
\label{thm2}
Suppose the multi-domain data is gathered from the DGP in equation~\eqref{eqn: dgp_scm_additive_nlinear}  under Assumptions~\ref{assm1: dgp1} and \ref{assm: imp_int_structure}. Then the autoencoder that solves the reconstruction
identity (equation~\eqref{eqn: recons}) under Constraints  \ref{assm3: h_poly_new} and \ref{assm: dist_inv} achieves block-affine identification, i.e.,
$\forall z \in \mathcal{Z}, \hat{z}_{\hat{\mathcal{S}}} = D z_{\mathcal{S}} + e$, 
where $\hat{z}$ is the encoder's output, $z$ is the true
latent, $D \in \mathbb{R}^{|\hat{\mathcal{S}}|\times |\mathcal{S}|}$,  and $e\in
\mathbb{R}^{|\hat{\mathcal{S}}|}$.
\end{restatable}

The proof of Theorem~\ref{thm2} is in the Appendix. Theorem \ref{thm2} implies that, under single-node imperfect interventions and polynomial mixing,  the invariant latents $z_{\mathcal{S}}$ are disentangled from the rest of the latents. 
While the SCM (equation~\eqref{eqn: dgp_scm_additive_nlinear}) of the DGP in Theorem~\ref{thm2} does not involve any confounders, we show how this result readily extends to settings with confounders in the Appendix.

We next study multi-domain data coming from multi-node imperfect interventions. For ease of exposition, we begin with two-node imperfect interventions and assume that the noise distributions are Gaussian. We discuss how to relax these assumptions in the Appendix.  Below we describe the key assumptions we make about the mechanisms underlying the interventions.

\begin{assumption}\label{assm: multi_int_str} (Multi-node imperfect interventions) (1) The children of any node in $\mathcal{U}$ must also belong to $\mathcal{U}$ and the underlying DAG must have at least two terminal nodes. Further, the noise $\varrho'$s in \eqref{eqn: dgp_scm_additive_nlinear} are zero-mean Gaussians with variances for observational data (domain $1$) sampled i.i.d. from a non-atomic density $p_{\sigma_{\varrho}}$. 

(2) Interventional data in each domain $j\geq 2$ is generated as follows. For each $i \in \mathcal{U}$, select a random node $j$ from $\mathcal{U}\setminus\{i\}$ uniformly. The noise variance for those two nodes $(i,j)$ are two independent draws from density $p_{\sigma_{\varrho}}$.  Repeat this procedure $t$ times for each node~$i\in \mathcal{U}$.

\end{assumption}

\begin{restatable}[Multi-node imperfect interventions]{theorem}{Multinode}
\label{thm3}
Suppose the multi-domain data is gathered from the DGP in equation~\eqref{eqn: dgp_scm_additive_nlinear}  under Assumptions~\ref{assm1: dgp1} and \ref{assm: multi_int_str}. If the number of multi-node interventions $t$ impacting each node is more than $   \frac{\log(d/\delta)}{\log(1/(1-1/2d))} $, then, with probability $1-\delta$, the autoencoder that solves the reconstruction
identity (equation \eqref{eqn: recons}) under Constraints \ref{assm3: h_poly_new} and \ref{assm: dist_inv} achieves block-affine identification, i.e.,
$\forall z \in \mathcal{Z}, \hat{z}_{\hat{\mathcal{S}}} = D z_{\mathcal{S}} + e$, 
where $\hat{z}$ is the encoder's output, $z$ is the true
latent, $D \in \mathbb{R}^{|\hat{\mathcal{S}}|\times |\mathcal{S}|}, e\in
\mathbb{R}^{|\hat{\mathcal{S}}|}$.
\end{restatable}

The proof of Theorem~\ref{thm3} is in the Appendix. Theorem~\ref{thm3} established that, given sufficiently many random multi-node interventions, we can block identify the stable latents $z_{\mathcal{S}}$. Moreover, the required number of domains scales as $d\Big(\frac{\log(d/\delta)}{\log(1/(1-1/2d))}\Big)$. 
Before closing this section, we remark that the crucial assumptions that make these results possible involve diversity of interventions and using the structure of the causal model. While we study some relaxations, we believe these results can inspire a lot of exciting future work.

\subsection{General Distributions $p_z$}

In the previous section, we made the standard assumption that the relationships between the latents $z$ generating the data $x$ are described by a fixed DAG. In this section, we study a relaxation that is suited to more complex multi-domain datasets, where a fixed DAG is insufficient to capture the complexities of the entire data. For example, in the cow-on-the-beach example, the relationship of the cow to its surroundings changes across samples \citep{beery2018recognition}. We consider  a weaker invariance than one considered in the previous section, i.e., the support of each latent in the target set $\mathcal{S}$ is invariant.  Under these relaxations, we prove that one can still identify the stable latents, except that the number of required domains is much larger. We will also discuss how additional assumptions can help reduce this number in the Appendix. Below we begin by stating the invariance condition. The support of $z_i$ in domain $p$ is denoted as $\mathcal{Z}_{i}^{(p)}$ and the support of estimate $\hat{z}_i$ in domain $p$ is denoted as $\hat{\mathcal{Z}}_i^{(p)}$. 

\begin{assumption} (Marginal support invariance.)

\label{assm: supp_invar}
   For each $i \in  \mathcal{S}$
     \begin{align*} 
       \min_{z \in \mathcal{Z}_{i}^{(p)}} z  =  \min_{z \in \mathcal{Z}_{i}^{(q)} } z, \quad \max_{z \in \mathcal{Z}_{i}^{(p)}} z  =  \max_{z \in \mathcal{Z}_{i}^{(q)}}, \quad  \forall p, q\in [k].
    \end{align*}
\end{assumption}

 We now state a key assumption for the next result: there exists a pair of domains whose supports are sufficiently different. We make this notion mathematically precise below.
\begin{assumption}[Support variability]
\label{assm: sup_var}
There exists two domains $p,q \in [k]$ such that for each $z \in \mathcal{Z}^{(p)}$, there exists a $z^{'} \in \mathcal{Z}^{(q)}$ such that $z^{'} \succcurlyeq z$, namely  each component of $z^{'}$ is greater than or equal to $z$, i.e., $z^{'}_i\geq z_i$. Further, we require that the inequality is strict for unstable components  $j \in \mathcal{U}$, $z^{'}_j > z_j$. 
\end{assumption}
We illustrate the above assumption using an example in Figure~\ref{fig:supp_assum_illus}. The two domains shown in Figure~\ref{fig:supp_assum_illus} satisfy Assumption~\ref{assm: supp_invar}, \ref{assm: sup_var}. The latent $z_1$ in Domains 1 and 2 satisfies support invariance (Assumption~\ref{assm: supp_invar}). The latents $z=[z_1,z_2]$ in Domains 1 and 2 satisfy Assumption~\ref{assm: sup_var}. We now state the invariance constraint that enforces that the latents in subset $\hat{\mathcal{S}}$ have the same minimum and maximum across domains. 


\begin{constraint} (Marginal support invariance)
\label{assm: supp_inv}

For each $i \in \hat{\mathcal{S}}$, 
   $$\min_{z \in \hat{\mathcal{Z}}_{i}^{(p)}} z  =  \min_{z \in \hat{\mathcal{Z}}_{i}^{(q)} } z,  \;\;\; \max_{z \in \hat{\mathcal{Z}}_{i}^{(p)}} z  =  \max_{z \in \hat{\mathcal{Z}}_{i}^{(q)} }, \;  \forall p, q\in [k].$$ 

\end{constraint}

\begin{figure}
\centering
\includegraphics[width = 2.8in]{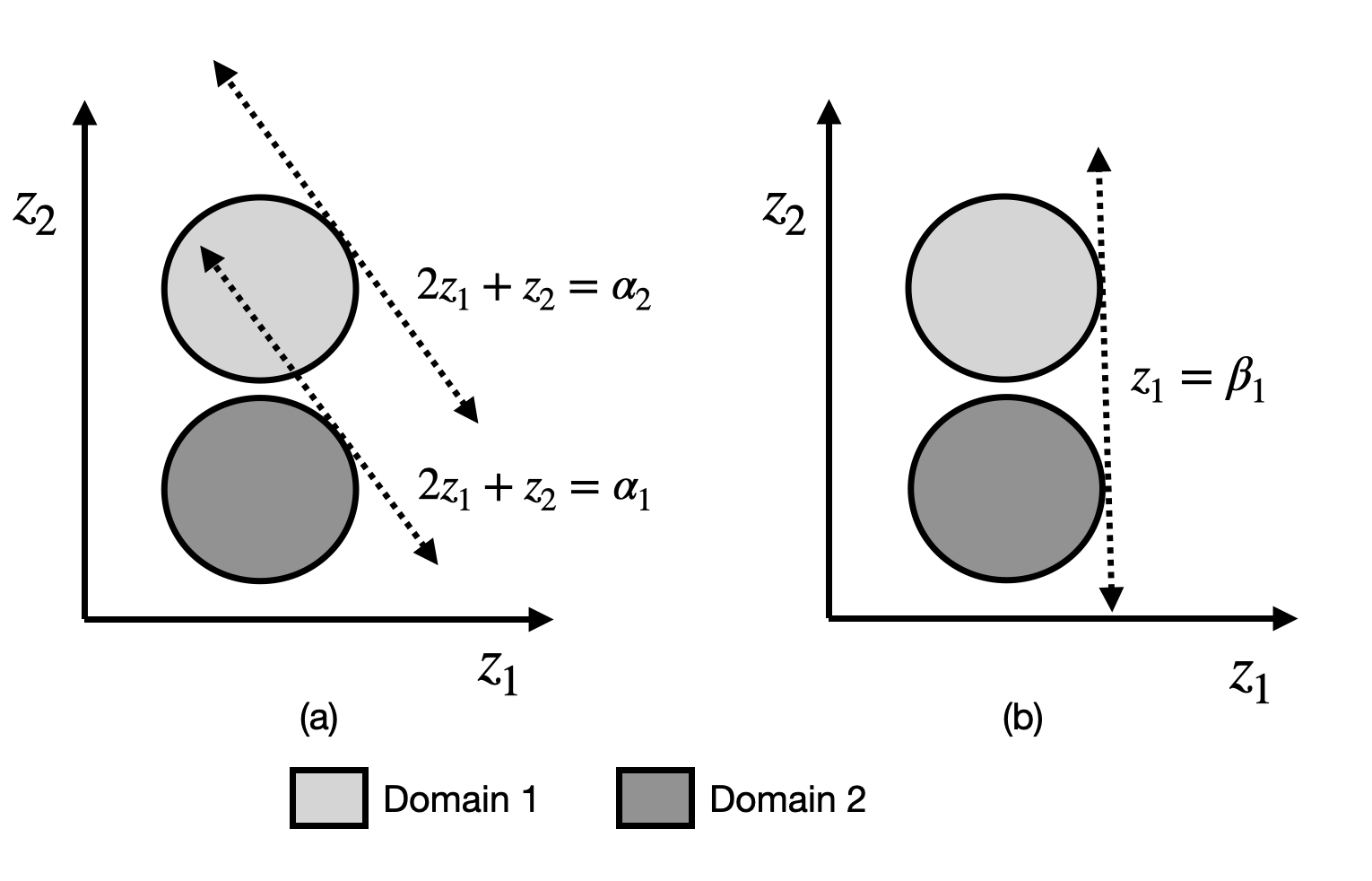}
	\caption{$z_1$ satisfies support invariance (Assumption~\ref{assm: supp_invar}). $[z_1,z_2]$ satisfies support variability (Assumption~\ref{assm: sup_var}). In panel a), we show that if $\hat{z}_1$ linearly depends on both $z_1$ and $z_2$, then it achieves a different maximum value across the two domains. Thus, support invariance (Constraint~\ref{assm: supp_inv}) is not satisfied by such functions that depend on both $z_1$ and $z_2$. In contrast, the function in panel b), which only depends on $z_1$, achieves the same maximum across domains and satisfies support invariance.}
 \label{fig:supp_assum_illus}
\end{figure}

Next, we use the above assumptions to provably identify the stable latents up to block affine transformations under polynomial mixing.

\begin{restatable}{theorem}{gmultipoly}
\label{thm4}
Suppose the multi-domain data is generated 
from  equation~\ref{eqn: dgp} and satisfies Assumptions~\ref{assm1: dgp1}, \ref{assm: supp_invar}, \ref{assm: sup_var}. Then the autoencoder that solves the reconstruction
identity in equation~\ref{eqn: recons} under Constraints \ref{assm3: h_poly_new} and \ref{assm: supp_inv}  achieves the following identification guarantees: Each latent component $i \in \mathcal{S}$ satisfies $\hat{z}_i = A_i^{\top}z + c_i$, where, among all the vectors $A_i\succcurlyeq 0$, the ones that are feasible under the assumptions and constraints in this theorem must satisfy $A_{ir}=0$ for all $r \in \mathcal{U}$. 
\end{restatable}

The proof of Theorem~\ref{thm4} is in the Appendix. 

\paragraph{Extending Theorem~\ref{thm4} beyond the positive orthant.} Theorem~\ref{thm4} leveraged the invariance assumption (Assumption~\ref{assm: supp_invar}) to show that $\hat{z}_i$ only depends on the set of invariant latents in $\mathcal{S}$, provided that $A_i$'s are from the positive orthant, i.e., $A_i  \succcurlyeq 0$.  We next extend this argument to other orthants. Consider $A_{i}$'s from a different orthant with sign vector $s$, where each component of $s$ corresponds to the sign of the corresponding component of $A_i$. We multiply $z$ element-wise with $s$ and denote it as $\bar{z} =  z \cdot s$ and define the set of transformed latents of domain $q$ as $\bar{\mathcal{Z}}^{(q)} = \{z \cdot s, z \in  \mathcal{Z}^{(q)}\}$. If we modify Assumption~\ref{assm: sup_var} with set $\bar{\mathcal{Z}}^{(q)}$ instead of $\mathcal{Z}^{(q)}$, then the condition in Theorem~\ref{thm4} extends to all vectors $A_i$ in orthant with sign vector $s$. Given this assumption, we require a pair of domains that satisfy a condition analogous to the one in Assumption~\ref{assm: sup_var} for each orthant. Since the total number of orthants is $2^{d}$, the total number of domains required grows as $2^{d+1}$. In Appendix~\ref{sec:poly}, we show that the number of domains required can be reduced to $d$ under some additional structural assumptions, e.g. the support is a polytope.

In Theorem~\ref{thm4}, we relied on the assumption that $g$ is a polynomial. We next relax this assumption. For ease of exposition, we consider the two-variable case and present the general case in the Appendix.

\paragraph{Two-variable case.} Consider two-dimensional $z$'s, i.e., $z=[z_1,z_2]$. We assume that the support of the first component $z_1$ is invariant across domains and the support of $z_2$ varies across domains. For the rest of this section, we assume that $z_1$ and $z_2$ are bounded between $0$ and $1$ across all domains. Specifically, the support of $z_1$ satisfies Assumption~\ref{assm: supp_invar} and is set to the entire interval $[0,1]$ across domains.   Recall that the support of the first component of the encoder in domain $p$ is $\hat{\mathcal{Z}}_{1}^{(p)}$. Under the support invariance constraint (Constraint~\ref{assm: supp_inv}), we require that $\hat{\mathcal{Z}}_{1}^{(p)}$ does not vary with $p$.  Recall $\hat{z} = f(x) =  a(z)$, where $a= f \circ g$. The first component of $\hat{z}$ thus satisfies $\hat{z}_1= a_1(z)$, where $a_1$ is the first component of  the map $a$. Under this notation, we define a large class of functions $\Gamma$ and show that, if the supports are sufficiently diverse, then $a_1$ cannot be an element of $\Gamma$, provided that the Constraint~\ref{assm: supp_inv} is enforced -- we call this \emph{$\Gamma^{c}$ identification}. The larger the set $\Gamma$ is,  the more likely $a_1(\cdot)$ is equal to a map that only depends on $z_1$, which is the ideal situation. In contrast, if Constrain~\ref{assm: supp_inv} is not enforced, then all the invertible maps $a(\cdot)$ will be allowed under reconstruction identity in equation~\eqref{eqn: recons}.  Below we state the result formally.

\begin{definition} 
\label{def: lipschitz_a}
Fix some constants $\eta>0$, $\varepsilon>0$, and $\iota>0$. We then define a set of functions $\Gamma$ as follows. Each function $\gamma_{\theta}: [0,1]\times[0,1] \rightarrow \mathbb{R}$ in $\Gamma$  satisfies i)  it is parameterized by $\theta \in \Theta$, where $\Theta$ is a bounded subset of $\mathbb{R}^s$,ii) the minima of $\gamma_{\theta}$ over $[0,1] \times [0,1]$ lie in the $\varepsilon$ interior of the set, i.e., in $[\varepsilon, 1-\varepsilon]\times[\varepsilon, 1-\varepsilon]$, and iii) there exists an interval $[\alpha^{\dagger}, \beta^{\dagger}]$  of width at least $\iota$ such that 
\begin{equation}
\left|\min_{z \in [0,1]\times[0,1]} \gamma_{\theta}(z_1, z_2) - \min_{z\in [0,1]\times [\alpha^{\dagger},\beta^{\dagger}]} \gamma_{\theta}(z_1,z_2)\right| \geq \eta.
\label{eqn:eta_min_diff}
\end{equation}
  For each $(z_1,z_2) \in [0,1]\times [0,1] $, $\gamma_{\theta}$ is Lipschitz continuous in $\theta \in \Theta$ with Lipschitz constant $L$.
\end{definition}

In simple words, $\Gamma$ consists of functions $\gamma_{\theta}$ whose minima over the entire support $[0,1] \times [0,1]$ is $\eta$ better than any other minima obtained by constraining $z_2$ to some interval. In particular, the functions that only depend on $z_1$ do not belong to $\Gamma$ because the minima of such a map do not depend on $z_2$.  A simple illustrative example of the function class $\Gamma$ is as follows: $\gamma_{\theta}:[0,1]\times [0,1] \rightarrow \mathbb{R},$ $\gamma_{\theta}(z_1,z_2) = (z_1-\frac{1}{2})^2 + (z_2-\theta)^2$, where $\theta \in [\frac{1}{2}\varepsilon, 1-\frac{3}{2}\varepsilon]$. This function has its minima over $[0,1]\times [0,1] $ at $(\frac{1}{2}, \theta)$. The function is Lipschitz continuous in $\theta$ for all $(z_1,z_2) \in [0,1] \times [0,1]$. Set $\eta=\frac{\varepsilon^2}{4}$ and $\alpha^{\dagger} = \theta + \frac{\varepsilon}{2}$ and $\beta^{\dagger}=\theta + \frac{5}{8}\varepsilon$; then the conditions in Definition~\ref{def: lipschitz_a} are satisfied.  This example illustrates how these conditions are satisfied when $\gamma_{\theta}$ has one unique global minima over the region $[0,1]\times[0,1]$.  We now state an assumption that requires that the domains are drawn at random and their supports satisfy a certain variability condition.

\begin{assumption}[Support variability]
\label{assm: diverse_int}
 The support of $z_1$ does not vary across domains and is fixed to be $[0,1].$ The support of $z_2$ satisfies $\mathbb{P}\big(\mathcal{Z}_{2}^{(p)} \subseteq [\alpha, \beta] \big) \geq c_1|(\beta-\alpha)|^{l}$  and $\mathbb{P}\big(\mathcal{Z}_{2}^{(p)} \supseteq [\kappa, 1-\kappa]\big) \geq c_2\kappa^r$, where $l$ and $r$ are some integers, $c_1$, $c_2$ are some constants and~$\alpha,\beta,\kappa \in [0,1]$.  
\end{assumption}

The first condition on $z_2$ in Assumption~\ref{assm: diverse_int} states that the probability of the support of $z_2$ in a randomly drawn domain being contained in the interval $[\alpha, \beta]$ grows faster than a polynomial in $(|\beta-\alpha|)$.  The second condition states that the support of $z_2$ captures the set $[\kappa, 1-\kappa]$ with probability at least $c_2\kappa^r$. Below we give an example where these conditions are satisfied: suppose the support of $z_2$ is sampled as follows. Sample two random variables $A$ and $B$ independently from the uniform distribution over the interval $[0,1]$. Define the upper and lower limit of the supports as $\max\{A,B\}$ and  $\min\{A,B\}$ respectively.  In this case, the probabilities in  Assumption~\ref{assm: diverse_int} are given as $(\beta-\alpha)^{2}$ and~$2\kappa^2$.

The next result builds on the following insight. If we sample sufficiently many diverse domains, then it is likely that, for each map $\gamma_{\theta} \in \Gamma$, we encounter two domains such that the values at the minima are at least $\eta$ apart as in Definition~\ref{def: lipschitz_a}. Thus, $\hat{z}_1$ constructed from any member of $\Gamma$ violates the support invariance constraint and thus $a_1$ is not in $\Gamma$.  

Define  $N(\delta, \varepsilon, \eta, \iota) = N_c\log\big(\frac{2N_c}{\delta}\big)\Big(\frac{1}{\log\big(\frac{1}{(1-c_1\iota^{l})}\big)} + \frac{1}{\log\big(\frac{1}{(1-c_2\varepsilon^{r})}\big)}\Big),$
with $N_c = \bigg( \frac{2 \max_{\theta \in \Theta}\|\theta\| \sqrt{s}}{\rho}\bigg)^{s}$, and $\rho=\frac{\eta}{4L}$.

\begin{restatable}{theorem}{gmultidiff}
\label{thm5} If we gather data generated from equation~\eqref{eqn: dgp}, where the support of $z_2$ for each domain is sampled i.i.d. from Assumption~\ref{assm: diverse_int} and support of $z_1$ is fixed to $[0,1]$. Further, suppose the number of domains satisfies $k\geq N(\delta, \varepsilon, \eta, \iota)$. Then the set of maps $a_1(\cdot)$ that relate $\hat{z}_1$ to $[z_1,z_2]$ does not contain any function from $\Gamma$ and thus achieves $\Gamma^{c}$ identification, where $\hat{z}$ is obtained by solving the reconstruction identity (equation~\ref{eqn: recons}) under support invariance constraint (Constraint~\ref{assm: supp_inv}) on $\hat{z}_1$. 
\end{restatable}

The proof of Theorem~\ref{thm5} is in the Appendix. The results studied in this section relied on support variability assumptions. While we study some variations in the Appendix, we believe there is room for new results on multi-domain datasets that are beyond one DAG explaining the entire observational data assumption. In the previous two sections, we saw two types of mixing -- a) polynomial mixing (Theorem~\ref{thm3},\ref{thm4}), b) general diffeomorphisms (Theorem~\ref{thm5}).  The results in a) rely on affine identification guarantees afforded by the polynomial mixing. Under different  assumptions on $g$  that afford affine identification, the results in Theorem~\ref{thm3},\ref{thm4} can be extended. The seminal result in \citet{donoho2003hessian} established affine identification for locally isometric $g$. 

\section{Learning Invariance-Constrained Representations} 
In this section, we describe practical criteria to learn autoencoders described in equation~\eqref{eqn: recons} under invariance constraints from equation~\eqref{eqn: const}.   We will learn in two stages. In the first stage, we learn an autoencoder ($\tilde{f}, \tilde{h}$) that minimizes the reconstruction error -- $  \mathbb{E}\big[\|h\circ f (x) - x \|^2\big] $, where the expectation is taken over the distribution of the raw input data $x$. In Stage $2$, we use the output of the encoder from Stage $1$ denoted as $\tilde{x}$ as inputs. In many cases, this output may have an affine relationship or a more structured relationship with the true latents than the raw inputs. In Stage $2$, we learn an autoencoder $(f^{\star}, h^{\star})$ that is constrained to satisfy certain invariances described in the previous section. We enforce these constraints by adding a penalty to the standard reconstruction error in autoencoders, i.e., the learning objective takes the form 
\begin{equation}
    \mathbb{E}\big[\|h\circ f (\tilde{x}) - \tilde{x} \|^2\big] + \lambda \cdot \mathrm{penalty},
\end{equation}

where the expectation is taken over the distribution of the outputs of the encoder from Stage $1$, $\tilde{x}$. In Constraint~\ref{assm: supp_inv}, we require that the smallest and the largest values to satisfy invariance. The penalty corresponding to this constraint is stated as 
\begin{equation} 
\begin{split}
 \sum_{p\not= q}\sum_{i\in \hat{\mathcal{S}}}\Big(\big(\min_{z\in \tilde{\mathcal{Z}}_{i}^{(p)}} z-\min_{z\in \tilde{\mathcal{Z}}_{i}^{(q)}} z\big)^2 + \big(\max_{z\in \tilde{\mathcal{Z}}_{i}^{(p)}} z-\max_{z\in \tilde{\mathcal{Z}}_{i}^{(q)}} z\big)^2 \Big),
\end{split}
\label{eqn: pen_minmax}
\end{equation}
where $\tilde{\mathcal{Z}}_{i}^{(p)}$ corresponds to the support of the $i^{th}$ component of $f^{\star}(\tilde{x})$ in domain $p$. We now describe a stronger form of invariance.   We can enforce the joint distribution of all components in $\hat{\mathcal{S}}$ to be invariant, which if enforced perfectly would satisfy both Constraint~\ref{assm: dist_inv} and \ref{assm: supp_inv}.  The penalty described below measures the maximum mean discrepancy (MMD) distance between the joint distributions $\hat{z}_{\hat{\mathcal{S}}}$ across all the domains:
\begin{equation} 
\begin{split}
 \sum_{p\not= q} \mathrm{MMD}(p_{\hat{z}_{\hat{\mathcal{S}}}}^{(p)}, p_{\hat{z}_{\hat{\mathcal{S}}}}^{(q)}).
\end{split}
\label{eqn: pen_mmd}
\end{equation}

\section{Empirical Findings} 

We carry out experiments to evaluate the invariance-constrained autoencoders in a host of settings that capture varying complexity of $g$ and varying complexity of the distribution $p_z$.   We study four different types of mixing maps $g$  -- i) linear mixing, ii) polynomial mixing, iii) image rendering of balls, iv) unlabeled colored MNIST data. We follow \citet{ahuja2022interventional} in the creation of datasets for both polynomial mixing and image rendering of balls.  Unlabeled colored MNIST is inspired from labeled colored MNIST used in \cite{arjovsky2019invariant}; note that the challenge posed by this version is significant as we do not use labels of the digits or colors while training to achieve block identification. Our multi-domain datasets respect the following invariance --  the distribution of a subset $\mathcal{S}$ of latents does not change across domains. On the other hand, the distributions of latents in $\mathcal{U}$ undergo change across domains. We particularly induce change by changing the support of latents in $\mathcal{U}$.  For each domain $j$ with distribution $p_z^{(j)}$, we study two types of distributions -- i) independent latents, ii) dependent latents. In the dependent latents data, the latents in $\mathcal{U}$ and $\mathcal{S}$ depend on each other. Further, for the dependent latents, the SCM for the latents is not fixed and it varies across data points and thus we call this setup as dynamic SCM (D-SCM). (Further details about data generation are deferred to the Appendix). 
\begin{figure}
    \centering
    \includegraphics[width=5.3in, trim=0 0 0 10mm]{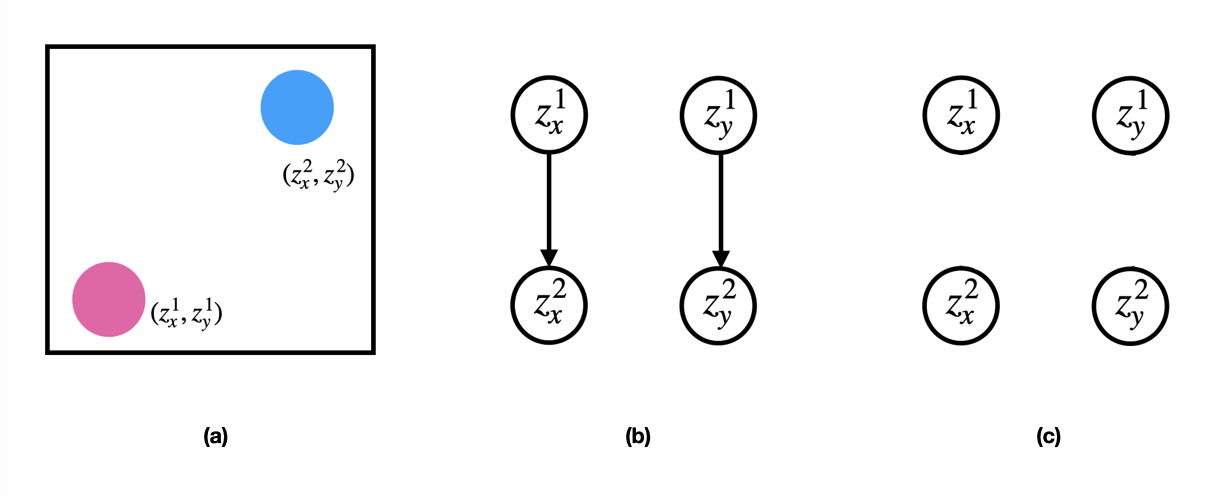}
    \caption{(a) Illustration of an image in balls image dataset. (b,c) The SCM dictating the relationship between the latents varies across data points. (b) For some samples, the coordinates of one ball depend on the other one. (c) For some samples, all causal variables are independent.}
    \label{fig2}
\end{figure}

Algorithmically, we employ the two-stage procedure described in the previous section. For the linear dataset,  we straight carry out Stage $2$ directly, because the raw inputs are already linearly related to the true latents.  However, for the polynomial and image datasets, we carry out the entire two-stage procedure.  For the polynomial dataset, we carry out Stage $1$ experiments with an MLP encoder and a polynomial decoder as prescribed in \citet{ahuja2022interventional}. For the image dataset, we carry out the Stage 1 experiment with a ResNet-based encoder and a simple ConvNet-based decoder. For both the polynomial and the image dataset, we use an MLP encoder-decoder for Stage $2$. We  train the Stage $2$ autoencoder under three different variations of the penalty described in the previous section -- i) support invariance penalty from \eqref{eqn: pen_minmax} (denoted Min-Max), ii) distribution invariance penalty using MMD distance from \eqref{eqn: pen_mmd} (denoted MMD), iii) combination of both support invariance and MMD based invariance (denoted MMD $+$ Min-Max).  Other experimental details can be found in the Appendix.  
\begin{table}[!h]
\small
    \centering
\begin{tabular}{ccc}
\toprule
      $p_Z$ &   Penalty                  &    $(R^2_{\mathcal{S}}, R^2_{\mathcal{U}})$\\
\midrule
         Independent &          Min-Max &   $(0.90\pm0.01, 0.10\pm0.01)$ \\
         Independent &           MMD              &  $(0.92\pm0.00, 0.16\pm0.01)$ \\
         Independent &          MMD $+$ Min-Max              &   $(0.94\pm0.01, 0.07\pm0.01)$ \\
         \midrule
        D-SCM &           Min-Max              & $(0.90\pm0.01, 0.10\pm0.01)$\\
          D-SCM &           MMD            & $(0.92\pm0.00, 0.16\pm0.01)$ \\
        D-SCM &          MMD $+$ Min-Max            &  $(0.97\pm0.00, 0.04\pm0.00)$ \\
         \midrule
\end{tabular}
\caption{Comparisons for linear mixing (latent dimension $d=32$, number of domains $k=16$) }
    \label{table1_results}
\end{table}

\begin{table}[!h]
\small
    \centering
\begin{tabular}{ccc}
\toprule
      $p_Z$ &   Penalty                  &    $(R^2_{\mathcal{S}}, R^2_{\mathcal{U}})$\\
\midrule
         Independent &          Min-Max &   $(0.91\pm0.01, 0.02\pm0.00)$ \\
         Independent &           MMD              &  $(0.93\pm0.01, 0.02\pm0.00)$ \\
         Independent &          MMD $+$ Min-Max              &   $(0.93\pm0.01, 0.02\pm0.00)$ \\
         \midrule
        D-SCM &           Min-Max              & $(0.93\pm0.00, 0.01\pm0.00)$\\
          D-SCM &           MMD            & $(0.95\pm0.00, 0.02\pm0.00)$ \\
        D-SCM &          MMD $+$ Min-Max            &  $(0.95\pm0.00, 0.01\pm0.00)$ \\
         \midrule
\end{tabular}
\caption{Comparisons for polynomial mixing (latent dimension $d=14$, polynomial degree $3$, number of domains $k=16$).}
    \label{table2_results}
\end{table}

\begin{table}[!h]
\small
    \centering
\begin{tabular}{ccc}
\toprule
      $p_Z$ &   Penalty                  &    $(R^2_{\mathcal{S}}, R^2_{\mathcal{U}})$\\
\midrule
        Independent &          Min-Max &   $(0.65\pm0.01, 0.19\pm0.01)$ \\
         Independent &           MMD              &  $(0.63\pm0.04, 0.27\pm0.05)$ \\
        Independent &          MMD $+$ Min-Max              &   $(0.81\pm0.04, 0.18\pm0.02)$ \\
         \midrule
        D-SCM &           Min-Max              & $(0.61\pm0.03, 0.22\pm0.01)$\\
          D-SCM &           MMD            & $(0.55\pm0.12, 0.15\pm0.04)$ \\
        D-SCM &          MMD $+$ Min-Max            &  $(0.82\pm0.02, 0.20\pm0.04)$ \\
         \midrule
\end{tabular}
\caption{Comparisons for  ball-images dataset (number of domains $k=16$).}
    \label{table3_results}
\end{table}

\begin{table}[!h]
\small
    \centering
\begin{tabular}{ccc}
\toprule
      $p_Z$ &   Penalty                  &    $(Acc_{\text{digits}}, R^2_{\text{color}})$\\
\midrule
        Independent &          Min-Max &   $(0.66\pm0.01, 0.49\pm0.02)$ \\
         Independent &           MMD              &  $(0.73\pm0.01, 0.63\pm0.02)$ \\
        Independent &          MMD $+$ Min-Max              &   $(0.74\pm0.01, 0.28 \pm 0.01)$ \\
         \midrule
        D-SCM &           Min-Max              & $(0.53\pm0.01, 0.43\pm0.02)$\\
          D-SCM &           MMD            & $(0.75\pm0.01, 0.65\pm0.02)$ \\
        D-SCM &          MMD $+$ Min-Max            &  $(0.72\pm0.02, 0.31\pm0.03)$ \\
         \midrule
\end{tabular}
\caption{Comparisons for  unlabeled colored MNIST dataset (number of domains $k=16$).}
    \label{table4_results}
\end{table}

\begin{table}[!h]
\small
    \centering
\begin{tabular}{ccc}
\toprule
      $g$ &   Domains                  &    $(R^2_{\mathcal{S}}, R^2_{\mathcal{U}})$\\
\midrule
         Linear &         $2$ &   $(0.33\pm0.01, 0.46\pm0.03)$ \\
         Linear &          $16$              &  $(0.97\pm0.00, 0.04\pm0.00)$ \\
         \midrule
        Polynomial &           $2$             & $(0.58\pm0.02, 0.07\pm0.01)$\\
         Polynomial &           $16$            & $(0.95\pm0.00, 0.01\pm0.00)$ \\ 
         \midrule
        Ball-images &          $2$            &  $(0.73\pm0.01, 0.35\pm0.02)$ \\
        Ball-images &          $16$            &  $(0.82\pm0.02, 0.20\pm0.04)$ \\
         \midrule
\end{tabular}
\caption{Results under varying number of domains.}
    \label{table5_results}
\end{table}

\begin{table}[!h]
\small
    \centering
\begin{tabular}{ccc}
\toprule
      $g$ &   Domains                  &    $(Acc_{\text{digits}}, R^2_{\text{color}})$\\
\midrule
         Unlabeled colored MNIST  & $2$  &$(0.73\pm0.02, 0.73\pm0.02)$ \\ 
         Unlabeled colored MNIST  & $16$ & $(0.74\pm0.01, 0.28\pm0.02)$ \\ 
         \midrule 
\end{tabular}
\caption{Results under varying number of domains.}
    \label{table6_results}
\end{table}

We evaluate the block affine identification of the models as follows. We predict $z_{\mathcal{S}}$ from $\hat{z}_{\hat{\mathcal{S}}}$ using a linear model and compute the $R^2$, which we denote as $R^2_{\mathcal{S}}$. We also predict $z_{\mathcal{U}}$ from  $\hat{z}_{\hat{\mathcal{S}}}$ using a linear model and compute the coefficient of determination $R^2$, which is denoted as $R^2_{\mathcal{U}}$. Here $\hat{\mathcal{S}}$ and $\hat{\mathcal{U}}$ are the set of latents on which invariance constraints are enforced and the set of latents on which no such constraints are enforced. High $R^2_{\mathcal{S}}$ and low $R^2_{\mathcal{U}}$ indicates block identification of the latents. For the unlabeled colored MNIST dataset, we do not have access to the $z$ corresponding to the digits. However, we have access to the labels of the digits for evaluation purposes. On this dataset, we predict the digit from $\hat{z}_{\hat{\mathcal{S}}}$ and predict the color from $\hat{z}_{\hat{\mathcal{S}}}$. We denote the accuracy of digit prediction as $Acc_{\text{digits}}$ and $R^2$ for predicting color as $R^2_{\text{color}}$. 

In \Cref{table1_results,table2_results,table3_results}, we show the results (averaged over five seeds) for independent latents and correlated latents (D-SCM) under linear mixing, polynomial mixing, and ball image rendering. For both linear and polynomial mixing, we find that all three types of penalties work well, i.e., the learned $\hat{z}_{\hat{\mathcal{S}}}$ achieves block affine disentanglement. For the ball-images dataset, we find that the combination of the MMD + Min-Max penalty works the best. In \Cref{table4_results}, we show the results for unlabeled colored MNIST dataset. Here we can see that the combination of the two penalties works much better as well. One important fact to underscore here is that unlabeled colored MNIST is more challenging than balls dataset and separation of color and digit attributes is even more non-trivial. Our approach achieves a noticeable degree of disentanglement in this setting without any supervision, which is quite remarkable given the challenge posed by this setting. In addition, Tables~\ref{table5_results} and \ref{table6_results} illustrate the role of the number of domains in identification. We find that increasing the number of domains helps achieve better identification; the number of required domains to achieve useful identification is less than the worst-case requirements in the theorems.

\section{Conclusions}

In this work, we advance the theory of multi-domain causal representation learning, making it applicable to multi-domain datasets from complex domain shifts (including multi-node imperfect interventions and beyond). We consider a simple invariance principle, namely certain distributional properties of the target latents remain invariant across domains. Following this invariance principle, we propose a class of autoencoders that enforce such weak distributional invariances. We establish identification guarantees of the stable latents for different invariances, ranging from weak invariance of the support to the stronger invariance on the marginal.  

\section*{Acknowledgements}
Yixin Wang was supported in part by the Office of Naval Research under grant
number N00014-23-1-2590 and the National Science Foundation under
Grant No. 2231174 and No. 2310831.

\bibliographystyle{apalike}
\bibliography{MD-CRL-ws.bib}

\clearpage
\newpage 
\appendix

\section*{Appendix}

\section{Theorems and Proofs}

\singlenode*

\begin{proof}
We begin by first checking that the solution to reconstruction identity under the above-said constraints exists. Set $f=g^{-1}$ and $h=g$ and $\hat{\mathcal{S}}=\mathcal{S}$. Firstly, the reconstruction identity is easily satisfied. Also, the  Constraint~\ref{assm: dist_inv} is satisfied as Assumption~\ref{assm: imp_int_structure} holds.  

We construct a proof based on the principle of induction. We sort the vertices in $\mathcal{U}$ in the reverse topological order based on the DAG to obtain a list $\mathcal{U}^{\star}$. We use the principle of induction on this sorted list.
Due to Assumption~\ref{assm: imp_int_structure}, it follows that the first node in the sorted list has to be a terminal node, say this node is $j$. Consider a component $\hat{z}_i$ of $\hat{z}_{\hat{\mathcal{S}}}$. From affine identification (follows from Theorem~\ref{thm1}), we already know that $\hat{z}_i = A_{i}^{\top}z+c_i$. Suppose $j$ undergoes an imperfect intervention in domain $p$. 
We write the invariance constraint condition equating the distribution of  $\hat{z}_i$ between domain $1$ and domain $p$ as

\begin{equation}
    \begin{split}
       & \hat{z}_{i}^{(1)} \stackrel{d}{=} \hat{z}_{i}^{(p)}, \\ 
       &  A_{i}^{\top}z^{(1)} \stackrel{d}{=} A_{i}^{\top}z^{(p)},  \\
       &  A_{i}^{\top}[z_j^{(1)}, z_{-j}^{(1)}] \stackrel{d}{=} A_{i}^{\top}[z_j^{(p)}, z_{-j}^{(p)}]. 
    \end{split}
\end{equation}

  Recall $z_{j}^{(q)} = q_j\big(z_{\mathrm{Pa}(j)}^{(q)}\big) + \varrho_j^{(q)}, \forall q \in [k]$. For all $q\in [k]$, define $w^{(q)} = A_{i,-j}^{\top}z_{-j}^{(q)} + A_{ij}q_j\big(z_{\mathrm{Pa}(j)}^{(q)}\big)$, where $A_{i,-j}$ is the vector of components in $A_{i}$ other than $A_{i,j}$ and $z_{-j}^{(q)}$ is the vector of all components of $z^{(q)}$ except $z_j^{(q)}$. Define $v^{(q)} = A_{ij}\varrho_{j}^{(q)}, \forall q \in [k]$. Substitute these in the above to obtain

  \begin{equation}
      w^{(1)} + v^{(1)} \stackrel{d}{=}  w^{(p)} + v^{(p)}. 
  \end{equation}

  We make some important observations now. Observe that $v^{(1)} \perp w^{(1)}$ and $v^{(p)} \perp w^{(p)}$. Also, since the intervention only changes the noise distribution of $j$ and leaves all rest nodes in the graph unaltered $w^{(1)} \stackrel{d}{=} w^{(p)} $.  We now write the moment generating function (MGF) of $ w^{(1)} + v^{(1)}$ and equate it to MGF of  $w^{(p)} +  v^{(p)}$ as follows. 
    \begin{equation}
        M_{w^{(1)}}(t) M_{v^{(1)}}(t) =  M_{w^{(p)}}(t) M_{v^{(p)}}(t) \label{eqn: proof_single_node1}
    \end{equation}
    Since $w^{(1)} \stackrel{d}{=} w^{(p)} $, the MGFs are equal. As a result, the MGFs of $v^{(1)}$ and $v^{(p)}$ are equal as well. If the MGFs are equal, then $v^{(1)}  \stackrel{d}{=} v^{(p)} $. If $A_{ij} \not =0$, then this implies $\varrho^{(1)} \stackrel{d}{=} \varrho^{(p)}$, which is a contradiction. Therefore, $A_{ij}=0$. This establishes the base case for the induction.

     \begin{equation}
    \begin{split}
       & \hat{z}_{i}^{(1)} \stackrel{d}{=} \hat{z}_{i}^{(s)}, \\ 
       &  A_{i}^{\top}z^{(1)} \stackrel{d}{=} A_{i}^{\top}z^{(s)},  \\
       &  A_{i}^{\top}[z_j^{(1)}, z_{-j}^{(1)}] \stackrel{d}{=} A_{i}^{\top}[z_j^{(s)}, z_{-j}^{(s)}]. 
    \end{split}
\end{equation}

In domain $s$, where node $j$ above is intervened, the only nodes that are impacted are $j$ and its descendants. In $w^{(q)} = A_{i,-j}^{\top}z_{-j}^{(q)} + A_{ij}q_j\big(z_{\mathrm{Pa}(j)}^{(q)}\big)$, the distribution of second term $A_{ij}q_j\big(z_{\mathrm{Pa}(j)}^{(q)}\big)$ is determinded by distribution of parents of $j$, which are not impacted. The first term $A_{i,-j}^{\top}z_{-j}^{(q)} $ comprises of both the descendants of $j$ and other non-descendants. Observe that all the descendants of $j$ precede it in the list $\mathcal{U}^{\star}$. As a result, all the coefficients in $A_{i,-j}$ corresponding to the descendants of $j$ are zero. Therefore, the distribution of the first term $A_{i,-j}^{\top}z_{-j}^{(s)} $ is same as distribution of $A_{i,-j}^{\top}z_{-j}^{(1)} $. On the whole, the distribution of $w^{(s)}$ is same as distribution of $w^{(1)}$. Also, since the contribution of descendants of $j$ in $w^{(q)}$ is zero, we can conclude that $v^{(q)} \perp w^{(q)}$. 
We now repeat the same argument as before.
 We now write the moment generating function (MGF) of $ w^{(1)} + v^{(1)}$ and equate it to MGF of  $w^{(s)} +  v^{(s)}$ as follows. 
    \begin{equation}
        M_{w^{(1)}}(t) M_{v^{(1)}}(t) =  M_{w^{(s)}}(t) M_{v^{(s)}}(t)
    \end{equation}
    Since $w^{(1)} \stackrel{d}{=} w^{(s)} $, the MGFs are equal. As a result, the MGFs of $v^{(1)}$ and $v^{(s)}$ are equal as well. If the MGFs are equal, then $v^{(1)}  \stackrel{d}{=} v^{(s)} $. If $A_{ij} \not =0$, then this implies $\varrho^{(1)} \stackrel{d}{=} \varrho^{(s)}$, which is a contradiction. Therefore, $A_{ij}=0$. This completes the proof.

\end{proof}

\paragraph{Extension of Theorem~\ref{thm2}} The DGP considered above has the form $z_{j}^{(q)} = q_j\big(z_{\mathrm{Pa}(j)}^{(q)}\big) + \varrho_j^{(q)}$. Alternatively, if we consider a new DGP that involves confounder $z_{j}^{(q)} = q_j\big(z_{\mathrm{Pa}(j)}^{(q)}, u_{\mathrm{Pa}(j)}^{(q)}\big) + \varrho_j^{(q)}$, where $ u_{\mathrm{Pa}(j)}^{(q)}$ are confounders that impact at least two latents but are not input to the mixing map $g$, i.e., $x\leftarrow g(z)$. The exact proof steps can be repeated for this more general data generation process provided the additive noise variable is independent of the parent variables, i.e., $\varrho_j^{(q)} \perp  \big(z_{\mathrm{Pa}(j)}^{(q)}, u_{\mathrm{Pa}(j)}^{(q)}\big)$. Observe that we have the following the crucial steps: i) affine identification, ii) $v^{(1)} \perp w^{(1)}$, $v^{(p)} \perp w^{(p)}$ and $w^{(1)} \stackrel{d}{=} w^{(p)} $, iii) product of MGFs based separation in equation~\eqref{eqn: proof_single_node1}, are not impacted by this change and as a result the proof of this extension goes through. 

Define $u(\delta) = \frac{\log\big(d/\delta\big)}{\log\big(1/(1-1/2d)\big)} $. We characterize \emph{good interventions} next. If a node $s$ is paired with terminal node $w$ and if the variance of both the intervened nodes increases or decreases in comparison to observational data, then $s$ undergoes a good intervention.

\begin{lemma}
\label{lemma: int_thresh}
Consider the random intervention mechanism described in Assumption~\ref{assm: multi_int_str}. If $t\geq u(\delta)$,  then with probability $1-\delta$ each node in $\mathcal{U}$ is involved in a good intervention with one of the terminal nodes.
\end{lemma}

\begin{proof}
Select one of the terminal nodes $w$.  Consider all other nodes in $\mathcal{U}\setminus \{w\}$. The mechanism of interventions described in Assumption~\ref{assm: multi_int_str} goes over the nodes in $\mathcal{U}$ iteratively.    In iteration corresponding to interventions for node $s$,  each node in $\mathcal{U}\setminus \{s\}$ is equally likely to be selected for concurrent intervention. Define an event $O$, which is true if under the intervention the variance of both intervened nodes is increased in comparison to observational data (Domain $1$) or if under the intervention variance of intervened is decreased in comparison to observational data. Due to symmetry and non-atomic density $p_{\sigma_{\varrho}}$, the probability of this event is $\frac{1}{2}$.  Therefore, the probability $p$ that in iteration for node $s$ it undergoes a good intervention  is $p=\frac{1}{2(|\mathcal{U}|-1)}$. 

Define an event $S$ such that $S$ occurs if in all of $(|\mathcal{U}|-1)t$ interventions each node in $\mathcal{U}\setminus \{w\}$ undergoes a good intervention, i.e., it is paired with the terminal node $w$ at least once and for each of these interventions event $O$ occurs for the paired nodes. Consider a node $s\in \mathcal{U} \setminus \{w\}$. Define event  $E_s$, where $E_s$ occurs if none of the $t$ interventions conducted in the iteration concerning $s$ are good interventions.  This probability evaluates to $P(E_s) = (1-p)^{t}$. The probability that at least one of $E_s$ is true is bounded above using union bound as follows: $P(\cup_{s\in \mathcal{U}\setminus\{w\}} E_s) \leq (|\mathcal{U}|-1)(1-p)^{t}$. The probability $P(S) = 1-P(\cup_{s\in \mathcal{U}\setminus \{w\}} E_s) \geq 1 - (|\mathcal{U}|-1)(1-p)^{t}$.
Observe that if $t \geq u(\delta) $, then  $P(S) \geq 1-\delta$.

\end{proof}

\Multinode*

\begin{proof}
We begin by first checking that the solution to reconstruction identity under the above-said constraints exists. Set $f=g^{-1}$ and $h=g$ and $\hat{\mathcal{S}}=\mathcal{S}$. Firstly, the reconstruction identity is easily satisfied. Also, the  Constraint~\ref{assm: dist_inv} is satisfied as Assumption~\ref{assm: multi_int_str} holds.  

We construct a proof based on the principle of induction.

Consider a component $\hat{z}_i$ of $\hat{z}_{\hat{\mathcal{S}}}$. From affine identification (follows from Theorem~\ref{thm1}), we already know that $\hat{z}_i = A_{i}^{\top}z+c_i$. We sort the vertices in $\mathcal{U}$ in the reverse topological order to obtain a list $\mathcal{U}^{\star}$. We use the principle of induction on this sorted list. Due to Assumption~\ref{assm: multi_int_str}, it follows that the first two nodes in the sorted list have to be a terminal node, which we denote as $\{j,l\}$. Suppose these nodes are intervened in domain $p$. Observe that since $t\geq u(\delta)$ both of these nodes are intervened with probability $1-\delta$.  From the invariance constraint on distribution $\hat{z}_i$ in domain $1$ and $p$ it follows

\begin{equation}
    \begin{split}
       & \hat{z}_{i}^{(1)} \stackrel{d}{=} \hat{z}_{i}^{(p)}, \\ 
       &  A_{i}^{\top}z^{(1)} \stackrel{d}{=} A_{i}^{\top}z^{(p)},  \\
       &  A_{i}^{\top}[z_j^{(1)}, z_{l}^{(1)}, z_{-jl}^{(1)}] \stackrel{d}{=} A_{i}^{\top}[z_j^{(p)}, z_{l}^{(p)}, z_{-jl}^{(p)}]. 
    \end{split}
\end{equation}

Recall $z_{j}^{(q)} = q_j\big(z_{\mathrm{Pa}(j)}^{(q)}\big) + \varrho_j^{(q)}, \forall q \in [k]$.  For all $q\in [k]$, define $w^{(q)} = A_{i,-jl}^{\top}z_{-jl}^{(q)} + A_{ij}q_j\big(z_{\mathrm{Pa}(j)}^{(q)}\big) + A_{il}q_l\big(z_{\mathrm{Pa}(l)}^{(q)}\big), \forall q \in [k]$, where $A_{i,-jl}$ is the vector of components in $A_{i}$ other than $A_{i,j}$ and $A_{i,l}$, and $z_{-jl}^{(q)}$ is the vector of all components of $z^{(q)}$ except $z_j^{(q)}$ and $z_{l}^{(q)}$. Define $v^{(q)} = A_{ij}\varrho_{j}^{(q)} + A_{il}\varrho_{l}^{(q)}, \; \forall q \in [k]$. Substitute these in the above to obtain 

  \begin{equation}
      w^{(1)} + v^{(1)} \stackrel{d}{=}  w^{(p)} + v^{(p)}. 
  \end{equation}

  We make some important observations now. Observe that $v^{(1)} \perp w^{(1)}$ and $v^{(p)} \perp w^{(p)}$. This is true since $v^{(q)}$ is determined by the noise variables at the terminal nodes. Also, since the intervention only changes the noise distribution of $j$ and $l$, which are terminal nodes, leaving the rest of the nodes unaltered. Therefore,  $w^{(1)} \stackrel{d}{=} w^{(p)} $.  We now write the moment generating function (MGF) of $ w^{(1)} + v^{(1)}$ and equate it to MGF of  $w^{(p)} +  v^{(p)}$ as follows
    \begin{equation}
        M_{w^{(1)}}(t) M_{v^{(1)}}(t) =  M_{w^{(p)}}(t) M_{v^{(p)}}(t).
    \end{equation}
    Since $w^{(1)} \stackrel{d}{=} w^{(p)} $, the MGFs are equal. As a result, the MGFs of $v^{(1)}$ and $v^{(p)}$ are equal as well. If the MGFs are equal, then $v^{(1)}  \stackrel{d}{=} v^{(p)} $. If $A_{ij} \not =0$ and $A_{il}=0$, then this implies $\varrho_{j}^{(1)} \stackrel{d}{=} \varrho_{j}^{(s)}$, which is a contradiction.  Similarly, $A_{il}\not=0$ and $A_{ij}=0$ is not possible either.   The last case is $A_{ij}\not=0$ and $A_{il}\not=0$.   From $v^{(1)}  \stackrel{d}{=} v^{(p)}  \implies A_{ij}\varrho_{j}^{(1)} + A_{il}\varrho_{l}^{(1)} \stackrel{d}{=} A_{ij}\varrho_{j}^{(p)} + A_{il}\varrho_{l}^{(p)}  $. This can only be true if $A_{ij}^2\sigma_{\varrho_{j}^{(1)}}^2 + A_{il}^2\sigma_{\varrho_{l}^{(1)}}^2 =  A_{ij}^2\sigma_{\varrho_{j}^{(p)}}^2 + A_{il}^2\sigma_{\varrho_{l}^{(p)}}^2  $. Due to Lemma~\ref{lemma: int_thresh}, the selected domain $p$ is such that the two terminal nodes undergo a good intervention and as a result, the variance in LHS is strictly less or strictly greater than the RHS, which makes the equality impossible. Therefore, $A_{ij}=0$ and $A_{il}=0$. 
    
    This establishes the base case for the induction.

     Consider an arbitrary vertex say $s \in \mathcal{U}^{\star}$. Suppose $A_{ir}=0$ for all that preceded $s$ in $\mathcal{U}^{\star}$. Further, suppose that this node $s$ undergoes an imperfect intervention along with terminal node $l$ in domain $u$. Note here again since $t\geq u(\delta)$, such a domain exists with probability $1-\delta$. 
     From the invariance condition between domain $1$ and domain $u$, it follows

     \begin{equation}
    \begin{split}
       & \hat{z}_{i}^{(1)} \stackrel{d}{=} \hat{z}_{i}^{(u)}, \\ 
       &  A_{i}^{\top}z^{(1)} \stackrel{d}{=} A_{i}^{\top}z^{(u)},  \\
       &   A_{i}^{\top}[z_s^{(1)}, z_{l}^{(1)}, z_{-sl}^{(1)}] \stackrel{d}{=} A_{i}^{\top}[z_s^{(u)}, z_{l}^{(u)}, z_{-sl}^{(u)}]. 
    \end{split}
    \label{eqn: multinode_imp_int}
\end{equation}

Consider domain $u$, where node $s$ and $l$ above are intervened simultaneously.  Recall $w^{(q)} = A_{i,-sl}^{\top}z_{-sl}^{(q)} + A_{is}q_s\big(z_{\mathrm{Pa}(s)}^{(q)}\big) + A_{il}q_l\big(z_{\mathrm{Pa}(l)}^{(q)}\big), \forall q \in [k]$.  We already showed that $A_{il}=0$ so the third term is zero. Further, in $A_{i,-sl}$ the terms corresponding to the descendants of $s$ are zero due to supposition in induction principle that  $A_{ir}=0$ for all that preceded $s$ in $\mathcal{U}^{\star}$.  Hence, no descendant of $s$ contributes to the expression $w^{(q)}$. The term $v^{(q)} = A_{is}\varrho_{s}^{(q)} + A_{il}\varrho_{l}^{(q)}$, which again simplifies to $v^{(q)} = A_{is}\varrho_{s}^{(q)} $. Since $w^{(q)}$ does not involve $s$ or its descendants, we can conclude that $w^{(q)} \perp v^{(q)}, \forall q \in [k]$ and $w^{(u)} \stackrel{d}{=} w^{(1)}$. 

The above expressions in equation~\eqref{eqn: multinode_imp_int} can be stated as 
\begin{equation}
    w^{(u)} + v^{(u)} \stackrel{d}{=} w^{(1)} + v^{(1)}
\end{equation}

Since $w^{(q)} \perp v^{(q)}$ and $w^{(u)} \stackrel{d}{=} w^{(1)}$, it follows that $v^{(u)}\stackrel{d}{=} v^{(1)}$. If $A_{is}\not=0$, then this implies $\varrho_{s}^{(u)} \stackrel{d}{=} \varrho_{s}^{(1)}$, which leads to a contradiction. Hence, $A_{is}=0$. This completes the proof.

\end{proof}

\paragraph{Extension of Theorem~\ref{thm3}} In Theorem~\ref{thm3}, we considered two-node interventions. Let us ask what happens if $m$-interventions occur at the same time. If we extend the Assumption~\ref{assm: imp_int_structure} to require $m$ terminal nodes, the rest of the argument extends to this case too. Firstly, in Lemma~\ref{lemma: int_thresh} we showed that if the minimum number of interventions $t$ that each node is involved is sufficiently large, then all the nodes end up being paired with one of the terminal nodes. The extension of Lemma~\ref{lemma: int_thresh} reads: if the minimum number of interventions $t$ that each node is involved is sufficiently large, then all the nodes end up being paired with $m-1$ terminal nodes under a good intervention. In the proof of Theorem~\ref{thm3}, in the base case, we showed that the $A_{ij}$ and $A_{il}$ are zero where $\{j,l\}$ are two terminal nodes involved in the intervention. In the extension, we consider the domain in which $m$ terminal nodes are involved in the intervention and the coefficient $A_{ir}$ is zero for all $r$ corresponding to indices of the terminal nodes intervened in that domain. The rest of the argument from the principle of induction is identical.

\gmultipoly*

\begin{proof}

We begin by first checking that the solution to reconstruction identity under the above-said constraints exists. Set $f=g^{-1}$ and $h=g$ and $\hat{\mathcal{S}}=\mathcal{S}$. Firstly, the reconstruction identity is easily satisfied. Also, the  Constraint~\ref{assm: supp_inv} is satisfied as Assumption~\ref{assm: supp_invar} holds.

From the Assumptions~\ref{assm1: dgp1} and Constraint~\ref{assm3: h_poly_new} we know that $\hat{z} = Az +c$ (follows from Theorem~\ref{thm1}). Let us consider $i \in \hat{\mathcal{S}}$. We know that $\hat{z}_i = A_i^{\top}z + c_i$.  Suppose $A_i\succcurlyeq 0$, where each component of $A_i$ is non-negative.

Let us consider the domains $p,q,$ from Assumption~\ref{assm: sup_var}. We compute the maximum value of $\hat{z}_i$ in domain $p$ and $q$ below. 

\begin{equation}
     z^{\max, p} = \arg\max_{z \in \mathcal{Z}^{(p)}}  A_i^{\top}z + c_i
\end{equation}

\begin{equation}
     z^{\max, q} = \arg\max_{z \in \mathcal{Z}^{(q)}}  A_i^{\top}z + c_i
\end{equation}

From Constraint~\ref{assm: supp_inv}, $  A_i^{\top}z^{\max, p} =  A_i^{\top}z^{\max, q}$. Suppose $A_{ik} > 0$ for some $k \in \mathcal{U}$. From Assumption~\ref{assm: sup_var}, it follows that there exists a $z \in  \mathcal{Z}^{(q)}$ such that $ z \succcurlyeq z^{\max, p} $ and $z_{j} > z^{\max, q}_{j}$ for all $j \in \mathcal{U}$. Therefore, $A_{i}^{\top}z > A_{i}^{\top}z^{\max, p}$.  This contradicts $  A_i^{\top}z^{\max, p} =  A_i^{\top}z^{\max, q}$. Therefore, $A_{ik}=0$. 
\end{proof}

\textbf{Remark on Definition \ref{def: lipschitz_a}}
We illustrate the type of functions captured by Definition \ref{def: lipschitz_a} in Figure~\ref{fig:supp_2var}. In Figure~\ref{fig:supp_2var}, we show that a function $\gamma_{\theta}$ has three minima (shown as stars) over $[0,1]\times[0,1]$. We illustrate two domains in panels a) and b). For Domain 1 in panel a), the minima over Domain 1 coincides with minima over $[0,1]\times [0,1]$ but for Domain 2 that is not the case. The figure lays down the examples idea behind the proof we are about to present next. Under sufficiently many diverse interventions, it can be guaranteed that we obtain one domain that is similar to Domain 1 (capturing the minima over $[0,1]\times[0,1]$) in Figure~\ref{fig:supp_2var} and another domain that is similar to Domain 2  (not capturing the minima over $[0,1]\times[0,1]$) in Figure~\ref{fig:supp_2var}.

\begin{figure}
	\includegraphics[trim =0  2.5in 0 0 , width = 5in]{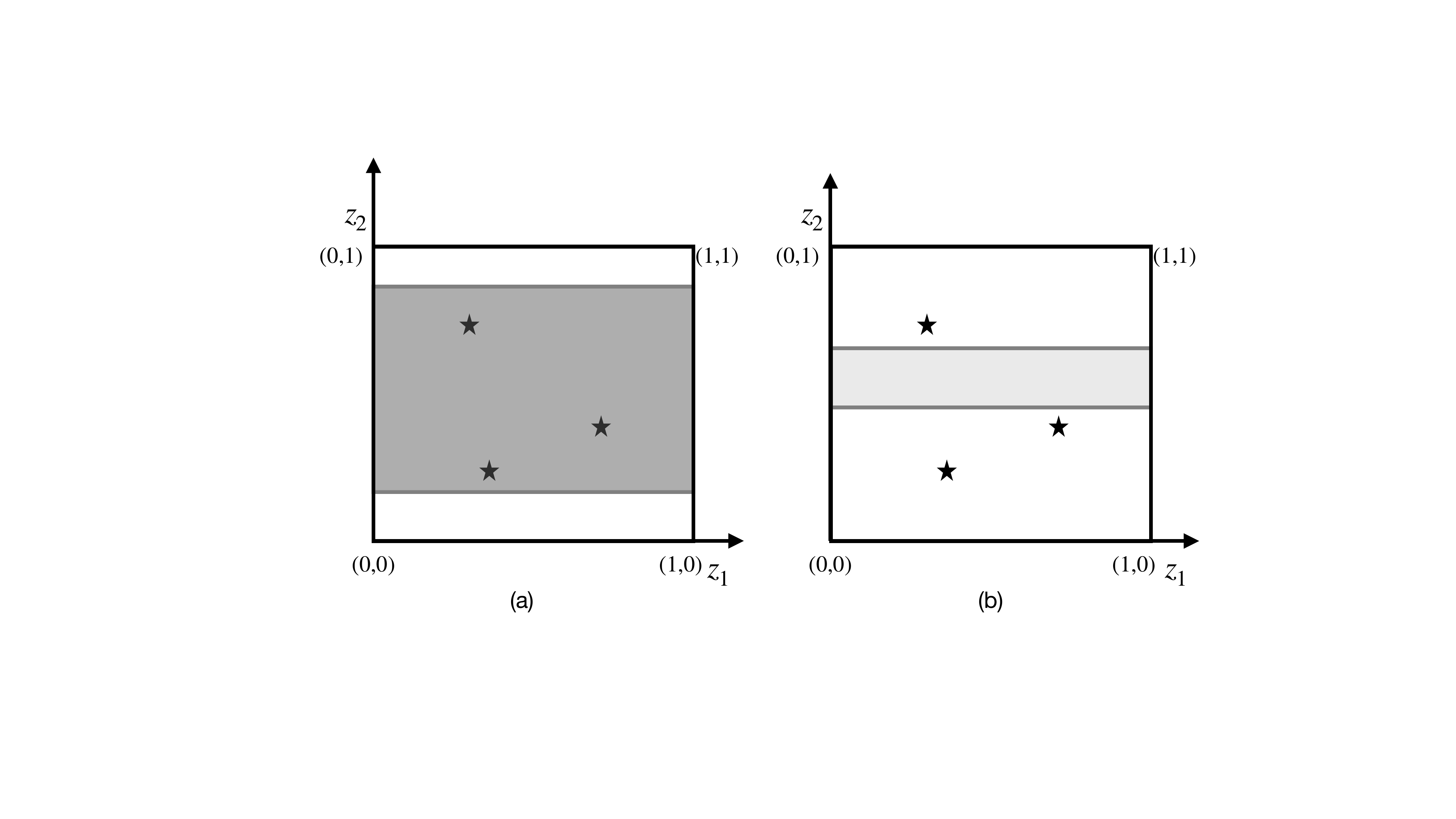}
	\caption{ The minima of a candidate function $\gamma_{\theta}$ over $[0,1]\times [0,1]$ is attained at points shown in stars. Support of Domain 1 and Domain 2 are shown in light and dark grey. The minimum value of $\gamma_{\theta}$ over Domain 1 is not the same as the minimum over Domain 2. Therefore, $a_1(\cdot)$ relating the first component of the autoencoder, which satisfies support invariance constraint, to the true latent cannot be equal to the candidate function $\gamma_{\theta}$.}
 \label{fig:supp_2var}
\end{figure}

\gmultidiff*


\begin{proof}

Consider the set $\Theta$ of parameters, which characterize all the functions in $\Gamma$. Let us construct a $\rho$-cover for the set $\Theta$ with $\rho = \frac{\eta}{4L}$, where $\eta$ and $L$ are constants from Definition \ref{def: lipschitz_a}. We define the set of functions in the cover as $\Gamma_c = 
\{\gamma_1, \cdots, \gamma_{N_{c}}\}$, where $N_c$ is the size of the cover and  $N_c = \bigg( \frac{2 \max_{\theta \in \Theta}\|\theta\| \sqrt{s}}{\rho}\bigg)^{s}$ (follows from \citep{shalev2014understanding}).

Consider a $\gamma_j \in \Gamma_c$ with parameters $\theta_j$.  From Definition \ref{def: lipschitz_a}, there exists an interval $[\alpha^{\dagger},\beta^{\dagger}]$ with width at least $\iota$ such that the minimum value in $[0,1]\times [\alpha^{\dagger},\beta^{\dagger}]$ is at least $\eta$ more than the minimum value over the entire set $[0,1]\times [0,1]$.  Since the support to $z_2$ is sampled randomly, we compute the probability that one of the sampled domain's support is contained in $[\alpha^{\dagger},\beta^{\dagger}]$. The probability of first success (where success is the event that support of $z_2$  is a subset of $[\alpha^{\dagger},\beta^{\dagger}]$) in one of the $t$ trials is $1-(1-p_{s})^{t}$. We want $$1-(1-p_{s})^{t} \geq 1-\frac{\delta}{2} \implies \frac{\delta}{2} \geq (1-p_{s})^{t} \implies \log\bigg(\frac{2}{\delta}\bigg)/\log(1/(1-p_s)) \leq t $$

We plug $p_s = c_1\iota^{l}$ following Assumption~\ref{assm: diverse_int}. If we set $t\geq t_{\min}^{1} = \log(\frac{2}{\delta})/\log(1/(1-c_1\iota^{l}))$, then with probability $1-\delta/2$ at least for one of the domains indexed from $1$ to $t_{\min}^{1}$ the minimum value of $\gamma_j$  in $[0,1]\times [\alpha^{\dagger},\beta^{\dagger}]$ is $\eta$ larger than the minimum value in $[0,1]\times [0,1]$. 

Next, we show that if the number of domains is sufficiently large, then the probability that one of the domains support contains $[\varepsilon, 1-\varepsilon]$ is sufficiently high.   
The probability of first success (where success is the event that the intervention support contains $[\varepsilon, 1-\varepsilon]$). In this case, we follow the same calculations as above. It follows that if $t\geq t_{\min}^{2} = \log(\frac{2}{\delta})/\log(1/(1-c_2\varepsilon^{r}))$, then with probability $1-\delta/2$ the support of $z_2$ in at least one of the domains indexed from $t_{\min}^{1} + 1$ to $t_{\min}^{1} + t_{\min}^{2}$ contains $[\varepsilon, 1-\varepsilon]$ the global minimum of $\gamma_j$ with probability at least $1-\delta/2$. Hence, we can conclude that with probability $1-\delta$ both the success events described above happen. In the case of this event, the function $\gamma_j$ cannot satisfy the support invariance constraint.  

Let us consider all the elements in $\Gamma_c$ together now. We now derive a bound on the number of domains such that none of the elements in $\Gamma_c$ satisfy the support invariance constraint.  We divide the total $k$ domains into blocks of equal length. The first block is chosen to be sufficiently large to ensure that with probability $1-\frac{\delta}{N_c}$, the first element of $\Gamma_c$, i.e., $\gamma_1$ does not satisfy support invariance constraints. Similarly, the second block is chosen to be sufficiently large such that $\gamma_2$ cannot satisfy support invariance constraints and so on. The minimum size of each block is computed by substituting $\delta$ with $\delta/N_c$ in the expression for $t_{\min}^{1} + t_{\min}^2$ derived above. The final expression for $N(\delta, \varepsilon, \eta, \iota)$ is given as 

$$N_c\Bigg(\log\bigg(\frac{2N_c}{\delta}\bigg)/\log\bigg(1/(1-c_1\iota^{l})\bigg) + \log\bigg(\frac{2N_c}{\delta}\bigg)/\log\bigg(1/(1-c_2\varepsilon^{r})\bigg)\Bigg)$$

where $N_c = \bigg( \frac{2 \max_{\theta \in \Theta}\|\theta\| \sqrt{s}}{\rho}\bigg)^{s}$ and $\rho = \frac{\eta}{4L}$.

Observe that since the probability of success is bounded below by $1-\frac{\delta}{N_c}$, the overall probability is bounded by at least $1-\delta$. So far, we have shown that none of the elements in the cover of $\Theta$, i.e., $\Gamma_c$ satisfy support invariance constraints. 

Let us now consider a $\gamma_{\theta} \in \Gamma$. The nearest neighbor of this $\gamma_{\theta}$ in the cover is say $\gamma_{j}$.  Suppose the parameter associated with $\gamma_{j}$ is $\theta_j$. Therefore, $\gamma_j = \gamma_{\theta_j}$. Since $\theta_j$ is an element of $\rho-$cover, the separation between their corresponding parameters is $\|\theta_j-\theta\|\leq \rho$.  Since the number of domains is larger than $N(\delta, \varepsilon, \eta, \iota)$ we can state the following. With probability $1-\delta/N_c$, there exists a pair of domains whose supports say $\mathcal{Z}$ and $\tilde{\mathcal{Z}}$, where $\gamma_{\theta_j}$'s minimum value on the former is at least $\eta$ higher than the minimum value on $\tilde{\mathcal{Z}}$. Let us now compute a lower bound on the minimum value of $\gamma$ on $\mathcal{Z}$. For all $z \in \mathcal{Z}$
$$|\gamma_{\theta}( z) - \gamma_{\theta_j}( z)| \leq L \|\theta-\theta_j\|\leq L \rho \implies \gamma_{\theta}(z) \geq  \gamma_{\theta_j}(z)  - L\rho$$

In the first inequality, we rely on Lipschitz continuity of $\gamma_{\theta}$ in $\theta$ (from Assumption~\ref{assm: diverse_int}).  From the above, it follows that 
\begin{equation}
\min_{z \in \mathcal{Z}}\gamma_{\theta}(z) \geq  \min_{z \in \mathcal{Z}}\gamma_{\theta_{j}}(z)  - L\rho
\label{eqn:bd1}
\end{equation}

Next, we compute an upper bound on the minimum value of $\gamma_{\theta}$ on $\tilde{\mathcal{Z}}$

$$|\gamma_{\theta}(z) - \gamma_{\theta_j}(z)| \leq L \|\theta-\theta_j\|\leq L \rho \implies \gamma_{\theta}( z) \leq  \gamma_{\theta_j}(z)  + L\rho$$

From the above, it follows that 
\begin{equation}    
\min_{z \in \tilde{\mathcal{Z}}}\gamma_{\theta}(z) \leq  \min_{z \in \tilde{\mathcal{Z}}}\gamma_{\theta_j}(z)  + L\rho
\label{eqn:bd2}
\end{equation}

We now take the difference of the bounds in equation~\eqref{eqn:bd1}  and \eqref{eqn:bd2} above to arrive at the following. 

$$\min_{z \in \tilde{\mathcal{Z}}}\gamma_{\theta}(z) -  \min_{z \in \mathcal{Z}}\gamma_{\theta}(z) \geq \min_{z \in \mathcal{Z}}\gamma_{\theta_j}( z) - \min_{z \in \tilde{\mathcal{Z}}}\gamma_{\theta_j}(z)- 2L\rho \geq \eta - 2L\rho  = \frac{\eta}{2}$$

where we set $\rho = \eta/4L$ in the last inequality. Therefore, $\gamma_{\theta}$ does not satisfy support invariance.  We require the above argument to hold for all $\gamma_{\theta} \in \Gamma$. Here we exploit the fact that with probability $1-\delta$ all elements in the cover $\Gamma_c$ do not satisfy the support invariance constraint. Therefore, we can pick any $\gamma_{\theta} \in \Gamma$, select the corresponding nearest neighbor in the cover, and apply the argument stated above. This completes the proof. 

\end{proof}

\subsection{Beyond the Two Variable Case}
 In this section, we aim to generalize the results presented in the previous section to more than two variables. We first adapt the Definition \ref{def: lipschitz_a}. 

\begin{definition}
	\label{def: lipschitz_a1}
	Fix some constants $\eta>0$, $\varepsilon>0$ and $\iota>0$. Given these constants, we define a set of functions $\Gamma$ as follows. Each function $\gamma_{\theta}: [0,1]^{d} \rightarrow \mathbb{R}$ in $\Gamma$  i)  is parameterized by $\theta \in \Theta$, where $\Theta$ is a bounded subset of $\mathbb{R}^s$,ii) the minima of $\gamma_{\theta}$ over the entire set $[0,1]^{d}$ lie in the $\varepsilon$ interior of the set, i.e., in $[\varepsilon, 1-\varepsilon]^{d}$, and iii) there exists a hypercube $\mathcal{L}$  of volume at least $\iota$ such that 
	 $$\left|\min_{z \in [0,1]^{d}} \gamma(z) - \min_{z\in [0,1]\times \mathcal{L}} \gamma(z)\right| \geq \eta. $$
	For each $z \in [0,1]^{d} $, $\gamma_{\theta}$ is Lipschitz continuous in the parameter $\theta \in \Theta$ with Lipschitz constant $L$.
\end{definition}

Next, we adapt Assumption~\ref{assm: diverse_int}.  

\begin{assumption}
\label{assm: diverse_int1}
We assume that the domains are drawn at random and the support of latents in $\mathcal{U}$ satisfy $\mathbb{P}\bigg(\mathcal{Z}_{\mathcal{U}}^{(p)} \subseteq [\alpha_1, \beta_1] \times \cdots [\alpha_{|\mathcal{U}|}, \beta_{|\mathcal{U}|}]\bigg) \geq c_1\mathrm{vol}^{l}[[\alpha_1, \beta_1] \times \cdots [\alpha_{|\mathcal{U}|}, \beta_{|\mathcal{U}|}]]$  and $\mathbb{P}\bigg(\mathcal{Z}_{\mathcal{U}}^{(p)} \supseteq [\kappa, 1-\kappa]^{q}\bigg) \geq c_2\kappa^{qr}$, where $l$ and $r$ are some integers and $c_1$, $c_2$ some constants.  
\end{assumption}

Define $$\tilde{N}(\delta, \varepsilon, \eta, \iota) = N_c\Bigg(\log\bigg(\frac{2N_c}{\delta}\bigg)/\log\bigg(1/(1-c_1\iota^{l})\bigg) + \log\bigg(\frac{2N_c}{\delta}\bigg)/\log\bigg(1/(1-c_2\varepsilon^{dr})\bigg)\Bigg)$$

where $N_c = \bigg( \frac{2 \max_{\theta \in \Theta}\|\theta\| \sqrt{s}}{\rho}\bigg)^{s}$ and $\rho = \frac{\eta}{4L}$

\begin{theorem}  

 If we gather data generated from equation~\eqref{eqn: dgp}, where the support of $z_2$ for each domain is sampled i.i.d. from Assumption~\ref{assm: diverse_int1}. Further, if the number of domains $k\geq \tilde{N}(\delta, \varepsilon, \eta, \iota)$, then the maps $a_1(\cdot)$, which are obtained from autoencoders that solve the reconstruction identity in equation~\eqref{eqn: recons} under support invariance constraint (Constraint~\ref{assm: supp_inv}) on $\hat{z}_1$, do not contain function from $\Gamma$.
\end{theorem}

\begin{proof}

We will follow the same line of reasoning as in the proof of the two-variable case. 
Consider the set $\Theta$ characterizing the functions $\gamma$. Let us construct a $\rho$-cover for the set $\Theta$. The cover consists of functions in the set $\Gamma_c = 
\{\gamma_1, \cdots, \gamma_{N_{c}}\}$, where $N_c$ is the size of the cover.  Consider $\gamma_j \in \Gamma_c$ with parameters $\theta_j$.  From Assumption~\ref{assm: diverse_int1}, there exists a hypercube $\mathcal{L}$ with volume at least $\iota$ such that the minimum value in that hypercube is $\eta$ more than the global minimum on the set $[0,1]^{d}$. The probability that one of the domain's support is contained in the hypercube $\mathcal{L}$ is calculated as follows. The probability of first success (where success is the event that intervention support is a subset of $\mathcal{L}$) in one of the $t$ trials is $1-(1-p_{s})^{t}$. We want $$1-(1-p_{s})^{t} \geq 1-\frac{\delta}{2} \implies \frac{\delta}{2} \geq (1-p_{s})^{t} \implies \log\bigg(\frac{2}{\delta}\bigg)/\log(1/(1-p_s)) \leq t $$

Finally we have $t\geq t_{\min}^{1} = \log\big(\frac{2}{\delta}\big)/\log(1/(1-c_1\iota^{l}))$. Therefore, with probability $1-\delta/2$ at least one of the domains s indexed from $1$ to $t_{\min}^{1}$ achieves a minima $\eta$ larger than the global minimum of $\gamma_j$. 

Next, we derive the probability that one of the domain's support contains $[\varepsilon, 1-\varepsilon]^{d}$.   
The probability of first success (where success is the event that the domain contains $[\varepsilon, 1-\varepsilon]^{d}$). In this case, we have $t\geq t_{\min}^{2} = \log(\frac{2}{\delta})/\log(1/(1-c_2\varepsilon^{rd}))$. Therefore, with probability $1-\delta/2$ at least one of the domains indexed from $t_{\min}^{1} + 1$ to $t_{\min}^{1} + t_{\min}^{2}$ achieves the global minimum of $\gamma_j$ with probability at least $1-\delta/2$. Hence, we can conclude that with probability $1-\delta$ both the success events described above happen. In the case of this event, the function $\gamma_j$ cannot satisfy the invariance constraint.

Let us consider all the elements in $\Gamma_c$ together now. We would require the total interventions to be divided into blocks of equal length. The first block is chosen to be sufficiently large to ensure that with probability $1-\frac{\delta}{N_c}$, $\gamma_1$ cannot satisfy support invariance constraints. Similarly, the second block is chosen to be sufficiently large such that $\gamma_2$ cannot satisfy support invariance constraints and so on. Due to symmetry, the minimum size of each block is computed by substituting $\delta$ with $\delta/N_c$ in the expression for $t_{\min}^{1} + t_{\min}^2$ derived above. The final expression for $\tilde{N}(\delta, \varepsilon, \eta, \iota)$ is given as 

$$N_c\Bigg(\log\bigg(\frac{2N_c}{\delta}\bigg)/\log\bigg(1/(1-c_1\iota^{l})\bigg) + \log\bigg(\frac{2N_c}{\delta}\bigg)/\log\bigg(1/(1-c_2\varepsilon^{dr})\bigg)\Bigg)$$

where $N_c = \bigg( \frac{2 \max_{\theta \in \Theta}\|\theta\| \sqrt{s}}{\rho}\bigg)^{s}$ and $\rho = \frac{\eta}{4L}$.
Observe that since the probability of success is bounded below by $1-\frac{\delta}{N_c}$, the overall probability is bounded by at least $1-\delta$. So far we have shown that none of the elements in the cover of $\Theta$, i.e., $\Gamma_c$ satisfy support invariance constraints. 

Let us now consider a $\gamma \in \Theta$. The nearest neighbor of this $\gamma$ in the cover is say $\gamma_j$. Suppose the parameter of $\gamma_j$ is $\theta_j$. Therefore, $\gamma_j = \gamma_{\theta_j}$. 
The separation between their corresponding parameters is $\|\theta_j-\theta\|\leq \rho$. We know that with probability $1-\delta$, $\gamma_i$ does not satisfy the support invariance constraint. There exist interventional
distributions whose supports say $\mathcal{Z}$ and $\tilde{\mathcal{Z}}$, where $\gamma_{j}$'s minimum value on the former is at least $\eta$ higher than the minimum value on $\tilde{\mathcal{Z}}$. Let us now compute a lower bound on the minimum value of $\gamma$ on $\mathcal{Z}$. 
$$|\gamma_{\theta}(z) - \gamma_{\theta_j}(z)| \leq L \|\theta-\theta_j\|\leq L \rho \implies \gamma_{\theta}(z) \geq  \gamma_{\theta_j}( z)  - L\rho$$

From the above, it follows that 
$$\min_{z \in \mathcal{Z}}\gamma_{\theta}( z) \geq  \min_{z \in \mathcal{Z}}\gamma_{\theta_j}( z)  - L\rho$$ 

Next, we compute an upper bound on the minimum value of $\gamma$ on $\tilde{\mathcal{Z}}$

$$|\gamma_{\theta}( z) - \gamma_{\theta_j}(z)| \leq L \|\theta-\theta_j\|\leq L \rho \implies \gamma_{\theta}( z) \leq  \gamma_{\theta_j}( z)  + L\rho$$

From the above, it follows that 
$$\min_{z \in \tilde{\mathcal{Z}}}\gamma_{\theta}(z) \leq  \min_{z \in \tilde{\mathcal{Z}}}\gamma_{\theta_j}(z)  + L\rho$$ 

We now take the difference of the bounds above to arrive at the following. 

$$\min_{z \in \tilde{\mathcal{Z}}}\gamma_{\theta}(z) -  \min_{z \in \mathcal{Z}}\gamma_{\theta}( z) \geq \min_{z \in \mathcal{Z}}\gamma_{\theta_j}(z) - \min_{z \in \tilde{\mathcal{Z}}}\gamma_{\theta_j}( z)- 2L\rho \geq \eta - 2L\rho  = \frac{\eta}{2}$$

where we set $\rho = \eta/4L$ in the last inequality. Therefore, $\gamma$ does not satisfy support invariance. Note that the above argument is general and applies to every $\gamma \in \Theta$ since we can pick the corresponding nearest neighbor in the cover.

This completes the proof. 
\end{proof}

\subsection{Polytope Support}
\label{sec:poly}
In this section, we assume that the support of latents in each domain is characterized by bounded polytopes -- the convex hull of a finite number of vertices, where each vertex has a bounded norm. Under the assumptions and the constraint (Assumption~\ref{assm1: dgp1}  and Constraint~\ref{assm3: h_poly_new}) we know that $\hat{z}$ is an affine function of $z$. If the support of $z$ is a bounded polytope, then evaluating the maximum and minimum value that each component of $\hat{z}$ depends only on the vertices of the polytope following the fundamental theorem of linear programming. This allows us to provide identification guarantees by placing assumptions on the diversity of these polytopes, i.e., on these vertices, observed across domains.
.

Following Constraint~\ref{assm: supp_inv}, we equate the maximum value of components in $\hat{\mathcal{S}}$ across domains. Suppose we equate the maximum of $\hat{z}_i$ across domain $p$ and $q$. We obtain $A_i^{\top}(z^{\max,p} - z^{\max,q}) = 0$, where $z^{\max,p}$, $z^{\max,q}$ correspond to the vertex of the support polytope in domain $p$, $q$ respectively.  Observe how the expression depends on the difference of vertices from different polytopes. We define a set of matrices $\mathcal{M}$ formed by taking the difference of vertices from the different polytopes as follows. Firstly, we fix the first domain as the reference domain and we define difference vectors with respect to the vertices in this domain. We also fix some arbitrary ordering of vertices in the polytope; say they are in the increasing order of the first coordinate.  We start with the first vertex in the first domain.  Next, pick the second domain and pick its first vertex. Take the difference of the two selected vectors, this difference vector forms the first row of one of the matrices. Pick the third domain, take its first vertex, and again take the difference to get the second row of the matrix. Repeat this process for all the domains. As a result, we get a matrix with $k-1$ rows and $d$ columns.  To summarize, the set of matrices $\mathcal{M}$ consists of $k-1\times d$ matrices that satisfy the following condition. For each matrix $M\in \mathcal{M}$, the $r^{th}$ row of the matrix is defined as the difference between some vertex from $(r+1)^{th}$ domain and another vertex from the first domain.

\begin{assumption}
\label{assm4: div_support}
\begin{itemize}
\item The support of latents in each domain $p\in [k]$, i.e., $\mathcal{Z}^{(p)}$ is a bounded polytope; the number of domains $k\geq d+1$. Each matrix $M \in \mathcal{M}$ has a rank equal to the number of non-zero columns. 

\item For each component  $j\in \mathcal{U}$, there exists a domain $p \in [k]$ such that the following condition holds. We denote the value assumed by $z_j$ in $\mathcal{Z}^{(1)}$ on vertex $r$ as $v^{r}$. We assume that there exists another domain $p$ with support $\mathcal{Z}^{(p)}$ such that $z_j$ does not take the same value as $v^{r}$ at any vertex of $\mathcal{Z}^{(p)}$.

\end{itemize}    
\end{assumption}

 The first part of the above assumption states a simple regularity condition on matrix $M$.  The second part of the above assumption is also a simple regularity condition on components in $\mathcal{U}$. The condition only requires that the value attained at some vertex is not attained at any other vertex for some other domain. 

\paragraph{Further remarks on Assumption~\ref{assm4: div_support}.} Next,  we illustrate that the Assumption~\ref{assm4: div_support} holds rather easily in many settings. Consider a setting where $z=[z_1,z_2]$ and both $z_1$, $z_2$ take values between $0$ and $1$. We consider the setting where the support of $z$ forms a polytope. For each domain $p$, the polytope is sampled as follows. Each polytope consists of $M$ vertices and we sample $M$ values for $z_2$ uniformly at random $[0,1]$. For $z_1$, we sample $M-2$ vertices uniformly at random from the interval $[0,1]$. For the remaining $2$ vertices, we fix $z_1$ to take value $0$ on one of them and $1$ on the other. We generate $k$ polytopes following the above process and check if the rank constraint in the first part of Assumption  \ref{assm4: div_support} is satisfied. We repeat this process over ten thousand trials and find that the assumption always holds for different values of $M$ and $k$. The second part of the assumption holds trivially in the above case as two uniform random variables sampled independently from $[0,1]$ are not equal to probability one.

In what follows, we use the notation $a_{\mathcal{B}}$ to denote a vector formed by components of $a$ whose indices in $a$ are from the set $\mathcal{B}$.

\begin{theorem}
    Suppose the data is generated from different domains following
equation  \eqref{eqn: dgp} such that Assumptions  \ref{assm1: dgp1},  \ref{assm: supp_invar}, \ref{assm4: div_support} are satisfied. The autoencoder that solves the reconstruction
identity in equation~\eqref{eqn: recons} under Constraint~\ref{assm3: h_poly_new}, \ref{assm: supp_inv} satisfies 

$$\hat{z}_{\hat{\mathcal{S}}} = Dz_{\mathcal{S}} + e$$

where $D \in \mathbb{R}^{|\hat{\mathcal{S}}| \times |\mathcal{S}|}$, $e \in \mathbb{R}^{|\hat{\mathcal{S}}|} $.
\end{theorem}

\begin{proof}

We begin by first checking that the solution to reconstruction identity under the above-said constraints exists. Set $f=g^{-1}$ and $h=g$ and $\hat{\mathcal{S}}=\mathcal{S}$. The reconstruction identity and  Constraint~\ref{assm: supp_inv} is satisfied as Assumption~\ref{assm: supp_invar} holds.

Consider a component $m \in \hat{\mathcal{S}}$. From Constraint~\ref{assm: supp_inv}, we know that the support of $\hat{z}_m$ does not change. From Theorem~\ref{thm1}, we also know that there is an affine relationship between $\hat{z}$ and $z$. Therefore, we can write

\begin{equation}
    \hat{z}_{m} = A^{\top}_m z + c_m
\end{equation}

The support of $\hat{z}_m $ is determined by the maximum and minimum of $A^{\top}_m z + c_m$ computed on the respective domains. Let us compute the maximum and minimum of $\hat{z}_m$ in domain $p$ as follows. 
\begin{equation}
\begin{split}
z^{\max}(A_{m}, p) = \arg\max_{z \in \mathcal{Z}^{(p)}} A^{\top}_m z + c_m 
\end{split}
\end{equation}

\begin{equation}
\begin{split}
z^{\min}(A_m, p) =  \arg\min_{z \in \mathcal{Z}^{(p)}} A^{\top}_m z + c_m 
\end{split}
\end{equation}


We define a vector $A_{m}^{\mathcal{U}}$ that contains components of $A_m$ whose indices in $A_m$ form the set $\mathcal{U}$. We now show that support invariance constraints in Constraint~\ref{assm: supp_inv} implies that $A_{m}^{\mathcal{U}}=0$.  Suppose $A_{m}^{\mathcal{U}}\not=0$ (at least one element of this vector is non-zero). In this case, we write the maximum value of the objective as
$$  \sum_{l}A_{ml} z^{\max}(A_m,p) $$
Due to the support invariance constraint we get 

$$  \sum_{l}A_{ml} z^{\max}_{l}(A_m,p)  = \sum_{l}A_{ml} z^{\max}_{l}(A_m,1) $$
$$\sum_{l}A_{ml} \big(z^{\max}_{l}(A_m,p)  - z^{\max}_{l}(A_m,1)\big)  = 0$$
$$ z_{\mathrm{diff}}^{\top}(A_m,p) A_m = 0$$

$ z_{\mathrm{diff}}^{\top}(A_m,p)$ is the difference vector formed by taking the difference $z^{\max}_{l}(A_m,p)  - z^{\max}_{l}(A_m,1)$. Construct a matrix $Z_{\mathrm{diff}}(A_m) \in \mathbb{R}^{k-1\times l}$ by stacking the difference vectors $z_{\mathrm{diff}}^{\top}(A_m,p)$ for all $p$ in $\{2,\cdots, k\}$.

 Let us consider the largest submatrix of $Z_{\mathrm{diff}}(A_m)$ with no zero columns and denote it as $Z_{\mathrm{diff}}^{s}(A_m)$. Following Assumption~\ref{assm4: div_support}, $Z_{\mathrm{diff}}^{s}(A_m)$ has a full column rank. Therefore, a non-trivial solution to $Z_{\mathrm{diff}}^{s}(A_m)v=0$ does not exist and thus $v=0$. 
Consider an element $j \in \mathcal{U}$. Due to Assumption~\ref{assm4: div_support}, the column in $Z_{\mathrm{diff}}(A_m)$ corresponding to $j$ is non-zero. Therefore, for each element $j \in \mathcal{U}$ the corresponding columns in $Z_{\mathrm{diff}}(A_m)$ are non-zero. The columns of $Z_{\mathrm{diff}}^{s}(A_m)$ contain all coefficients in $\mathcal{U}$, which implies $A_{m}^{\mathcal{U}}=0$.    This completes the proof.

\end{proof}
\section{Experiments}
In this section, we provide additional experimental results and other additional details for the experiments. 

\subsection{Linear Mixing}
\label{appendix:linear_mixing}

\subsubsection{Data Generation and Model Architecture}
For both linear and polynomial $g(\cdot)$ we sample $z=(z_{\mathcal{S}},z_{\mathcal{U}})^\top$ as follows ($z_{\mathcal{S}}, z_{\mathcal{U}} \in R^{d/2}$). We sample $z_{\mathcal{S}} \sim \text{Uniform}[0,1]$ across all domains. For domain $i\in [k]$ we sample $z_{\mathcal{U}}\sim \text{Uniform}[l^i, h^i]$, where $l^i, h^i\sim \text{Uniform}[-5,5]$. Then for the Independent SCM setting, we obtain the observational data via $x=Az$, where $A\in R^{n\times d}$ is a full-rank random matrix whose entries are drawn from $\text{Uniform}[0,1]$. We obtain a Dynamic SCM by altering the above $z$ as follows. For each sample, $z_{\mathcal{U}}^j$ will be offset by the $z_{\mathcal{S}}^j$ with probability $p$ and remain unchanged otherwise. In our experiments we set $p$ to 0.5. We generate 10000 samples for the training split and 2000 for the validation split. The test and validation sets are the same, since we do not search over hyperparameters (see Section \ref{appendix: hp_search}).\\
For Linear mixing $g(\cdot)$, stages 1, 2 are carried out simultaneously by a linear autoencoder that is jointly optimized with reconstruction objective and invariance penalty.

\subsubsection{Results}
In Table~\ref{table: linear_indep_supp}, \ref{table: linear_dynamic_supp}  we provide additional results when $z$'s follow an independent and dynamic SCM as described in the main body respectively. We conducted experiments for different values of the dimension of the underlying latents $d$ and for a varying number of domains $k$. We divide the table into three sections with top block corresponding to Min-Max penalty, the middle block corresponding to the MMD penalty and the bottom block corresponding to the combination of the two denoted Min-Max + MMD. Across the different settings, we observe that as the number of domains increase we achieve high $R^2_{\mathcal{S}}$ and low  $R^2_{\mathcal{U}}$. 

\begin{table}[h]
  \centering{
    {\caption{Linear Mixing Dataset $R^2$ scores, \textbf{Independent SCM} DGP. The results are averaged over 5 seeds. $\hat{z},z\in R^{d}$ and $z_\mathcal{S}, z_{\mathcal{U}}\in R^{d/2}$ and $x=g(z)\in R^{2d}$. The three sections from top to bottom correspond to Min-Max, MMD, and the combination.}
         \label{table: linear_indep_supp}
        }%
      {
        \begin{tabular}{c cccc cccc}
          \hline
          \multicolumn{1}{c}{}
        & \multicolumn{4}{c}{$R^2_{\mathcal{S}}$}
& \multicolumn{4}{c}{$R^2_{\mathcal{U}}$}  \\ \cmidrule(lr){2-5} \cmidrule(lr){6-9}
          \multicolumn{1}{c}{$d$} & \multicolumn{1}{c}{$k=2$}  & \multicolumn{1}{c}{$k=4$}   & \multicolumn{1}{c}{$k=8$} & $k=16$ & \multicolumn{1}{c}{$k=2$}  & \multicolumn{1}{c}{$k=4$} & \multicolumn{1}{c}{$k=8$} & $k=16$ \\ \hline
          \multirow{1}{*}{$8$}  & \multicolumn{1}{c}{0.45$\pm0.02$} & \multicolumn{1}{c}{0.99$\pm0.00$} & 0.93$\pm0.04$ & \multicolumn{1}{c}{0.98$\pm0.01$} & \multicolumn{1}{c}{0.30$\pm0.03$} & 0.01$\pm0.00$ & 0.08$\pm0.04$ & 0.03$\pm0.01$ \\
  
          \multirow{1}{*}{$16$} & \multicolumn{1}{c}{0.40$\pm0.01$} & \multicolumn{1}{c}{0.81$\pm0.02$} & 0.95$\pm0.02$ & \multicolumn{1}{c}{0.91$\pm0.03$} & \multicolumn{1}{c}{0.30$\pm0.03$} & 0.10$\pm0.02$ & 0.03$\pm0.01$ & 0.11$\pm0.03$ \\
          \multirow{1}{*}{$32$} & \multicolumn{1}{c}{0.35$\pm0.00$} & \multicolumn{1}{c}{0.51$\pm0.03$}     & 0.80$\pm0.01$ & \multicolumn{1}{c}{0.88$\pm0.01$} & \multicolumn{1}{c}{0.45$\pm0.05$} & 0.24$\pm0.02$ & 0.13$\pm0.02$ & 0.12$\pm0.02$ \\
          \multirow{1}{*}{$64$} & \multicolumn{1}{c}{0.32$\pm0.00$} & \multicolumn{1}{c}{0.35$\pm0.00$}     & 0.63$\pm0.01$ & \multicolumn{1}{c}{0.80$\pm0.00$} & \multicolumn{1}{c}{0.52$\pm0.04$} & 0.35$\pm0.01$ & 0.19$\pm0.00$ & 0.15$\pm0.00$ \\\hline\hline
          \multirow{1}{*}{$8$}  & \multicolumn{1}{c}{0.64$\pm0.05$} & \multicolumn{1}{c}{0.94$\pm0.03$} & 1.00$\pm0.00$ & \multicolumn{1}{c}{0.99$\pm0.00$} & \multicolumn{1}{c}{0.19$\pm0.01$} & 0.03$\pm0.01$ & 0.01$\pm0.00$ & 0.01$\pm0.00$ \\
          \multirow{1}{*}{$16$} & \multicolumn{1}{c}{0.51$\pm0.04$} & \multicolumn{1}{c}{0.77$\pm0.01$} & 0.92$\pm0.01$ & \multicolumn{1}{c}{0.97$\pm0.01$} & \multicolumn{1}{c}{0.09$\pm0.04$} & 0.10$\pm0.02$ & 0.09$\pm0.04$ & 0.06$\pm0.01$ \\
          \multirow{1}{*}{$32$} & \multicolumn{1}{c}{0.38$\pm0.02$} & \multicolumn{1}{c}{0.60$\pm0.01$}     & 0.75$\pm0.03$ & \multicolumn{1}{c}{0.90$\pm0.01$} & \multicolumn{1}{c}{0.34$\pm0.04$} & 0.17$\pm0.00$ & 0.12$\pm0.02$ & 0.07$\pm0.01$ \\
          \multirow{1}{*}{$64$} & \multicolumn{1}{c}{0.30$\pm0.01$} & \multicolumn{1}{c}{0.34$\pm0.01$}     & 0.57$\pm0.02$ & \multicolumn{1}{c}{0.75$\pm0.00$} & \multicolumn{1}{c}{0.44$\pm0.02$} & 0.27$\pm0.01$ & 0.19$\pm0.01$ & 0.12$\pm0.00$ \\\hline\hline
          \multirow{1}{*}{$8$}  & \multicolumn{1}{c}{0.63$\pm0.06$} & \multicolumn{1}{c}{0.99$\pm0.00$} & 0.99$\pm0.00$ & \multicolumn{1}{c}{0.99$\pm0.00$} & \multicolumn{1}{c}{0.19$\pm0.01$} & 0.01$\pm0.00$ & 0.02$\pm0.00$ & 0.02$\pm0.00$ \\
          \multirow{1}{*}{$16$} & \multicolumn{1}{c}{0.53$\pm0.02$} & \multicolumn{1}{c}{0.82$\pm0.01$} & 0.93$\pm0.02$ & \multicolumn{1}{c}{0.97$\pm0.00$} & \multicolumn{1}{c}{0.24$\pm0.04$} & 0.10$\pm0.02$ & 0.06$\pm0.02$ & 0.04$\pm0.00$ \\
          \multirow{1}{*}{$32$} & \multicolumn{1}{c}{0.39$\pm0.02$} & \multicolumn{1}{c}{0.64$\pm0.04$}     & 0.81$\pm0.01$ & \multicolumn{1}{c}{0.91$\pm0.01$} & \multicolumn{1}{c}{0.38$\pm0.06$} & 0.17$\pm0.03$ & 0.11$\pm0.00$ & 0.08$\pm0.01$ \\
          \multirow{1}{*}{$64$} & \multicolumn{1}{c}{0.33$\pm0.00$} & \multicolumn{1}{c}{0.39$\pm0.00$}     & 0.64$\pm0.01$ & \multicolumn{1}{c}{0.80$\pm0.00$} & \multicolumn{1}{c}{0.44$\pm0.03$} & 0.30$\pm0.02$ & 0.16$\pm0.01$ & 0.12$\pm0.01$ \\\hline
        \end{tabular}
      }
  }
\end{table}

\begin{table}[h]
  \centering{
    {\caption{Linear Mixing Dataset $R^2$ scores, \textbf{Dynamic SCM} DGP. The results are averaged over 5 seeds. $\hat{z},z\in R^{d}$ and $z_\mathcal{S}, z_{\mathcal{U}}\in R^{d/2}$ and $x=g(z)\in R^{2d}$. The three sections from top to bottom correspond to Min-Max, MMD, and the combination.}
       \label{table: linear_dynamic_supp}
        }%
      {
        \begin{tabular}{c cccc cccc}
          \hline
          \multicolumn{1}{c}{}
        & \multicolumn{4}{c}{$R^2_{\mathcal{S}}$}
& \multicolumn{4}{c}{$R^2_{\mathcal{U}}$}  \\ \cmidrule(lr){2-5} \cmidrule(lr){6-9}
          \multicolumn{1}{c}{$d$} & \multicolumn{1}{c}{$k=2$}  & \multicolumn{1}{c}{$k=4$}   & \multicolumn{1}{c}{$k=8$} & $k=16$ & \multicolumn{1}{c}{$k=2$}  & \multicolumn{1}{c}{$k=4$} & \multicolumn{1}{c}{$k=8$} & $k=16$ \\ \hline
          \multirow{1}{*}{$8$}  & \multicolumn{1}{c}{0.48$\pm0.01$} & \multicolumn{1}{c}{0.98$\pm0.00$} & 0.97$\pm0.01$ & \multicolumn{1}{c}{0.95$\pm0.01$} & \multicolumn{1}{c}{0.30$\pm0.02$} & 0.02$\pm0.00$ & 0.03$\pm0.01$ & 0.03$\pm0.01$ \\
          \multirow{1}{*}{$16$} & \multicolumn{1}{c}{0.43$\pm0.04$} & \multicolumn{1}{c}{0.73$\pm0.02$} & 0.93$\pm0.02$ & \multicolumn{1}{c}{0.98$\pm0.00$} & \multicolumn{1}{c}{0.33$\pm0.05$} & 0.14$\pm0.01$ & 0.06$\pm0.00$ & 0.03$\pm0.00$ \\
          \multirow{1}{*}{$32$} & \multicolumn{1}{c}{0.35$\pm0.00$} & \multicolumn{1}{c}{0.51$\pm0.02$}     & 0.81$\pm0.00$ & \multicolumn{1}{c}{0.89$\pm0.00$} & \multicolumn{1}{c}{0.38$\pm0.05$} & 0.24$\pm0.01$ & 0.14$\pm0.01$ & 0.11$\pm0.00$ \\
          \multirow{1}{*}{$64$} & \multicolumn{1}{c}{0.27$\pm0.01$} & \multicolumn{1}{c}{0.34$\pm0.00$}     & 0.63$\pm0.01$ & \multicolumn{1}{c}{0.80$\pm0.00$} & \multicolumn{1}{c}{0.55$\pm0.02$} & 0.33$\pm0.01$ & 0.19$\pm0.00$ & 0.14$\pm0.01$ \\\hline\hline
          \multirow{1}{*}{$8$}  & \multicolumn{1}{c}{0.60$\pm0.06$} & \multicolumn{1}{c}{0.92$\pm0.03$} & 0.99$\pm0.00$ & \multicolumn{1}{c}{0.99$\pm0.00$} & \multicolumn{1}{c}{0.25$\pm0.04$} & 0.05$\pm0.01$ & 0.02$\pm0.00$ & 0.02$\pm0.00$ \\
          \multirow{1}{*}{$16$} & \multicolumn{1}{c}{0.46$\pm0.03$} & \multicolumn{1}{c}{0.74$\pm0.01$} & 0.92$\pm0.02$ & \multicolumn{1}{c}{0.98$\pm0.01$} & \multicolumn{1}{c}{0.29$\pm0.04$} & 0.13$\pm0.02$ & 0.04$\pm0.00$ & 0.04$\pm0.01$ \\
          \multirow{1}{*}{$32$} & \multicolumn{1}{c}{0.35$\pm0.02$} & \multicolumn{1}{c}{0.56$\pm0.01$}     & 0.74$\pm0.03$ & \multicolumn{1}{c}{0.91$\pm0.01$} & \multicolumn{1}{c}{0.37$\pm0.04$} & 0.20$\pm0.01$ & 0.11$\pm0.01$ & 0.08$\pm0.02$ \\
          \multirow{1}{*}{$64$} & \multicolumn{1}{c}{0.28$\pm0.00$} & \multicolumn{1}{c}{0.31$\pm0.01$}     & 0.53$\pm0.01$ & \multicolumn{1}{c}{0.72$\pm0.00$} & \multicolumn{1}{c}{0.47$\pm0.02$} & 0.30$\pm0.01$ & 0.20$\pm0.00$ & 0.16$\pm0.00$ \\\hline\hline
          \multirow{1}{*}{$8$}  & \multicolumn{1}{c}{0.61$\pm0.05$} & \multicolumn{1}{c}{0.98$\pm0.00$} & 0.99$\pm0.00$ & \multicolumn{1}{c}{0.99$\pm0.00$} & \multicolumn{1}{c}{0.21$\pm0.01$} & 0.02$\pm0.00$ & 0.02$\pm0.00$ & 0.02$\pm0.00$ \\
          \multirow{1}{*}{$16$} & \multicolumn{1}{c}{0.51$\pm0.01$} & \multicolumn{1}{c}{0.80$\pm0.02$} & 0.96$\pm0.01$ & \multicolumn{1}{c}{0.98$\pm0.01$} & \multicolumn{1}{c}{0.28$\pm0.05$} & 0.10$\pm0.02$ & 0.04$\pm0.00$ & 0.04$\pm0.00$ \\
          \multirow{1}{*}{$32$} & \multicolumn{1}{c}{0.36$\pm0.02$} & \multicolumn{1}{c}{0.65$\pm0.03$}     & 0.82$\pm0.01$ & \multicolumn{1}{c}{0.91$\pm0.00$} & \multicolumn{1}{c}{0.35$\pm0.05$} & 0.18$\pm0.02$ & 0.09$\pm0.01$ & 0.09$\pm0.00$ \\
          \multirow{1}{*}{$64$} & \multicolumn{1}{c}{0.31$\pm0.00$} & \multicolumn{1}{c}{0.38$\pm0.00$}     & 0.64$\pm0.01$ & \multicolumn{1}{c}{0.81$\pm0.00$} & \multicolumn{1}{c}{0.48$\pm0.03$} & 0.31$\pm0.01$ & 0.18$\pm0.00$ & 0.11$\pm0.00$ \\\hline
        \end{tabular}
      }
  }
\end{table}

\begin{table}[hpbt]
  \centering{
    {\caption{$\beta$-VAE - Linear Mixing Dataset $R^2$ scores, \textbf{Independent SCM} DGP. The results are averaged over 5 seeds. Each set of 4 rows correspond to a specific $d$ and from top to bottom denote the $R^2$ scores before training $\beta$-VAE, and after training with $\beta\in [0.1,1.0,10.0]$. $\hat{z},z\in R^{d}$ and $z_\mathcal{S}, z_{\mathcal{U}}\in R^{d/2}$ and $x=g(z)\in R^{2d}$.}
        \label{table:beta_linear_indep_scm}
        }%
      {
        \begin{tabular}{c cccc cccc}
          \hline
          \multicolumn{1}{c}{}
        & \multicolumn{4}{c}{$R^2_{\mathcal{S}}$}
& \multicolumn{4}{c}{$R^2_{\mathcal{U}}$}  \\ \cmidrule(lr){2-5} \cmidrule(lr){6-9}
          \multicolumn{1}{c}{$d$} & \multicolumn{1}{c}{$k=2$}  & \multicolumn{1}{c}{$k=4$}   & \multicolumn{1}{c}{$k=8$} & $k=16$ & \multicolumn{1}{c}{$k=2$}  & \multicolumn{1}{c}{$k=4$} & \multicolumn{1}{c}{$k=8$} & $k=16$ \\ \hline
          \multirow{4}{*}{$8$}  & \multicolumn{1}{c}{0.11$\pm0.08$} & \multicolumn{1}{c}{$-0.21\pm0.25$} & $-0.59\pm0.31$ & \multicolumn{1}{c}{$-0.57\pm0.30$} & \multicolumn{1}{c}{0.69$\pm0.16$} & 0.81$\pm0.05$ & 0.78$\pm0.04$ & 0.81$\pm0.03$ \\
          \multirow{4}{*}{}  & \multicolumn{1}{c}{0.15$\pm0.09$} & \multicolumn{1}{c}{$-0.50\pm0.40$} & $-0.53\pm0.41$ & \multicolumn{1}{c}{$-0.31\pm0.32$} & \multicolumn{1}{c}{0.83$\pm0.08$} & 0.85$\pm0.04$ & 0.77$\pm0.05$ & 0.80$\pm0.06$ \\
          \multirow{4}{*}{}  & \multicolumn{1}{c}{0.17$\pm0.06$} & \multicolumn{1}{c}{$-0.86\pm0.40$} & $-0.84\pm0.51$ & \multicolumn{1}{c}{$-0.92\pm0.47$} & \multicolumn{1}{c}{0.87$\pm0.04$} & 0.86$\pm0.03$ & 0.85$\pm0.03$ & 0.84$\pm0.03$ \\
          \multirow{4}{*}{}  & \multicolumn{1}{c}{0.05$\pm0.08$} & \multicolumn{1}{c}{$-0.56\pm0.30$} & $-0.95\pm0.35$ & \multicolumn{1}{c}{$-1.24\pm0.39$} & \multicolumn{1}{c}{0.81$\pm0.10$} & 0.92$\pm0.01$ & 0.85$\pm0.04$ & 0.90$\pm0.02$ \\ 
          \midrule 
          \multirow{4}{*}{$16$} & \multicolumn{1}{c}{0.42$\pm0.29$} & \multicolumn{1}{c}{$-0.61\pm0.23$} & $-1.46\pm0.26$ & \multicolumn{1}{c}{$-1.57\pm0.43$} & \multicolumn{1}{c}{0.31$\pm0.27$} & 0.63$\pm0.04$ & 0.49$\pm0.03$ & 0.51$\pm0.02$ \\
          \multirow{4}{*}{}  & \multicolumn{1}{c}{0.19$\pm0.03$} & \multicolumn{1}{c}{0.04$\pm0.04$} & $-0.24\pm0.26$ & \multicolumn{1}{c}{$-0.49\pm0.42$} & \multicolumn{1}{c}{0.83$\pm0.07$} & 0.92$\pm0.02$ & 0.90$\pm0.01$ & 0.90$\pm0.00$ \\
          \multirow{4}{*}{}  & \multicolumn{1}{c}{0.19$\pm0.03$} & \multicolumn{1}{c}{$-0.03\pm0.22$} & $-0.42\pm0.29$ & \multicolumn{1}{c}{$-0.60\pm0.25$} & \multicolumn{1}{c}{0.86$\pm0.05$} & 0.94$\pm0.02$ & 0.92$\pm0.01$ & 0.90$\pm0.00$ \\
          \multirow{4}{*}{}  & \multicolumn{1}{c}{0.15$\pm0.06$} & \multicolumn{1}{c}{$-0.19\pm0.30$} & $-0.09\pm0.07$ & \multicolumn{1}{c}{$-0.52\pm0.29$} & \multicolumn{1}{c}{0.94$\pm0.01$} & 0.96$\pm0.01$ & 0.94$\pm0.00$ & 0.94$\pm0.01$ \\
          \midrule
          \multirow{4}{*}{$32$} & \multicolumn{1}{c}{0.35$\pm0.17$} & \multicolumn{1}{c}{$-0.49\pm0.23$}     & $-1.10\pm0.48$ & \multicolumn{1}{c}{$-1.24\pm0.32$} & \multicolumn{1}{c}{0.25$\pm0.17$} & 0.41$\pm0.04$ & 0.35$\pm0.02$ & 0.28$\pm0.01$ \\
          \multirow{4}{*}{}  & \multicolumn{1}{c}{0.23$\pm0.01$} & \multicolumn{1}{c}{0.10$\pm0.03$} & 0.06$\pm0.03$ & \multicolumn{1}{c}{0.01$\pm0.04$} & \multicolumn{1}{c}{0.91$\pm0.02$} & 0.92$\pm0.00$ & 0.92$\pm0.00$ & 0.92$\pm0.00$ \\
          \multirow{4}{*}{}  & \multicolumn{1}{c}{0.20$\pm0.01$} & \multicolumn{1}{c}{0.12$\pm0.02$} & 0.06$\pm0.01$ & \multicolumn{1}{c}{0.02$\pm0.02$} & \multicolumn{1}{c}{0.94$\pm0.01$} & 0.95$\pm0.00$ & 0.95$\pm0.00$ & 0.93$\pm0.00$ \\
          \multirow{4}{*}{}  & \multicolumn{1}{c}{0.22$\pm0.01$} & \multicolumn{1}{c}{0.13$\pm0.02$} & 0.06$\pm0.03$ & \multicolumn{1}{c}{$-0.11\pm0.17$} & \multicolumn{1}{c}{0.93$\pm0.01$} & 0.94$\pm0.00$ & 0.95$\pm0.00$ & 0.94$\pm0.00$ \\
          \midrule
          \multirow{4}{*}{$64$} & \multicolumn{1}{c}{0.30$\pm0.14$} & \multicolumn{1}{c}{$-0.52\pm0.41$}     & $-0.72\pm0.31$ & \multicolumn{1}{c}{$-0.85\pm0.33$} & \multicolumn{1}{c}{0.26$\pm0.09$} & 0.33$\pm0.05$ & 0.27$\pm0.02$ & 0.19$\pm0.01$ \\
          \multirow{4}{*}{}  & \multicolumn{1}{c}{0.25$\pm0.01$} & \multicolumn{1}{c}{0.18$\pm0.00$} & 0.14$\pm0.01$ & \multicolumn{1}{c}{0.10$\pm0.01$} & \multicolumn{1}{c}{0.91$\pm0.01$} & 0.95$\pm0.00$ & 0.95$\pm0.00$ & 0.95$\pm0.00$ \\
          \multirow{4}{*}{}  & \multicolumn{1}{c}{0.24$\pm0.02$} & \multicolumn{1}{c}{0.18$\pm0.01$} & 0.13$\pm0.02$ & \multicolumn{1}{c}{0.07$\pm0.03$} & \multicolumn{1}{c}{0.91$\pm0.01$} & 0.94$\pm0.00$ & 0.95$\pm0.00$ & 0.94$\pm0.00$ \\
          \multirow{4}{*}{}  & \multicolumn{1}{c}{0.27$\pm0.02$} & \multicolumn{1}{c}{0.23$\pm0.01$} & 0.16$\pm0.01$ & \multicolumn{1}{c}{0.11$\pm0.03$} & \multicolumn{1}{c}{0.89$\pm0.01$} & 0.92$\pm0.00$ & 0.93$\pm0.00$ & 0.92$\pm0.00$ \\\hline
        \end{tabular}
      }
  }
\end{table}

\begin{table}[hpbt]
  \centering{
    {\caption{$\beta$-VAE - Linear Mixing Dataset $R^2$ scores, \textbf{Dynamic SCM} DGP. The results are averaged over 5 seeds. Each set of 4 rows correspond to a specific $d$ and from top to bottom denote the $R^2$ scores before training $\beta$-VAE, and after training with $\beta\in [0.1,1.0,10.0]$. $\hat{z},z\in R^{d}$ and $z_\mathcal{S}, z_{\mathcal{U}}\in R^{d/2}$ and $x=g(z)\in R^{2d}$.}
        \label{table:beta_linear_dynamic_scm}
        }%
      {
        \begin{tabular}{c cccc cccc}
          \hline
          \multicolumn{1}{c}{}
        & \multicolumn{4}{c}{$R^2_{\mathcal{S}}$}
& \multicolumn{4}{c}{$R^2_{\mathcal{U}}$}  \\ \cmidrule(lr){2-5} \cmidrule(lr){6-9}
          \multicolumn{1}{c}{$d$} & \multicolumn{1}{c}{$k=2$}  & \multicolumn{1}{c}{$k=4$}   & \multicolumn{1}{c}{$k=8$} & $k=16$ & \multicolumn{1}{c}{$k=2$}  & \multicolumn{1}{c}{$k=4$} & \multicolumn{1}{c}{$k=8$} & $k=16$ \\ \hline
          \multirow{4}{*}{$8$}  & \multicolumn{1}{c}{0.08$\pm0.10$} & \multicolumn{1}{c}{$-0.23\pm0.32$} & $-0.54\pm0.30$ & \multicolumn{1}{c}{$-0.51\pm0.27$} & \multicolumn{1}{c}{0.73$\pm0.13$} & 0.81$\pm0.05$ & 0.79$\pm0.04$ & 0.81$\pm0.04$ \\
          \multirow{4}{*}{}  & \multicolumn{1}{c}{0.26$\pm0.02$} & \multicolumn{1}{c}{$-0.44\pm0.39$} & $-0.35\pm0.27$ & \multicolumn{1}{c}{$-0.22\pm0.57$} & \multicolumn{1}{c}{0.85$\pm0.05$} & 0.85$\pm0.04$ & 0.77$\pm0.05$ & 0.82$\pm0.06$ \\
          \multirow{4}{*}{}  & \multicolumn{1}{c}{0.17$\pm0.09$} & \multicolumn{1}{c}{$-0.61\pm0.34$} & $-0.44\pm0.32$ & \multicolumn{1}{c}{$-0.74\pm0.39$} & \multicolumn{1}{c}{0.88$\pm0.03$} & 0.87$\pm0.03$ & 0.84$\pm0.02$ & 0.83$\pm0.04$ \\
          \multirow{4}{*}{}  & \multicolumn{1}{c}{$-0.17\pm0.33$} & \multicolumn{1}{c}{$-0.52\pm0.51$} & $-0.71\pm1.07$ & \multicolumn{1}{c}{$-1.04\pm0.33$} & \multicolumn{1}{c}{0.85$\pm0.04$} & 0.73$\pm0.19$ & 0.89$\pm0.04$ & 0.90$\pm0.03$ \\ 
          \midrule
          \multirow{4}{*}{$16$} & \multicolumn{1}{c}{$-0.53\pm0.43$} & \multicolumn{1}{c}{$-0.56\pm0.54$} & $-1.18\pm0.44$ & \multicolumn{1}{c}{$-1.23\pm0.32$} & \multicolumn{1}{c}{0.33$\pm0.28$} & 0.62$\pm0.04$ & 0.50$\pm0.03$ & 0.51$\pm0.02$ \\
          \multirow{4}{*}{}  & \multicolumn{1}{c}{0.19$\pm0.03$} & \multicolumn{1}{c}{0.04$\pm0.03$} & $-0.15\pm0.16$ & \multicolumn{1}{c}{$-0.32\pm0.31$} & \multicolumn{1}{c}{0.85$\pm0.05$} & 0.91$\pm0.03$ & 0.90$\pm0.02$ & 0.91$\pm0.01$ \\
          \multirow{4}{*}{}  & \multicolumn{1}{c}{0.17$\pm0.09$} & \multicolumn{1}{c}{0.00$\pm0.08$} & $-0.37\pm0.31$ & \multicolumn{1}{c}{$-0.51\pm0.49$} & \multicolumn{1}{c}{0.88$\pm0.03$} & 0.94$\pm0.01$ & 0.93$\pm0.01$ & 0.92$\pm0.01$ \\
          \multirow{4}{*}{}  & \multicolumn{1}{c}{0.13$\pm0.08$} & \multicolumn{1}{c}{$-0.13\pm0.32$} & $-0.15\pm0.16$ & \multicolumn{1}{c}{$-0.56\pm0.28$} & \multicolumn{1}{c}{0.94$\pm0.01$} & 0.95$\pm0.01$ & 0.93$\pm0.00$ & 0.93$\pm0.00$ \\ 
          \midrule
          \multirow{4}{*}{$32$} & \multicolumn{1}{c}{$-0.33\pm0.40$} & \multicolumn{1}{c}{$-0.35\pm0.11$}     & $-0.79\pm0.36$ & \multicolumn{1}{c}{$-0.86\pm0.21$} & \multicolumn{1}{c}{0.29$\pm0.15$} & 0.41$\pm0.03$ & 0.35$\pm0.02$ & 0.29$\pm0.01$ \\
          \multirow{4}{*}{}  & \multicolumn{1}{c}{0.22$\pm0.01$} & \multicolumn{1}{c}{0.12$\pm0.02$} & 0.05$\pm0.02$ & \multicolumn{1}{c}{$-0.01\pm0.12$} & \multicolumn{1}{c}{0.91$\pm0.02$} & 0.93$\pm0.00$ & 0.93$\pm0.00$ & 0.92$\pm0.00$ \\
          \multirow{4}{*}{}  & \multicolumn{1}{c}{0.20$\pm0.01$} & \multicolumn{1}{c}{0.12$\pm0.02$} & 0.07$\pm0.00$ & \multicolumn{1}{c}{$-0.04\pm0.10$} & \multicolumn{1}{c}{0.93$\pm0.01$} & 0.94$\pm0.00$ & 0.95$\pm0.00$ & 0.93$\pm0.01$ \\
          \multirow{4}{*}{}  & \multicolumn{1}{c}{0.23$\pm0.01$} & \multicolumn{1}{c}{0.15$\pm0.02$} & 0.10$\pm0.02$ & \multicolumn{1}{c}{0.03$\pm0.04$} & \multicolumn{1}{c}{0.93$\pm0.01$} & 0.93$\pm0.01$ & 0.94$\pm0.00$ & 0.90$\pm0.01$ \\
          \midrule
          \multirow{4}{*}{$64$} & \multicolumn{1}{c}{$-0.16\pm0.11$} & \multicolumn{1}{c}{$-0.31\pm0.22$}     & $-0.47\pm0.20$ & \multicolumn{1}{c}{$-0.53\pm0.16$} & \multicolumn{1}{c}{0.34$\pm0.05$} & 0.35$\pm0.04$ & 0.27$\pm0.02$ & 0.19$\pm0.01$ \\
          \multirow{4}{*}{}  & \multicolumn{1}{c}{0.24$\pm0.01$} & \multicolumn{1}{c}{0.19$\pm0.00$} & 0.15$\pm0.01$ & \multicolumn{1}{c}{0.11$\pm0.01$} & \multicolumn{1}{c}{0.90$\pm0.01$} & 0.95$\pm0.00$ & 0.95$\pm0.00$ & 0.94$\pm0.00$ \\
          \multirow{4}{*}{}  & \multicolumn{1}{c}{0.20$\pm0.01$} & \multicolumn{1}{c}{0.20$\pm0.01$} & 0.15$\pm0.01$ & \multicolumn{1}{c}{0.10$\pm0.03$} & \multicolumn{1}{c}{0.93$\pm0.01$} & 0.94$\pm0.00$ & 0.94$\pm0.00$ & 0.93$\pm0.00$ \\
          \multirow{4}{*}{}  & \multicolumn{1}{c}{0.29$\pm0.01$} & \multicolumn{1}{c}{0.23$\pm0.01$} & 0.18$\pm0.01$ & \multicolumn{1}{c}{0.16$\pm0.03$} & \multicolumn{1}{c}{0.88$\pm0.01$} & 0.92$\pm0.00$ & 0.92$\pm0.00$ & 0.89$\pm0.01$ \\\hline
        \end{tabular}
      }
  }
\end{table}

\subsection{Polynomial Mixing}
\label{appendix:polynomial_mixing}

\subsubsection{Data Generation and Model Architecture}
The latents $z$ are sampled identical to the procedure for linear mixing dataset and the details of the polynomial mixing function $g(\cdot)$ are found in Assumption \ref{assm1: dgp1}. For stage 1, we use a polynomial autoencoder as follows. The encoder architecture is given in Table \ref{table: poly_ae_architecture_supp} where $n,d$ denote the dimensions of $x,z$, respectively. For decoding the outputs of the above encoder $\hat{z}$, we use a polynomial decoder which takes $\hat{z}$ and follows the procedure explained in Assumption \ref{assm1: dgp1}, where the coefficient matrix $G$ is to be learned and is parameterized by a linear layer.

\begin{table}[hpbt]
  \centering{
    {\caption{Polynomial Encoder.}
        \label{table: poly_ae_architecture_supp}}
      {
        \begin{tabular}{ccccc}
          \hline
          Layer      & Input Size & Output Size & Bias  & Activation \\ \hline
          Linear (1) & $n$         & $n/2$          & False  & LeakyReLU(0.5)       \\
          Linear (2) & $n/2$         & $n/2$          & False  & LeakyReLU(0.5)       \\
          Linear (3) & $n/2$         & $d$          & False & -       \\\hline
        \end{tabular}
      }}
\end{table}

\subsubsection{Results}
In Tables~\ref{table: poly_indep_supp_minmax}, \ref{table: poly_indep_supp_mmd}, \ref{table: poly_indep_supp_minmax_mmd} and \ref{table: poly_dynamic_supp_minmax}, \ref{table: poly_dynamic_supp_mmd}, \ref{table: poly_dynamic_supp_minmax_mmd}  we provide additional results when $z$'s follow an independent and dynamic SCM via a polynomial mixing $g(\cdot)$. We conducted experiments for different values of the dimension of the underlying latents $d$, different polynomial degrees, and for a varying number of domains $k$. We divide each table into two sections with top block corresponding to degree two polynomials with varying $d$, and the bottom block corresponding to the degree three polynomials. For each dimension we present two rows, the top row corresponding to the $R^2$ scores after training an autoencoder with reconstruction objective only, and right before enforcing any distributional invariances. Since we only need an autoencoder that can fully reconstruct the input, there is no need for training multiple perfect autoencoders, hence there is no standard error reported for such entries. We then take the perfectly trained autoencoder and enforce the distributional invariance penalty with 5 seeds, and present the results in the bottom row per each dimension. Across the different settings, we observe that as the number of domains increase we achieve high $R^2_{\mathcal{S}}$ and low  $R^2_{\mathcal{U}}$.

\begin{table}[h]
  \centering{
    {\caption{Polynomial Mixing Dataset $R^2$ scores, \textbf{Independent} DGP. The results are averaged over 5 seeds. $\hat{z},z\in R^{d}$ and $z_\mathcal{S}, z_{\mathcal{U}}\in R^{d/2}$ and $x=g(z)\in R^{200}$. Penalty used here is Min-Max. Top section and bottom section correspond to polynomial degrees of 2 and 3. For each dimension $d$, the top row corresponds to the scores after training the autoencoder with reconstruction objective only, and the bottom row denotes the scores after enforcing distributional invariances in 5 different runs.}
        \label{table: poly_indep_supp_minmax}
        }%
      {
        \begin{tabular}{c cccc cccc}
          \hline
          \multicolumn{1}{c}{}
        & \multicolumn{4}{c}{$R^2_{\mathcal{S}}$}
& \multicolumn{4}{c}{$R^2_{\mathcal{U}}$}  \\ \cmidrule(lr){2-5} \cmidrule(lr){6-9}
          \multicolumn{1}{c}{$d$} & \multicolumn{1}{c}{$k=2$}  & \multicolumn{1}{c}{$k=4$}   & \multicolumn{1}{c}{$k=8$} & $k=16$ & \multicolumn{1}{c}{$k=2$}  & \multicolumn{1}{c}{$k=4$} & \multicolumn{1}{c}{$k=8$} & $k=16$ \\ \hline
          \multirow{2}{*}{$6$}  & \multicolumn{1}{c}{0.03} & \multicolumn{1}{c}{0.08} & 0.09 & \multicolumn{1}{c}{0.02} & \multicolumn{1}{c}{0.96} & 0.83 & 0.89 & 0.98 \\
          \multirow{2}{*}{}  & \multicolumn{1}{c}{0.42$\pm0.04$} & \multicolumn{1}{c}{0.62$\pm0.01$} & 0.99$\pm0.00$ & \multicolumn{1}{c}{0.99$\pm0.00$} & \multicolumn{1}{c}{0.01$\pm0.00$} & 0.01$\pm0.00$ & 0.00$\pm0.00$ & 0.00$\pm0.00$ \\
          \midrule
          \multirow{2}{*}{$8$}  & \multicolumn{1}{c}{0.26} & \multicolumn{1}{c}{0.15} & 0.08 & \multicolumn{1}{c}{0.12} & \multicolumn{1}{c}{0.80} & 0.92 & 0.91 & 0.95 \\
          \multirow{2}{*}{}  & \multicolumn{1}{c}{0.34$\pm0.02$} & \multicolumn{1}{c}{0.99$\pm0.00$} & 0.98$\pm0.01$ & \multicolumn{1}{c}{0.97$\pm0.01$} & \multicolumn{1}{c}{0.01$\pm0.00$} & 0.01$\pm0.00$ & 0.00$\pm0.00$ & 0.01$\pm0.00$ \\
          \midrule
          \multirow{2}{*}{$10$}  & \multicolumn{1}{c}{0.22} & \multicolumn{1}{c}{0.04} & 0.03 & \multicolumn{1}{c}{0.05} & \multicolumn{1}{c}{0.79} & 0.97 & 0.98 & 0.96 \\
          \multirow{2}{*}{}  & \multicolumn{1}{c}{0.32$\pm0.03$} & \multicolumn{1}{c}{0.90$\pm0.04$} & 0.94$\pm0.04$ & \multicolumn{1}{c}{0.92$\pm0.04$} & \multicolumn{1}{c}{0.04$\pm0.00$} & 0.01$\pm0.00$ & 0.01$\pm0.00$ & 0.01$\pm0.00$ \\
          \midrule
          \multirow{2}{*}{$12$}  & \multicolumn{1}{c}{0.07} & \multicolumn{1}{c}{0.19} & 0.04 & \multicolumn{1}{c}{0.17} & \multicolumn{1}{c}{0.95} & 0.90 & 0.98 & 0.88 \\
          \multirow{2}{*}{}  & \multicolumn{1}{c}{0.38$\pm0.03$} & \multicolumn{1}{c}{0.83$\pm0.01$} & 0.89$\pm0.00$ & \multicolumn{1}{c}{0.95$\pm0.02$} & \multicolumn{1}{c}{0.06$\pm0.02$} & 0.01$\pm0.00$ & 0.01$\pm0.00$ & 0.01$\pm0.00$ \\
          \midrule
          \multirow{2}{*}{$14$}  & \multicolumn{1}{c}{0.12} & \multicolumn{1}{c}{0.17} & 0.17 & \multicolumn{1}{c}{0.11} & \multicolumn{1}{c}{0.94} & 0.93 & 0.86 & 0.93 \\
          \multirow{2}{*}{}  & \multicolumn{1}{c}{0.34$\pm0.03$} & \multicolumn{1}{c}{0.67$\pm0.01$} & 0.95$\pm0.02$ & \multicolumn{1}{c}{0.96$\pm0.02$} & \multicolumn{1}{c}{0.05$\pm0.01$} & 0.04$\pm0.01$ & 0.01$\pm0.00$ & 0.01$\pm0.00$ \\ \hline\hline
   
          \multirow{2}{*}{$6$}  & \multicolumn{1}{c}{0.16} & \multicolumn{1}{c}{0.04} & 0.05 & \multicolumn{1}{c}{0.09} & \multicolumn{1}{c}{0.83} & 0.96 & 0.95 & 0.92 \\
          \multirow{2}{*}{}  & \multicolumn{1}{c}{0.33$\pm0.01$} & \multicolumn{1}{c}{0.62$\pm0.00$} & 0.80$\pm0.00$ & \multicolumn{1}{c}{0.97$\pm0.01$} & \multicolumn{1}{c}{0.05$\pm0.02$} & 0.00$\pm0.00$ & 0.00$\pm0.00$ & 0.00$\pm0.00$ \\
          \midrule
          \multirow{2}{*}{$8$}  & \multicolumn{1}{c}{0.06} & \multicolumn{1}{c}{0.25} & 0.23 & \multicolumn{1}{c}{0.20} & \multicolumn{1}{c}{0.95} & 0.87 & 0.83 & 0.81 \\
          \multirow{2}{*}{}  & \multicolumn{1}{c}{0.45$\pm0.06$} & \multicolumn{1}{c}{0.84$\pm0.03$} & 0.93$\pm0.01$ & \multicolumn{1}{c}{0.92$\pm0.01$} & \multicolumn{1}{c}{0.06$\pm0.04$} & 0.01$\pm0.00$ & 0.01$\pm0.00$ & 0.00$\pm0.00$ \\
          \midrule
          \multirow{2}{*}{$10$}  & \multicolumn{1}{c}{0.13} & \multicolumn{1}{c}{0.15} & 0.11 & \multicolumn{1}{c}{0.04} & \multicolumn{1}{c}{0.89} & 0.80 & 0.93 & 0.96 \\
          \multirow{2}{*}{}  & \multicolumn{1}{c}{0.48$\pm0.01$} & \multicolumn{1}{c}{0.83$\pm0.00$} & 0.87$\pm0.00$ & \multicolumn{1}{c}{0.93$\pm0.01$} & \multicolumn{1}{c}{0.02$\pm0.00$} & 0.01$\pm0.00$ & 0.01$\pm0.00$ & 0.01$\pm0.00$ \\
          \midrule
          \multirow{2}{*}{$12$}  & \multicolumn{1}{c}{0.20} & \multicolumn{1}{c}{0.21} & 0.18 & \multicolumn{1}{c}{0.19} & \multicolumn{1}{c}{0.85} & 0.84 & 0.85 & 0.82 \\
          \multirow{2}{*}{}  & \multicolumn{1}{c}{0.52$\pm0.02$} & \multicolumn{1}{c}{0.80$\pm0.03$} & 0.88$\pm0.01$ & \multicolumn{1}{c}{0.95$\pm0.01$} & \multicolumn{1}{c}{0.07$\pm0.02$} & 0.02$\pm0.00$ & 0.01$\pm0.00$ & 0.02$\pm0.00$ \\
          \midrule
          \multirow{2}{*}{$14$}  & \multicolumn{1}{c}{0.18} & \multicolumn{1}{c}{0.06} & 0.29 & \multicolumn{1}{c}{0.27} & \multicolumn{1}{c}{0.74} & 0.80 & 0.76 & 0.71 \\
          \multirow{2}{*}{}  & \multicolumn{1}{c}{0.26$\pm0.02$} & \multicolumn{1}{c}{0.32$\pm0.01$} & 0.91$\pm0.01$ & \multicolumn{1}{c}{0.93$\pm0.00$} & \multicolumn{1}{c}{0.10$\pm0.01$} & 0.03$\pm0.00$ & 0.01$\pm0.00$ & 0.01$\pm0.00$ \\ \hline
        \end{tabular}
      }
  }
\end{table}

\begin{table}[h]
  \centering{
    {\caption{Polynomial Mixing Dataset $R^2$ scores, \textbf{Independent} DGP. The results are averaged over 5 seeds. $\hat{z},z\in R^{d}$ and $z_\mathcal{S}, z_{\mathcal{U}}\in R^{d/2}$ and $x=g(z)\in R^{200}$. Penalty used here is MMD. Top section and bottom section correspond to polynomial degrees of 2 and 3. For each dimension $d$, the top and bottom rows correspond to the scores after Stages 1, 2.}
        \label{table: poly_indep_supp_mmd}
        }%
      {
        \begin{tabular}{c cccc cccc}
          \hline
          \multicolumn{1}{c}{}
        & \multicolumn{4}{c}{$R^2_{\mathcal{S}}$}
& \multicolumn{4}{c}{$R^2_{\mathcal{U}}$}  \\ \cmidrule(lr){2-5} \cmidrule(lr){6-9}
          \multicolumn{1}{c}{$d$} & \multicolumn{1}{c}{$k=2$}  & \multicolumn{1}{c}{$k=4$}   & \multicolumn{1}{c}{$k=8$} & $k=16$ & \multicolumn{1}{c}{$k=2$}  & \multicolumn{1}{c}{$k=4$} & \multicolumn{1}{c}{$k=8$} & $k=16$ \\ \hline
          \multirow{2}{*}{$6$}  & \multicolumn{1}{c}{0.03} & \multicolumn{1}{c}{0.08} & 0.09 & \multicolumn{1}{c}{0.02} & \multicolumn{1}{c}{0.96} & 0.83 & 0.86 & 0.98 \\
          \multirow{2}{*}{}  & \multicolumn{1}{c}{0.54$\pm0.04$} & \multicolumn{1}{c}{0.55$\pm0.04$} & 0.99$\pm0.00$ & \multicolumn{1}{c}{0.99$\pm0.00$} & \multicolumn{1}{c}{0.01$\pm0.00$} & 0.01$\pm0.00$ & 0.01$\pm0.00$ & 0.00$\pm0.00$ \\
          \midrule
          \multirow{2}{*}{$8$}  & \multicolumn{1}{c}{0.26} & \multicolumn{1}{c}{0.15} & 0.08 & \multicolumn{1}{c}{0.12} & \multicolumn{1}{c}{0.80} & 0.92 & 0.91 & 0.92 \\
          \multirow{2}{*}{}  & \multicolumn{1}{c}{0.42$\pm0.05$} & \multicolumn{1}{c}{0.76$\pm0.00$} & 0.98$\pm0.01$ & \multicolumn{1}{c}{0.97$\pm0.02$} & \multicolumn{1}{c}{0.08$\pm0.05$} & 0.01$\pm0.00$ & 0.00$\pm0.00$ & 0.01$\pm0.00$ \\ \midrule
          \multirow{2}{*}{$10$}  & \multicolumn{1}{c}{0.22} & \multicolumn{1}{c}{0.04} & 0.03 & \multicolumn{1}{c}{0.05} & \multicolumn{1}{c}{0.79} & 0.97 & 0.98 & 0.96 \\
          \multirow{2}{*}{}  & \multicolumn{1}{c}{0.52$\pm0.04$} & \multicolumn{1}{c}{0.81$\pm0.00$} & 0.86$\pm0.00$ & \multicolumn{1}{c}{0.91$\pm0.04$} & \multicolumn{1}{c}{0.05$\pm0.01$} & 0.01$\pm0.00$ & 0.01$\pm0.00$ & 0.01$\pm0.00$ \\ \midrule
          \multirow{2}{*}{$12$}  & \multicolumn{1}{c}{0.07} & \multicolumn{1}{c}{0.19} & 0.04 & \multicolumn{1}{c}{0.17} & \multicolumn{1}{c}{0.95} & 0.90 & 0.98 & 0.88 \\
          \multirow{2}{*}{}  & \multicolumn{1}{c}{0.48$\pm0.01$} & \multicolumn{1}{c}{0.65$\pm0.02$} & 0.90$\pm0.00$ & \multicolumn{1}{c}{0.98$\pm0.00$} & \multicolumn{1}{c}{0.12$\pm0.03$} & 0.02$\pm0.00$ & 0.01$\pm0.00$ & 0.01$\pm0.00$ \\ \midrule
          \multirow{2}{*}{$14$}  & \multicolumn{1}{c}{0.12} & \multicolumn{1}{c}{0.17} & 0.17 & \multicolumn{1}{c}{0.11} & \multicolumn{1}{c}{0.94} & 0.93 & 0.86 & 0.93 \\
          \multirow{2}{*}{}  & \multicolumn{1}{c}{0.52$\pm0.05$} & \multicolumn{1}{c}{0.55$\pm0.01$} & 0.99$\pm0.00$ & \multicolumn{1}{c}{0.98$\pm0.01$} & \multicolumn{1}{c}{0.05$\pm0.01$} & 0.06$\pm0.02$ & 0.01$\pm0.00$ & 0.01$\pm0.00$ \\ \hline\hline
          \multirow{2}{*}{$6$}  & \multicolumn{1}{c}{0.16} & \multicolumn{1}{c}{0.04} & 0.05 & \multicolumn{1}{c}{0.09} & \multicolumn{1}{c}{0.83} & 0.96 & 0.95 & 0.92 \\ 
          \multirow{2}{*}{}  & \multicolumn{1}{c}{0.46$\pm0.05$} & \multicolumn{1}{c}{0.63$\pm0.00$} & 0.80$\pm0.00$ & \multicolumn{1}{c}{0.98$\pm0.01$} & \multicolumn{1}{c}{0.05$\pm0.03$} & 0.01$\pm0.00$ & 0.00$\pm0.00$ & 0.00$\pm0.00$ \\ \midrule
          \multirow{2}{*}{$8$}  & \multicolumn{1}{c}{0.06} & \multicolumn{1}{c}{0.25} & 0.23 & \multicolumn{1}{c}{0.20} & \multicolumn{1}{c}{0.95} & 0.87 & 0.83 & 0.81 \\
          \multirow{2}{*}{}  & \multicolumn{1}{c}{0.54$\pm0.02$} & \multicolumn{1}{c}{0.73$\pm0.01$} & 0.92$\pm0.04$ & \multicolumn{1}{c}{0.98$\pm0.00$} & \multicolumn{1}{c}{0.07$\pm0.04$} & 0.02$\pm0.00$ & 0.01$\pm0.00$ & 0.00$\pm0.00$ \\ \midrule
          \multirow{2}{*}{$10$}  & \multicolumn{1}{c}{0.13} & \multicolumn{1}{c}{0.15} & 0.11 & \multicolumn{1}{c}{0.04} & \multicolumn{1}{c}{0.89} & 0.80 & 0.93 & 0.96 \\
          \multirow{2}{*}{}  & \multicolumn{1}{c}{0.49$\pm0.04$} & \multicolumn{1}{c}{0.72$\pm0.01$} & 0.87$\pm0.00$ & \multicolumn{1}{c}{0.98$\pm0.00$} & \multicolumn{1}{c}{0.05$\pm0.01$} & 0.01$\pm0.00$ & 0.01$\pm0.00$ & 0.01$\pm0.00$ \\ \midrule
          \multirow{2}{*}{$12$}  & \multicolumn{1}{c}{0.20} & \multicolumn{1}{c}{0.21} & 0.18 & \multicolumn{1}{c}{0.19} & \multicolumn{1}{c}{0.85} & 0.84 & 0.85 & 0.82 \\
          \multirow{2}{*}{}  & \multicolumn{1}{c}{0.51$\pm0.02$} & \multicolumn{1}{c}{0.74$\pm0.02$} & 0.89$\pm0.00$ & \multicolumn{1}{c}{0.97$\pm0.00$} & \multicolumn{1}{c}{0.12$\pm0.02$} & 0.04$\pm0.00$ & 0.01$\pm0.00$ & 0.01$\pm0.00$ \\ \midrule
          \multirow{2}{*}{$14$}  & \multicolumn{1}{c}{0.18} & \multicolumn{1}{c}{0.06} & 0.29 & \multicolumn{1}{c}{0.27} & \multicolumn{1}{c}{0.74} & 0.80 & 0.76 & 0.71 \\
          \multirow{2}{*}{}  & \multicolumn{1}{c}{0.38$\pm0.02$} & \multicolumn{1}{c}{0.40$\pm0.00$} & 0.94$\pm0.00$ & \multicolumn{1}{c}{0.95$\pm0.00$} & \multicolumn{1}{c}{0.08$\pm0.00$} & 0.03$\pm0.00$ & 0.02$\pm0.00$ & 0.01$\pm0.00$ \\ \hline
        \end{tabular}
      }
  }
\end{table}

\begin{table}[h]
  \centering{
    {\caption{Polynomial Mixing Dataset $R^2$ scores, \textbf{Independent} DGP. The results are averaged over 5 seeds. $\hat{z},z\in R^{d}$ and $z_\mathcal{S}, z_{\mathcal{U}}\in R^{d/2}$ and $x=g(z)\in R^{200}$. Penalty used here is MMD$+$Min-Max. Top section and bottom section correspond to polynomial degrees of 2 and 3. For each dimension $d$, the top and bottom rows correspond to the scores after Stages 1, 2.}
        \label{table: poly_indep_supp_minmax_mmd}
        }%
      {
        \begin{tabular}{c cccc cccc}
          \hline
          \multicolumn{1}{c}{}
        & \multicolumn{4}{c}{$R^2_{\mathcal{S}}$}
& \multicolumn{4}{c}{$R^2_{\mathcal{U}}$}  \\ \cmidrule(lr){2-5} \cmidrule(lr){6-9}
          \multicolumn{1}{c}{$d$} & \multicolumn{1}{c}{$k=2$}  & \multicolumn{1}{c}{$k=4$}   & \multicolumn{1}{c}{$k=8$} & $k=16$ & \multicolumn{1}{c}{$k=2$}  & \multicolumn{1}{c}{$k=4$} & \multicolumn{1}{c}{$k=8$} & $k=16$ \\ \hline
          \multirow{2}{*}{$6$}  & \multicolumn{1}{c}{0.03} & \multicolumn{1}{c}{0.08} & 0.09 & \multicolumn{1}{c}{0.02} & \multicolumn{1}{c}{0.96} & 0.83 & 0.89 & 0.98 \\
          \multirow{2}{*}{}  & \multicolumn{1}{c}{0.60$\pm0.06$} & \multicolumn{1}{c}{0.60$\pm0.00$} & 0.99$\pm0.00$ & \multicolumn{1}{c}{0.99$\pm0.01$} & \multicolumn{1}{c}{0.02$\pm0.01$} & 0.00$\pm0.00$ & 0.01$\pm0.00$ & 0.00$\pm0.00$ \\ \midrule
          \multirow{2}{*}{$8$}  & \multicolumn{1}{c}{0.26} & \multicolumn{1}{c}{0.15} & 0.08 & \multicolumn{1}{c}{0.12} & \multicolumn{1}{c}{0.80} & 0.92 & 0.91 & 0.92 \\
          \multirow{2}{*}{}  & \multicolumn{1}{c}{0.52$\pm0.04$} & \multicolumn{1}{c}{0.98$\pm0.00$} & 0.97$\pm0.01$ & \multicolumn{1}{c}{0.99$\pm0.00$} & \multicolumn{1}{c}{0.03$\pm0.02$} & 0.00$\pm0.00$ & 0.00$\pm0.00$ & 0.01$\pm0.00$ \\ \midrule
          \multirow{2}{*}{$10$}  & \multicolumn{1}{c}{0.22} & \multicolumn{1}{c}{0.04} & 0.03 & \multicolumn{1}{c}{0.05} & \multicolumn{1}{c}{0.79} & 0.97 & 0.98 & 0.96 \\
          \multirow{2}{*}{}  & \multicolumn{1}{c}{0.68$\pm0.03$} & \multicolumn{1}{c}{0.96$\pm0.01$} & 0.95$\pm0.03$ & \multicolumn{1}{c}{0.94$\pm0.01$} & \multicolumn{1}{c}{0.02$\pm0.00$} & 0.01$\pm0.00$ & 0.01$\pm0.00$ & 0.01$\pm0.00$ \\ \midrule
          \multirow{2}{*}{$12$}  & \multicolumn{1}{c}{0.07} & \multicolumn{1}{c}{0.19} & 0.04 & \multicolumn{1}{c}{0.17} & \multicolumn{1}{c}{0.95} & 0.90 & 0.98 & 0.88 \\
          \multirow{2}{*}{}  & \multicolumn{1}{c}{0.63$\pm0.04$} & \multicolumn{1}{c}{0.92$\pm0.00$} & 0.90$\pm0.00$ & \multicolumn{1}{c}{0.97$\pm0.01$} & \multicolumn{1}{c}{0.02$\pm0.00$} & 0.01$\pm0.00$ & 0.01$\pm0.00$ & 0.01$\pm0.00$ \\ \midrule
          \multirow{2}{*}{$14$}  & \multicolumn{1}{c}{0.12} & \multicolumn{1}{c}{0.17} & 0.17 & \multicolumn{1}{c}{0.11} & \multicolumn{1}{c}{0.94} & 0.93 & 0.86 & 0.93 \\
          \multirow{2}{*}{}  & \multicolumn{1}{c}{0.63$\pm0.02$} & \multicolumn{1}{c}{0.65$\pm0.00$} & 0.96$\pm0.02$ & \multicolumn{1}{c}{0.98$\pm0.01$} & \multicolumn{1}{c}{0.04$\pm0.02$} & 0.04$\pm0.01$ & 0.01$\pm0.00$ & 0.01$\pm0.00$ \\ \hline\hline
          \multirow{2}{*}{$6$}  & \multicolumn{1}{c}{0.16} & \multicolumn{1}{c}{0.04} & 0.05 & \multicolumn{1}{c}{0.10} & \multicolumn{1}{c}{0.83} & 0.96 & 0.95 & 0.92 \\
          \multirow{2}{*}{}  & \multicolumn{1}{c}{0.44$\pm0.03$} & \multicolumn{1}{c}{0.63$\pm0.00$} & 0.80$\pm0.00$ & \multicolumn{1}{c}{0.96$\pm0.02$} & \multicolumn{1}{c}{0.03$\pm0.01$} & 0.00$\pm0.00$ & 0.00$\pm0.00$ & 0.01$\pm0.00$ \\ \midrule
          \multirow{2}{*}{$8$}  & \multicolumn{1}{c}{0.06} & \multicolumn{1}{c}{0.25} & 0.23 & \multicolumn{1}{c}{0.20} & \multicolumn{1}{c}{0.95} & 0.87 & 0.83 & 0.81 \\
          \multirow{2}{*}{}  & \multicolumn{1}{c}{0.63$\pm0.04$} & \multicolumn{1}{c}{0.91$\pm0.02$} & 0.93$\pm0.02$ & \multicolumn{1}{c}{0.97$\pm0.01$} & \multicolumn{1}{c}{0.03$\pm0.02$} & 0.01$\pm0.00$ & 0.01$\pm0.00$ & 0.00$\pm0.00$ \\ \midrule
          \multirow{2}{*}{$10$}  & \multicolumn{1}{c}{0.13} & \multicolumn{1}{c}{0.15} & 0.11 & \multicolumn{1}{c}{0.04} & \multicolumn{1}{c}{0.89} & 0.80 & 0.93 & 0.96 \\
          \multirow{2}{*}{}  & \multicolumn{1}{c}{0.62$\pm0.04$} & \multicolumn{1}{c}{0.79$\pm0.04$} & 0.85$\pm0.01$ & \multicolumn{1}{c}{0.97$\pm0.00$} & \multicolumn{1}{c}{0.02$\pm0.00$} & 0.01$\pm0.00$ & 0.01$\pm0.00$ & 0.01$\pm0.00$ \\ \midrule
          \multirow{2}{*}{$12$}  & \multicolumn{1}{c}{0.20} & \multicolumn{1}{c}{0.21} & 0.18 & \multicolumn{1}{c}{0.20} & \multicolumn{1}{c}{0.85} & 0.84 & 0.85 & 0.82 \\
          \multirow{2}{*}{}  & \multicolumn{1}{c}{0.62$\pm0.02$} & \multicolumn{1}{c}{0.81$\pm0.01$} & 0.89$\pm0.00$ & \multicolumn{1}{c}{0.97$\pm0.00$} & \multicolumn{1}{c}{0.08$\pm0.02$} & 0.02$\pm0.00$ & 0.01$\pm0.00$ & 0.01$\pm0.00$ \\ \midrule
          \multirow{2}{*}{$14$}  & \multicolumn{1}{c}{0.18} & \multicolumn{1}{c}{0.06} & 0.29 & \multicolumn{1}{c}{0.27} & \multicolumn{1}{c}{0.74} & 0.80 & 0.76 & 0.71 \\
          \multirow{2}{*}{}  & \multicolumn{1}{c}{0.36$\pm0.01$} & \multicolumn{1}{c}{0.39$\pm0.00$} & 0.93$\pm0.02$ & \multicolumn{1}{c}{0.95$\pm0.00$} & \multicolumn{1}{c}{0.10$\pm0.01$} & 0.03$\pm0.00$ & 0.02$\pm0.00$ & 0.01$\pm0.00$ \\ \hline
        \end{tabular}
      }
  }
\end{table}

\begin{table}[h]
  \centering{
    {\caption{Polynomial Mixing Dataset $R^2$ scores, \textbf{Dynamic SCM} DGP. The results are averaged over 5 seeds. $\hat{z},z\in R^{d}$ and $z_\mathcal{S}, z_{\mathcal{U}}\in R^{d/2}$ and $x=g(z)\in R^{200}$. Penalty used here is Min-Max. Top section and bottom section correspond to polynomial degrees of 2 and 3. For each dimension $d$, the top and bottom rows correspond to the scores after Stages 1, 2.}
        \label{table: poly_dynamic_supp_minmax}
        }%
      {
        \begin{tabular}{c cccc cccc}
          \hline
          \multicolumn{1}{c}{}
        & \multicolumn{4}{c}{$R^2_{\mathcal{S}}$}
& \multicolumn{4}{c}{$R^2_{\mathcal{U}}$}  \\ \cmidrule(lr){2-5} \cmidrule(lr){6-9}
          \multicolumn{1}{c}{$d$} & \multicolumn{1}{c}{$k=2$}  & \multicolumn{1}{c}{$k=4$}   & \multicolumn{1}{c}{$k=8$} & $k=16$ & \multicolumn{1}{c}{$k=2$}  & \multicolumn{1}{c}{$k=4$} & \multicolumn{1}{c}{$k=8$} & $k=16$ \\ \hline
          \multirow{2}{*}{$6$}  & \multicolumn{1}{c}{0.05} & \multicolumn{1}{c}{0.02} & 0.01 & \multicolumn{1}{c}{0.05} & \multicolumn{1}{c}{0.95} & 0.97 & 0.99 & 0.95 \\
          \multirow{2}{*}{}  & \multicolumn{1}{c}{0.28$\pm0.04$} & \multicolumn{1}{c}{0.96$\pm0.01$} & 0.71$\pm0.00$ & \multicolumn{1}{c}{0.99$\pm0.00$} & \multicolumn{1}{c}{0.03$\pm0.02$} & 0.01$\pm0.00$ & 0.01$\pm0.00$ & 0.01$\pm0.00$ \\ \midrule
          \multirow{2}{*}{$8$}  & \multicolumn{1}{c}{0.20} & \multicolumn{1}{c}{0.18} & 0.14 & \multicolumn{1}{c}{0.01} & \multicolumn{1}{c}{0.88} & 0.87 & 0.85 & 0.99 \\
          \multirow{2}{*}{}  & \multicolumn{1}{c}{0.39$\pm0.03$} & \multicolumn{1}{c}{0.71$\pm0.01$} & 0.78$\pm0.01$ & \multicolumn{1}{c}{0.93$\pm0.04$} & \multicolumn{1}{c}{0.05$\pm0.01$} & 0.01$\pm0.00$ & 0.01$\pm0.00$ & 0.01$\pm0.00$ \\ \midrule 
          \multirow{2}{*}{$10$}  & \multicolumn{1}{c}{0.19} & \multicolumn{1}{c}{0.04} & 0.02 & \multicolumn{1}{c}{0.05} & \multicolumn{1}{c}{0.86} & 0.98 & 0.99 & 0.97 \\
          \multirow{2}{*}{}  & \multicolumn{1}{c}{0.35$\pm0.04$} & \multicolumn{1}{c}{0.76$\pm0.02$} & 0.98$\pm0.00$ & \multicolumn{1}{c}{0.99$\pm0.00$} & \multicolumn{1}{c}{0.07$\pm0.03$} & 0.01$\pm0.00$ & 0.01$\pm0.00$ & 0.01$\pm0.00$ \\ \midrule
          \multirow{2}{*}{$12$}  & \multicolumn{1}{c}{0.05} & \multicolumn{1}{c}{0.18} & 0.13 & \multicolumn{1}{c}{0.11} & \multicolumn{1}{c}{0.97} & 0.90 & 0.92 & 0.93 \\
          \multirow{2}{*}{}  & \multicolumn{1}{c}{0.41$\pm0.03$} & \multicolumn{1}{c}{0.80$\pm0.01$} & 0.97$\pm0.01$ & \multicolumn{1}{c}{0.97$\pm0.01$} & \multicolumn{1}{c}{0.09$\pm0.03$} & 0.01$\pm0.00$ & 0.01$\pm0.00$ & 0.01$\pm0.00$ \\ \midrule
          \multirow{2}{*}{$14$}  & \multicolumn{1}{c}{0.11} & \multicolumn{1}{c}{0.16} & 0.17 & \multicolumn{1}{c}{0.04} & \multicolumn{1}{c}{0.95} & 0.94 & 0.83 & 0.98 \\
          \multirow{2}{*}{}  & \multicolumn{1}{c}{0.39$\pm0.01$} & \multicolumn{1}{c}{0.69$\pm0.01$} & 0.95$\pm0.02$ & \multicolumn{1}{c}{0.97$\pm0.01$} & \multicolumn{1}{c}{0.06$\pm0.01$} & 0.04$\pm0.01$ & 0.02$\pm0.00$ & 0.01$\pm0.00$ \\ \hline\hline
          \multirow{2}{*}{$6$}  & \multicolumn{1}{c}{0.30} & \multicolumn{1}{c}{0.16} & 0.30 & \multicolumn{1}{c}{0.07} & \multicolumn{1}{c}{0.71} & 0.88 & 0.74 & 0.94 \\
          \multirow{2}{*}{}  & \multicolumn{1}{c}{0.34$\pm0.01$} & \multicolumn{1}{c}{0.92$\pm0.01$} & 0.98$\pm0.00$ & \multicolumn{1}{c}{0.99$\pm0.00$} & \multicolumn{1}{c}{0.03$\pm0.02$} & 0.01$\pm0.00$ & 0.01$\pm0.00$ & 0.01$\pm0.00$ \\ \midrule
          \multirow{2}{*}{$8$}  & \multicolumn{1}{c}{0.18} & \multicolumn{1}{c}{0.06} & 0.13 & \multicolumn{1}{c}{0.21} & \multicolumn{1}{c}{0.86} & 0.98 & 0.89 & 0.80 \\
          \multirow{2}{*}{}  & \multicolumn{1}{c}{0.45$\pm0.05$} & \multicolumn{1}{c}{0.90$\pm0.04$} & 0.95$\pm0.01$ & \multicolumn{1}{c}{0.95$\pm0.02$} & \multicolumn{1}{c}{0.07$\pm0.04$} & 0.01$\pm0.00$ & 0.01$\pm0.00$ & 0.01$\pm0.00$ \\ \midrule
          \multirow{2}{*}{$10$}  & \multicolumn{1}{c}{0.21} & \multicolumn{1}{c}{0.12} & 0.22 & \multicolumn{1}{c}{0.07} & \multicolumn{1}{c}{0.87} & 0.84 & 0.87 & 0.94 \\
          \multirow{2}{*}{}  & \multicolumn{1}{c}{0.42$\pm0.02$} & \multicolumn{1}{c}{0.70$\pm0.00$} & 0.95$\pm0.01$ & \multicolumn{1}{c}{0.95$\pm0.01$} & \multicolumn{1}{c}{0.08$\pm0.03$} & 0.01$\pm0.00$ & 0.01$\pm0.00$ & 0.01$\pm0.00$ \\ \midrule
          \multirow{2}{*}{$12$}  & \multicolumn{1}{c}{0.15} & \multicolumn{1}{c}{0.23} & 0.14 & \multicolumn{1}{c}{0.17} & \multicolumn{1}{c}{0.88} & 0.81 & 0.87 & 0.85 \\
          \multirow{2}{*}{}  & \multicolumn{1}{c}{0.44$\pm0.02$} & \multicolumn{1}{c}{0.81$\pm0.02$} & 0.87$\pm0.03$ & \multicolumn{1}{c}{0.94$\pm0.01$} & \multicolumn{1}{c}{0.11$\pm0.01$} & 0.02$\pm0.00$ & 0.01$\pm0.00$ & 0.01$\pm0.00$ \\ \midrule
          \multirow{2}{*}{$14$}  & \multicolumn{1}{c}{0.22} & \multicolumn{1}{c}{0.16} & 0.25 & \multicolumn{1}{c}{0.22} & \multicolumn{1}{c}{0.69} & 0.79 & 0.80 & 0.73 \\
          \multirow{2}{*}{}  & \multicolumn{1}{c}{0.34$\pm0.02$} & \multicolumn{1}{c}{0.63$\pm0.01$} & 0.88$\pm0.02$ & \multicolumn{1}{c}{0.92$\pm0.00$} & \multicolumn{1}{c}{0.08$\pm0.01$} & 0.03$\pm0.00$ & 0.02$\pm0.00$ & 0.02$\pm0.00$ \\ \hline
        \end{tabular}
      }
  }
\end{table}

\begin{table}[h]
  \centering{
    {\caption{Polynomial Mixing Dataset $R^2$ scores, \textbf{Dynamic SCM} DGP. The results are averaged over 5 seeds. $\hat{z},z\in R^{d}$ and $z_\mathcal{S}, z_{\mathcal{U}}\in R^{d/2}$ and $x=g(z)\in R^{200}$. Penalty used here is MMD. Top section and bottom section correspond to polynomial degrees of 2 and 3. 
    For each dimension $d$, the top and bottom rows correspond to the scores after Stages 1, 2.}
        \label{table: poly_dynamic_supp_mmd}
        }%
      {
        \begin{tabular}{c cccc cccc}
          \hline
          \multicolumn{1}{c}{}
        & \multicolumn{4}{c}{$R^2_{\mathcal{S}}$}
& \multicolumn{4}{c}{$R^2_{\mathcal{U}}$}  \\ \cmidrule(lr){2-5} \cmidrule(lr){6-9}
          \multicolumn{1}{c}{$d$} & \multicolumn{1}{c}{$k=2$}  & \multicolumn{1}{c}{$k=4$}   & \multicolumn{1}{c}{$k=8$} & $k=16$ & \multicolumn{1}{c}{$k=2$}  & \multicolumn{1}{c}{$k=4$} & \multicolumn{1}{c}{$k=8$} & $k=16$ \\ \hline
          \multirow{2}{*}{$6$}  & \multicolumn{1}{c}{0.05} & \multicolumn{1}{c}{0.02} & 0.01 & \multicolumn{1}{c}{0.05} & \multicolumn{1}{c}{0.95} & 0.97 & 0.99 & 0.95 \\
          \multirow{2}{*}{}  & \multicolumn{1}{c}{0.25$\pm0.06$} & \multicolumn{1}{c}{0.85$\pm0.03$} & 0.64$\pm0.02$ & \multicolumn{1}{c}{0.99$\pm0.00$} & \multicolumn{1}{c}{0.08$\pm0.05$} & 0.05$\pm0.01$ & 0.01$\pm0.00$ & 0.01$\pm0.00$ \\ \midrule
          \multirow{2}{*}{$8$}  & \multicolumn{1}{c}{0.20} & \multicolumn{1}{c}{0.18} & 0.14 & \multicolumn{1}{c}{0.01} & \multicolumn{1}{c}{0.88} & 0.87 & 0.85 & 0.99 \\
          \multirow{2}{*}{}  & \multicolumn{1}{c}{0.41$\pm0.03$} & \multicolumn{1}{c}{0.58$\pm0.03$} & 0.73$\pm0.02$ & \multicolumn{1}{c}{0.95$\pm0.03$} & \multicolumn{1}{c}{0.08$\pm0.04$} & 0.08$\pm0.03$ & 0.02$\pm0.00$ & 0.01$\pm0.00$ \\ \midrule
          \multirow{2}{*}{$10$}  & \multicolumn{1}{c}{0.19} & \multicolumn{1}{c}{0.04} & 0.02 & \multicolumn{1}{c}{0.04} & \multicolumn{1}{c}{0.86} & 0.98 & 0.99 & 0.97 \\
          \multirow{2}{*}{}  & \multicolumn{1}{c}{0.41$\pm0.02$} & \multicolumn{1}{c}{0.53$\pm0.02$} & 0.91$\pm0.01$ & \multicolumn{1}{c}{0.96$\pm0.02$} & \multicolumn{1}{c}{0.10$\pm0.03$} & 0.04$\pm0.01$ & 0.02$\pm0.00$ & 0.01$\pm0.00$ \\ \midrule
          \multirow{2}{*}{$12$}  & \multicolumn{1}{c}{0.05} & \multicolumn{1}{c}{0.18} & 0.13 & \multicolumn{1}{c}{0.11} & \multicolumn{1}{c}{0.97} & 0.90 & 0.92 & 0.93 \\
          \multirow{2}{*}{}  & \multicolumn{1}{c}{0.39$\pm0.03$} & \multicolumn{1}{c}{0.54$\pm0.02$} & 0.89$\pm0.01$ & \multicolumn{1}{c}{0.99$\pm0.00$} & \multicolumn{1}{c}{0.10$\pm0.03$} & 0.04$\pm0.01$ & 0.03$\pm0.01$ & 0.01$\pm0.00$ \\ \midrule
          \multirow{2}{*}{$14$}  & \multicolumn{1}{c}{0.11} & \multicolumn{1}{c}{0.16} & 0.17 & \multicolumn{1}{c}{0.04} & \multicolumn{1}{c}{0.95} & 0.94 & 0.83 & 0.98 \\
          \multirow{2}{*}{}  & \multicolumn{1}{c}{0.38$\pm0.02$} & \multicolumn{1}{c}{0.48$\pm0.02$} & 0.89$\pm0.01$ & \multicolumn{1}{c}{0.99$\pm0.00$} & \multicolumn{1}{c}{0.09$\pm0.02$} & 0.10$\pm0.02$ & 0.03$\pm0.00$ & 0.01$\pm0.00$ \\ \hline\hline
          \multirow{2}{*}{$6$}  & \multicolumn{1}{c}{0.30} & \multicolumn{1}{c}{0.16} & 0.30 & \multicolumn{1}{c}{0.07} & \multicolumn{1}{c}{0.71} & 0.88 & 0.74 & 0.94 \\
          \multirow{2}{*}{}  & \multicolumn{1}{c}{0.41$\pm0.03$} & \multicolumn{1}{c}{0.77$\pm0.03$} & 0.86$\pm0.03$ & \multicolumn{1}{c}{0.98$\pm0.00$} & \multicolumn{1}{c}{0.04$\pm0.01$} & 0.05$\pm0.02$ & 0.02$\pm0.01$ & 0.01$\pm0.00$ \\ \midrule
          \multirow{2}{*}{$8$}  & \multicolumn{1}{c}{0.18} & \multicolumn{1}{c}{0.06} & 0.13 & \multicolumn{1}{c}{0.21} & \multicolumn{1}{c}{0.86} & 0.98 & 0.89 & 0.80 \\
          \multirow{2}{*}{}  & \multicolumn{1}{c}{0.46$\pm0.02$} & \multicolumn{1}{c}{0.66$\pm0.02$} & 0.84$\pm0.03$ & \multicolumn{1}{c}{0.98$\pm0.00$} & \multicolumn{1}{c}{0.05$\pm0.02$} & 0.02$\pm0.01$ & 0.02$\pm0.00$ & 0.01$\pm0.00$ \\ \midrule
          \multirow{2}{*}{$10$}  & \multicolumn{1}{c}{0.21} & \multicolumn{1}{c}{0.12} & 0.22 & \multicolumn{1}{c}{0.07} & \multicolumn{1}{c}{0.87} & 0.84 & 0.87 & 0.94 \\
          \multirow{2}{*}{}  & \multicolumn{1}{c}{0.42$\pm0.03$} & \multicolumn{1}{c}{0.55$\pm0.02$} & 0.90$\pm0.01$ & \multicolumn{1}{c}{0.96$\pm0.00$} & \multicolumn{1}{c}{0.08$\pm0.01$} & 0.03$\pm0.00$ & 0.03$\pm0.00$ & 0.01$\pm0.00$ \\ \midrule
          \multirow{2}{*}{$12$}  & \multicolumn{1}{c}{0.15} & \multicolumn{1}{c}{0.23} & 0.14 & \multicolumn{1}{c}{0.17} & \multicolumn{1}{c}{0.88} & 0.81 & 0.87 & 0.85 \\
          \multirow{2}{*}{}  & \multicolumn{1}{c}{0.45$\pm0.01$} & \multicolumn{1}{c}{0.64$\pm0.02$} & 0.85$\pm0.01$ & \multicolumn{1}{c}{0.96$\pm0.00$} & \multicolumn{1}{c}{0.11$\pm0.01$} & 0.07$\pm0.01$ & 0.04$\pm0.00$ & 0.01$\pm0.00$ \\ \midrule
          \multirow{2}{*}{$14$}  & \multicolumn{1}{c}{0.22} & \multicolumn{1}{c}{0.16} & 0.25 & \multicolumn{1}{c}{0.22} & \multicolumn{1}{c}{0.69} & 0.79 & 0.80 & 0.73 \\
          \multirow{2}{*}{}  & \multicolumn{1}{c}{0.34$\pm0.01$} & \multicolumn{1}{c}{0.52$\pm0.01$} & 0.85$\pm0.01$ & \multicolumn{1}{c}{0.93$\pm0.01$} & \multicolumn{1}{c}{0.09$\pm0.01$} & 0.09$\pm0.02$ & 0.05$\pm0.00$ & 0.02$\pm0.00$ \\ \hline
        \end{tabular}
      }
  }
\end{table}

\begin{table}[h]
  \centering{
    {\caption{Polynomial Mixing Dataset $R^2$ scores, \textbf{Dynamic SCM} DGP. The results are averaged over 5 seeds. $\hat{z},z\in R^{d}$ and $z_\mathcal{S}, z_{\mathcal{U}}\in R^{d/2}$ and $x=g(z)\in R^{200}$. Penalty used here is MMD$+$Min-Max. Top section and bottom section correspond to polynomial degrees of 2 and 3. For each dimension $d$, the top and bottom rows correspond to the scores after Stages 1, 2.}
        \label{table: poly_dynamic_supp_minmax_mmd}
        }%
      {
        \begin{tabular}{c cccc cccc}
          \hline
          \multicolumn{1}{c}{}
        & \multicolumn{4}{c}{$R^2_{\mathcal{S}}$}
& \multicolumn{4}{c}{$R^2_{\mathcal{U}}$}  \\ \cmidrule(lr){2-5} \cmidrule(lr){6-9}
          \multicolumn{1}{c}{$d$} & \multicolumn{1}{c}{$k=2$}  & \multicolumn{1}{c}{$k=4$}   & \multicolumn{1}{c}{$k=8$} & $k=16$ & \multicolumn{1}{c}{$k=2$}  & \multicolumn{1}{c}{$k=4$} & \multicolumn{1}{c}{$k=8$} & $k=16$ \\ \hline
          \multirow{2}{*}{$6$}  & \multicolumn{1}{c}{0.05} & \multicolumn{1}{c}{0.16} & 0.01 & \multicolumn{1}{c}{0.05} & \multicolumn{1}{c}{0.95} & 0.88 & 0.99 & 0.95 \\
          \multirow{2}{*}{}  & \multicolumn{1}{c}{0.32$\pm0.04$} & \multicolumn{1}{c}{0.81$\pm0.03$} & 0.69$\pm0.01$ & \multicolumn{1}{c}{0.99$\pm0.00$} & \multicolumn{1}{c}{0.02$\pm0.00$} & 0.05$\pm0.02$ & 0.01$\pm0.00$ & 0.01$\pm0.00$ \\  \midrule
          \multirow{2}{*}{$8$}  & \multicolumn{1}{c}{0.20} & \multicolumn{1}{c}{0.06} & 0.14 & \multicolumn{1}{c}{0.01} & \multicolumn{1}{c}{0.88} & 0.98 & 0.85 & 0.99 \\
          \multirow{2}{*}{}  & \multicolumn{1}{c}{0.43$\pm0.02$} & \multicolumn{1}{c}{0.82$\pm0.02$} & 0.80$\pm0.00$ & \multicolumn{1}{c}{0.95$\pm0.03$} & \multicolumn{1}{c}{0.05$\pm0.01$} & 0.02$\pm0.01$ & 0.01$\pm0.00$ & 0.01$\pm0.00$ \\ \midrule
          \multirow{2}{*}{$10$}  & \multicolumn{1}{c}{0.19} & \multicolumn{1}{c}{0.04} & 0.02 & \multicolumn{1}{c}{0.04} & \multicolumn{1}{c}{0.86} & 0.98 & 0.99 & 0.97 \\
          \multirow{2}{*}{}  & \multicolumn{1}{c}{0.53$\pm0.02$} & \multicolumn{1}{c}{0.65$\pm0.02$} & 0.90$\pm0.01$ & \multicolumn{1}{c}{0.99$\pm0.00$} & \multicolumn{1}{c}{0.07$\pm0.01$} & 0.02$\pm0.00$ & 0.02$\pm0.00$ & 0.01$\pm0.00$ \\ \midrule
          \multirow{2}{*}{$12$}  & \multicolumn{1}{c}{0.05} & \multicolumn{1}{c}{0.18} & 0.13 & \multicolumn{1}{c}{0.11} & \multicolumn{1}{c}{0.97} & 0.90 & 0.92 & 0.93 \\
          \multirow{2}{*}{}  & \multicolumn{1}{c}{0.48$\pm0.03$} & \multicolumn{1}{c}{0.72$\pm0.02$} & 0.88$\pm0.03$ & \multicolumn{1}{c}{0.98$\pm0.00$} & \multicolumn{1}{c}{0.07$\pm0.02$} & 0.03$\pm0.00$ & 0.02$\pm0.00$ & 0.01$\pm0.00$ \\ \midrule
          \multirow{2}{*}{$14$}  & \multicolumn{1}{c}{0.11} & \multicolumn{1}{c}{0.16} & 0.17 & \multicolumn{1}{c}{0.04} & \multicolumn{1}{c}{0.95} & 0.94 & 0.83 & 0.98 \\
          \multirow{2}{*}{}  & \multicolumn{1}{c}{0.49$\pm0.02$} & \multicolumn{1}{c}{0.56$\pm0.02$} & 0.89$\pm0.01$ & \multicolumn{1}{c}{0.97$\pm0.02$} & \multicolumn{1}{c}{0.06$\pm0.01$} & 0.07$\pm0.01$ & 0.03$\pm0.00$ & 0.01$\pm0.00$ \\ \hline\hline
          \multirow{2}{*}{$6$}  & \multicolumn{1}{c}{0.30} & \multicolumn{1}{c}{0.02} & 0.30 & \multicolumn{1}{c}{0.07} & \multicolumn{1}{c}{0.71} & 0.97 & 0.74 & 0.94 \\
          \multirow{2}{*}{}  & \multicolumn{1}{c}{0.36$\pm0.03$} & \multicolumn{1}{c}{0.88$\pm0.03$} & 0.87$\pm0.04$ & \multicolumn{1}{c}{0.99$\pm0.00$} & \multicolumn{1}{c}{0.04$\pm0.02$} & 0.05$\pm0.01$ & 0.02$\pm0.00$ & 0.01$\pm0.00$ \\ \midrule
          \multirow{2}{*}{$8$}  & \multicolumn{1}{c}{0.18} & \multicolumn{1}{c}{0.18} & 0.13 & \multicolumn{1}{c}{0.21} & \multicolumn{1}{c}{0.86} & 0.87 & 0.89 & 0.80 \\
          \multirow{2}{*}{}  & \multicolumn{1}{c}{0.53$\pm0.04$} & \multicolumn{1}{c}{0.62$\pm0.03$} & 0.83$\pm0.04$ & \multicolumn{1}{c}{0.98$\pm0.00$} & \multicolumn{1}{c}{0.06$\pm0.03$} & 0.04$\pm0.01$ & 0.02$\pm0.00$ & 0.01$\pm0.00$ \\ \midrule
          \multirow{2}{*}{$10$}  & \multicolumn{1}{c}{0.21} & \multicolumn{1}{c}{0.12} & 0.22 & \multicolumn{1}{c}{0.07} & \multicolumn{1}{c}{0.87} & 0.84 & 0.87 & 0.94 \\
          \multirow{2}{*}{}  & \multicolumn{1}{c}{0.53$\pm0.02$} & \multicolumn{1}{c}{0.62$\pm0.01$} & 0.90$\pm0.01$ & \multicolumn{1}{c}{0.96$\pm0.00$} & \multicolumn{1}{c}{0.08$\pm0.03$} & 0.03$\pm0.00$ & 0.03$\pm0.00$ & 0.01$\pm0.00$ \\ \midrule
          \multirow{2}{*}{$12$}  & \multicolumn{1}{c}{0.15} & \multicolumn{1}{c}{0.23} & 0.14 & \multicolumn{1}{c}{0.17} & \multicolumn{1}{c}{0.88} & 0.81 & 0.87 & 0.85 \\
          \multirow{2}{*}{}  & \multicolumn{1}{c}{0.52$\pm0.02$} & \multicolumn{1}{c}{0.73$\pm0.02$} & 0.83$\pm0.01$ & \multicolumn{1}{c}{0.96$\pm0.00$} & \multicolumn{1}{c}{0.10$\pm0.02$} & 0.05$\pm0.00$ & 0.04$\pm0.00$ & 0.01$\pm0.00$ \\ \midrule
          \multirow{2}{*}{$14$}  & \multicolumn{1}{c}{0.22} & \multicolumn{1}{c}{0.16} & 0.25 & \multicolumn{1}{c}{0.22} & \multicolumn{1}{c}{0.69} & 0.79 & 0.80 & 0.73 \\
          \multirow{2}{*}{}  & \multicolumn{1}{c}{0.41$\pm0.01$} & \multicolumn{1}{c}{0.55$\pm0.01$} & 0.85$\pm0.01$ & \multicolumn{1}{c}{0.93$\pm0.00$} & \multicolumn{1}{c}{0.08$\pm0.01$} & 0.07$\pm0.01$ & 0.04$\pm0.00$ & 0.02$\pm0.00$ \\ \hline
        \end{tabular}
      }
  }
\end{table}

\begin{table}[hpbt]
  \centering{
    {\caption{$\beta$-VAE - Polynomial Mixing Dataset $R^2$ scores, \textbf{Independent SCM} DGP. The results are averaged over 5 seeds. $\hat{z},z\in R^{d}$ and $z_\mathcal{S}, z_{\mathcal{U}}\in R^{d/2}$ and $x=g(z)\in R^{200}$. Top section and bottom section correspond to polynomial degrees of 2 and 3. Each set of 4 rows correspond to a specific $d$ and from top to bottom denote the $R^2$ scores before training $\beta$-VAE, and after training with $\beta\in [0.1,1.0,10.0]$.}
        \label{table:beta_poly_indep_scm}
        }%
      {
        \begin{tabular}{c cccc cccc}
          \hline
          \multicolumn{1}{c}{}
        & \multicolumn{4}{c}{$R^2_{\mathcal{S}}$}
& \multicolumn{4}{c}{$R^2_{\mathcal{U}}$}  \\ \cmidrule(lr){2-5} \cmidrule(lr){6-9}
          \multicolumn{1}{c}{$d$} & \multicolumn{1}{c}{$k=2$}  & \multicolumn{1}{c}{$k=4$}   & \multicolumn{1}{c}{$k=8$} & $k=16$ & \multicolumn{1}{c}{$k=2$}  & \multicolumn{1}{c}{$k=4$} & \multicolumn{1}{c}{$k=8$} & $k=16$ \\ \hline
          \multirow{4}{*}{$6$}  & \multicolumn{1}{c}{0.37$\pm0.02$} & \multicolumn{1}{c}{0.19$\pm0.05$} & 0.35$\pm0.00$ & \multicolumn{1}{c}{0.19$\pm0.04$} & \multicolumn{1}{c}{0.96$\pm0.02$} & 0.99$\pm0.00$ & 0.98$\pm0.00$ & 0.99$\pm0.00$\\
          \multirow{4}{*}{}  & \multicolumn{1}{c}{0.12$\pm0.07$} & \multicolumn{1}{c}{0.03$\pm0.01$} & 0.08$\pm0.05$ & \multicolumn{1}{c}{0.02$\pm0.01$} & \multicolumn{1}{c}{0.91$\pm0.04$} & 0.95$\pm0.02$ & 0.84$\pm0.07$ & 0.95$\pm0.02$ \\ 
          \multirow{4}{*}{}  & \multicolumn{1}{c}{0.07$\pm0.05$} & \multicolumn{1}{c}{0.02$\pm0.00$} & 0.02$\pm0.00$ & \multicolumn{1}{c}{0.01$\pm0.00$} & \multicolumn{1}{c}{0.96$\pm0.01$} & 0.94$\pm0.04$ & 0.90$\pm0.08$ & 0.93$\pm0.04$ \\
          \multirow{4}{*}{}  & \multicolumn{1}{c}{0.09$\pm0.06$} & \multicolumn{1}{c}{0.01$\pm0.00$} & 0.02$\pm0.00$ & \multicolumn{1}{c}{0.01$\pm0.00$} & \multicolumn{1}{c}{0.97$\pm0.01$} & 0.96$\pm0.03$ & 0.95$\pm0.03$ & 0.92$\pm0.06$ \\ \midrule
          \multirow{4}{*}{$8$}  & \multicolumn{1}{c}{0.05$\pm0.01$} & \multicolumn{1}{c}{0.12$\pm0.05$} & 0.05$\pm0.03$ & \multicolumn{1}{c}{0.06$\pm0.04$} & \multicolumn{1}{c}{0.97$\pm0.00$} & 0.92$\pm0.04$ & 0.95$\pm0.04$ & 0.95$\pm0.04$\\
          \multirow{4}{*}{}  & \multicolumn{1}{c}{0.10$\pm0.04$} & \multicolumn{1}{c}{0.03$\pm0.01$} & 0.08$\pm0.06$ & \multicolumn{1}{c}{0.01$\pm0.00$} & \multicolumn{1}{c}{0.93$\pm0.03$} & 0.92$\pm0.03$ & 0.95$\pm0.03$ & 0.93$\pm0.03$ \\
          \multirow{4}{*}{}  & \multicolumn{1}{c}{0.07$\pm0.03$} & \multicolumn{1}{c}{0.02$\pm0.00$} & 0.04$\pm0.01$ & \multicolumn{1}{c}{0.03$\pm0.01$} & \multicolumn{1}{c}{0.91$\pm0.02$} & 0.96$\pm0.02$ & 0.94$\pm0.03$ & 0.91$\pm0.04$ \\
          \multirow{4}{*}{}  & \multicolumn{1}{c}{0.05$\pm0.02$} & \multicolumn{1}{c}{0.04$\pm0.01$} & 0.07$\pm0.04$ & \multicolumn{1}{c}{0.03$\pm0.01$} & \multicolumn{1}{c}{0.93$\pm0.03$} & 0.92$\pm0.03$ & 0.95$\pm0.04$ & 0.93$\pm0.04$ \\ \midrule
          \multirow{4}{*}{$10$}  & \multicolumn{1}{c}{0.06$\pm0.03$} & \multicolumn{1}{c}{0.05$\pm0.01$} & 0.04$\pm0.02$ & \multicolumn{1}{c}{0.02$\pm0.01$} & \multicolumn{1}{c}{0.80$\pm0.03$} & 0.83$\pm0.02$ & 0.77$\pm0.04$ & 0.80$\pm0.02$ \\
          \multirow{4}{*}{}  & \multicolumn{1}{c}{0.05$\pm0.01$} & \multicolumn{1}{c}{0.09$\pm0.03$} & 0.05$\pm0.02$ & \multicolumn{1}{c}{0.05$\pm0.03$} & \multicolumn{1}{c}{0.95$\pm0.01$} & 0.94$\pm0.03$ & 0.98$\pm0.00$ & 0.95$\pm0.01$ \\
          \multirow{4}{*}{}  & \multicolumn{1}{c}{0.09$\pm0.03$} & \multicolumn{1}{c}{0.12$\pm0.04$} & 0.04$\pm0.01$ & \multicolumn{1}{c}{0.09$\pm0.04$} & \multicolumn{1}{c}{0.94$\pm0.02$} & 0.96$\pm0.02$ & 0.95$\pm0.03$ & 0.95$\pm0.04$ \\
          \multirow{4}{*}{}  & \multicolumn{1}{c}{0.08$\pm0.03$} & \multicolumn{1}{c}{0.03$\pm0.00$} & 0.04$\pm0.01$ & \multicolumn{1}{c}{0.07$\pm0.04$} & \multicolumn{1}{c}{0.97$\pm0.00$} & 0.96$\pm0.03$ & 0.99$\pm0.00$ & 0.98$\pm0.01$ \\ \midrule
          \multirow{4}{*}{$12$}  & \multicolumn{1}{c}{0.04$\pm0.02$} & \multicolumn{1}{c}{0.03$\pm0.01$} & 0.02$\pm0.00$ & \multicolumn{1}{c}{0.01$\pm0.00$} & \multicolumn{1}{c}{0.70$\pm0.02$} & 0.79$\pm0.01$ & 0.72$\pm0.02$ & 0.69$\pm0.02$ \\
          \multirow{4}{*}{}  & \multicolumn{1}{c}{0.15$\pm0.03$} & \multicolumn{1}{c}{0.07$\pm0.02$} & 0.07$\pm0.03$ & \multicolumn{1}{c}{0.05$\pm0.02$} & \multicolumn{1}{c}{0.91$\pm0.02$} & 0.96$\pm0.01$ & 0.97$\pm0.00$ & 0.90$\pm0.03$ \\
          \multirow{4}{*}{}  & \multicolumn{1}{c}{0.08$\pm0.02$} & \multicolumn{1}{c}{0.09$\pm0.03$} & 0.06$\pm0.03$ & \multicolumn{1}{c}{0.10$\pm0.03$} & \multicolumn{1}{c}{0.97$\pm0.00$} & 0.95$\pm0.01$ & 0.95$\pm0.02$ & 0.96$\pm0.02$ \\
          \multirow{4}{*}{}  & \multicolumn{1}{c}{0.08$\pm0.03$} & \multicolumn{1}{c}{0.09$\pm0.03$} & 0.03$\pm0.01$ & \multicolumn{1}{c}{0.02$\pm0.00$} & \multicolumn{1}{c}{0.96$\pm0.01$} & 0.97$\pm0.00$ & 0.95$\pm0.03$ & 0.99$\pm0.00$ \\ \midrule
          \multirow{4}{*}{$14$}  & \multicolumn{1}{c}{0.02$\pm0.00$} & \multicolumn{1}{c}{0.03$\pm0.01$} & 0.02$\pm0.01$ & \multicolumn{1}{c}{0.01$\pm0.00$} & \multicolumn{1}{c}{0.70$\pm0.02$} & 0.68$\pm0.04$ & 0.65$\pm0.03$ & 0.58$\pm0.02$ \\
          \multirow{4}{*}{}  & \multicolumn{1}{c}{0.14$\pm0.01$} & \multicolumn{1}{c}{0.08$\pm0.02$} & 0.17$\pm0.02$ & \multicolumn{1}{c}{0.07$\pm0.02$} & \multicolumn{1}{c}{0.95$\pm0.01$} & 0.95$\pm0.00$ & 0.96$\pm0.01$ & 0.92$\pm0.02$ \\
          \multirow{4}{*}{}  & \multicolumn{1}{c}{0.08$\pm0.02$} & \multicolumn{1}{c}{0.07$\pm0.02$} & 0.07$\pm0.01$ & \multicolumn{1}{c}{0.11$\pm0.02$} & \multicolumn{1}{c}{0.94$\pm0.01$} & 0.94$\pm0.01$ & 0.96$\pm0.02$ & 0.93$\pm0.02$ \\
          \multirow{4}{*}{}  & \multicolumn{1}{c}{0.05$\pm0.00$} & \multicolumn{1}{c}{0.09$\pm0.03$} & 0.09$\pm0.03$ & \multicolumn{1}{c}{0.03$\pm0.01$} & \multicolumn{1}{c}{0.95$\pm0.02$} & 0.97$\pm0.00$ & 0.96$\pm0.02$ & 0.97$\pm0.01$ \\\hline\hline
          \multirow{4}{*}{$6$}  & \multicolumn{1}{c}{0.35$\pm0.03$} & \multicolumn{1}{c}{0.13$\pm0.04$} & 0.28$\pm0.01$ & \multicolumn{1}{c}{0.33$\pm0.00$} & \multicolumn{1}{c}{0.97$\pm0.03$} & 0.97$\pm0.02$ & 1.00$\pm0.00$ & 1.00$\pm0.00$ \\
          \multirow{4}{*}{}  & \multicolumn{1}{c}{0.13$\pm0.06$} & \multicolumn{1}{c}{0.02$\pm0.01$} & 0.02$\pm0.00$ & \multicolumn{1}{c}{0.03$\pm0.01$} & \multicolumn{1}{c}{0.86$\pm0.06$} & 0.93$\pm0.02$ & 0.97$\pm0.01$ & 0.97$\pm0.01$ \\
          \multirow{4}{*}{}  & \multicolumn{1}{c}{0.10$\pm0.05$} & \multicolumn{1}{c}{0.02$\pm0.00$} & 0.01$\pm0.00$ & \multicolumn{1}{c}{0.03$\pm0.01$} & \multicolumn{1}{c}{0.90$\pm0.03$} & 0.93$\pm0.03$ & 0.98$\pm0.01$ & 0.93$\pm0.04$ \\
          \multirow{4}{*}{}  & \multicolumn{1}{c}{0.05$\pm0.03$} & \multicolumn{1}{c}{0.02$\pm0.01$} & 0.01$\pm0.00$ & \multicolumn{1}{c}{0.02$\pm0.01$} & \multicolumn{1}{c}{0.97$\pm0.02$} & 0.95$\pm0.03$ & 0.99$\pm0.00$ & 0.99$\pm0.00$ \\ \midrule
          \multirow{4}{*}{$8$}  & \multicolumn{1}{c}{0.19$\pm0.04$} & \multicolumn{1}{c}{0.09$\pm0.03$} & 0.05$\pm0.02$ & \multicolumn{1}{c}{0.14$\pm0.05$} & \multicolumn{1}{c}{0.82$\pm0.06$} & 0.93$\pm0.02$ & 0.94$\pm0.02$ & 0.84$\pm0.07$ \\
          \multirow{4}{*}{}  & \multicolumn{1}{c}{0.08$\pm0.03$} & \multicolumn{1}{c}{0.14$\pm0.05$} & 0.08$\pm0.04$ & \multicolumn{1}{c}{0.05$\pm0.03$} & \multicolumn{1}{c}{0.87$\pm0.02$} & 0.92$\pm0.03$ & 0.96$\pm0.02$ & 0.94$\pm0.04$ \\
          \multirow{4}{*}{}  & \multicolumn{1}{c}{0.12$\pm0.04$} & \multicolumn{1}{c}{0.13$\pm0.04$} & 0.05$\pm0.03$ & \multicolumn{1}{c}{0.07$\pm0.02$} & \multicolumn{1}{c}{0.90$\pm0.04$} & 0.92$\pm0.03$ & 0.98$\pm0.00$ & 0.91$\pm0.03$ \\
          \multirow{4}{*}{}  & \multicolumn{1}{c}{0.05$\pm0.01$} & \multicolumn{1}{c}{0.13$\pm0.04$} & 0.03$\pm0.01$ & \multicolumn{1}{c}{0.01$\pm0.00$} & \multicolumn{1}{c}{0.94$\pm0.02$} & 0.97$\pm0.02$ & 0.97$\pm0.03$ & 0.88$\pm0.04$ \\ \midrule
          \multirow{4}{*}{$10$}  & \multicolumn{1}{c}{0.08$\pm0.04$} & \multicolumn{1}{c}{0.05$\pm0.01$} & 0.14$\pm0.03$ & \multicolumn{1}{c}{0.07$\pm0.03$} & \multicolumn{1}{c}{0.77$\pm0.05$} & 0.75$\pm0.03$ & 0.67$\pm0.05$ & 0.76$\pm0.04$ \\
          \multirow{4}{*}{}  & \multicolumn{1}{c}{0.12$\pm0.02$} & \multicolumn{1}{c}{0.07$\pm0.02$} & 0.13$\pm0.04$ & \multicolumn{1}{c}{0.08$\pm0.04$} & \multicolumn{1}{c}{0.95$\pm0.01$} & 0.90$\pm0.03$ & 0.88$\pm0.04$ & 0.94$\pm0.02$ \\
          \multirow{4}{*}{}  & \multicolumn{1}{c}{0.15$\pm0.01$} & \multicolumn{1}{c}{0.08$\pm0.04$} & 0.12$\pm0.03$ & \multicolumn{1}{c}{0.15$\pm0.03$} & \multicolumn{1}{c}{0.93$\pm0.02$} & 0.95$\pm0.01$ & 0.92$\pm0.03$ & 0.90$\pm0.03$ \\
          \multirow{4}{*}{}  & \multicolumn{1}{c}{0.08$\pm0.03$} & \multicolumn{1}{c}{0.04$\pm0.01$} & 0.05$\pm0.01$ & \multicolumn{1}{c}{0.10$\pm0.04$} & \multicolumn{1}{c}{0.94$\pm0.01$} & 0.93$\pm0.03$ & 0.96$\pm0.02$ & 0.98$\pm0.00$ \\ \midrule
          \multirow{4}{*}{$12$}  & \multicolumn{1}{c}{0.04$\pm0.01$} & \multicolumn{1}{c}{0.06$\pm0.01$} & 0.06$\pm0.01$ & \multicolumn{1}{c}{0.06$\pm0.01$} & \multicolumn{1}{c}{0.66$\pm0.03$} & 0.63$\pm0.04$ & 0.65$\pm0.03$ & 0.65$\pm0.02$ \\
          \multirow{4}{*}{}  & \multicolumn{1}{c}{0.16$\pm0.03$} & \multicolumn{1}{c}{0.13$\pm0.02$} & 0.15$\pm0.02$ & \multicolumn{1}{c}{0.15$\pm0.03$} & \multicolumn{1}{c}{0.89$\pm0.02$} & 0.91$\pm0.02$ & 0.88$\pm0.02$ & 0.86$\pm0.02$ \\
          \multirow{4}{*}{}  & \multicolumn{1}{c}{0.15$\pm0.01$} & \multicolumn{1}{c}{0.13$\pm0.01$} & 0.11$\pm0.01$ & \multicolumn{1}{c}{0.12$\pm0.02$} & \multicolumn{1}{c}{0.89$\pm0.02$} & 0.92$\pm0.02$ & 0.92$\pm0.01$ & 0.90$\pm0.02$ \\
          \multirow{4}{*}{}  & \multicolumn{1}{c}{0.13$\pm0.02$} & \multicolumn{1}{c}{0.11$\pm0.03$} & 0.12$\pm0.02$ & \multicolumn{1}{c}{0.11$\pm0.03$} & \multicolumn{1}{c}{0.93$\pm0.02$} & 0.91$\pm0.03$ & 0.93$\pm0.02$ & 0.96$\pm0.01$ \\ \midrule
          \multirow{4}{*}{$14$}  & \multicolumn{1}{c}{0.07$\pm0.02$} & \multicolumn{1}{c}{0.04$\pm0.01$} & 0.07$\pm0.01$ & \multicolumn{1}{c}{0.08$\pm0.01$} & \multicolumn{1}{c}{0.54$\pm0.02$} & 0.51$\pm0.04$ & 0.52$\pm0.01$ & 0.51$\pm0.03$ \\
          \multirow{4}{*}{}  & \multicolumn{1}{c}{0.13$\pm0.01$} & \multicolumn{1}{c}{0.09$\pm0.02$} & 0.16$\pm0.02$ & \multicolumn{1}{c}{0.11$\pm0.03$} & \multicolumn{1}{c}{0.81$\pm0.02$} & 0.81$\pm0.01$ & 0.88$\pm0.01$ & 0.80$\pm0.02$ \\
          \multirow{4}{*}{}  & \multicolumn{1}{c}{0.10$\pm0.01$} & \multicolumn{1}{c}{0.07$\pm0.01$} & 0.10$\pm0.02$ & \multicolumn{1}{c}{0.13$\pm0.02$} & \multicolumn{1}{c}{0.83$\pm0.01$} & 0.83$\pm0.02$ & 0.91$\pm0.01$ & 0.85$\pm0.02$ \\
          \multirow{4}{*}{}  & \multicolumn{1}{c}{0.08$\pm0.01$} & \multicolumn{1}{c}{0.05$\pm0.01$} & 0.07$\pm0.01$ & \multicolumn{1}{c}{0.10$\pm0.03$} & \multicolumn{1}{c}{0.85$\pm0.01$} & 0.92$\pm0.01$ & 0.92$\pm0.02$ & 0.91$\pm0.02$ \\\hline
        \end{tabular}
      }
  }
\end{table}

\begin{table}[hpbt]
  \centering{
    {\caption{$\beta$-VAE - Polynomial Mixing Dataset $R^2$ scores, \textbf{Dynamic SCM} DGP. The results are averaged over 5 seeds. $\hat{z},z\in R^{d}$ and $z_\mathcal{S}, z_{\mathcal{U}}\in R^{d/2}$ and $x=g(z)\in R^{200}$. Top section and bottom section correspond to polynomial degrees of 2 and 3. Each set of 4 rows correspond to a specific $d$ and from top to bottom denote the $R^2$ scores before training $\beta$-VAE, and after training with $\beta\in [0.1,1.0,10.0]$.}
        \label{table:beta_poly_dynamic_scm}
        }%
      {
        \begin{tabular}{c cccc cccc}
          \hline
          \multicolumn{1}{c}{}
        & \multicolumn{4}{c}{$R^2_{\mathcal{S}}$}
& \multicolumn{4}{c}{$R^2_{\mathcal{U}}$}  \\ \cmidrule(lr){2-5} \cmidrule(lr){6-9}
          \multicolumn{1}{c}{$d$} & \multicolumn{1}{c}{$k=2$}  & \multicolumn{1}{c}{$k=4$}   & \multicolumn{1}{c}{$k=8$} & $k=16$ & \multicolumn{1}{c}{$k=2$}  & \multicolumn{1}{c}{$k=4$} & \multicolumn{1}{c}{$k=8$} & $k=16$ \\ \hline
          \multirow{4}{*}{$6$}  & \multicolumn{1}{c}{0.27$\pm0.05$} & \multicolumn{1}{c}{0.41$\pm0.06$} & 0.23$\pm0.04$ & \multicolumn{1}{c}{0.34$\pm0.00$} & \multicolumn{1}{c}{0.98$\pm0.01$} & 0.92$\pm0.06$ & 0.99$\pm0.00$ & 0.99$\pm0.00$ \\
          \multirow{4}{*}{}  & \multicolumn{1}{c}{0.05$\pm0.01$} & \multicolumn{1}{c}{0.19$\pm0.04$} & 0.07$\pm0.03$ & \multicolumn{1}{c}{0.07$\pm0.03$} & \multicolumn{1}{c}{0.87$\pm0.06$} & 0.91$\pm0.03$ & 0.92$\pm0.04$ & 0.92$\pm0.02$ \\
          \multirow{4}{*}{}  & \multicolumn{1}{c}{0.07$\pm0.03$} & \multicolumn{1}{c}{0.25$\pm0.04$} & 0.07$\pm0.03$ & \multicolumn{1}{c}{0.07$\pm0.03$} & \multicolumn{1}{c}{0.91$\pm0.04$} & 0.83$\pm0.05$ & 0.89$\pm0.04$ & 0.89$\pm0.04$ \\
          \multirow{4}{*}{}  & \multicolumn{1}{c}{0.05$\pm0.01$} & \multicolumn{1}{c}{0.21$\pm0.06$} & 0.08$\pm0.04$ & \multicolumn{1}{c}{0.13$\pm0.05$} & \multicolumn{1}{c}{0.85$\pm0.07$} & 0.81$\pm0.05$ & 0.87$\pm0.04$ & 0.90$\pm0.04$ \\ \midrule
          \multirow{4}{*}{$8$}  & \multicolumn{1}{c}{0.18$\pm0.04$} & \multicolumn{1}{c}{0.05$\pm0.01$} & 0.05$\pm0.02$ & \multicolumn{1}{c}{0.03$\pm0.01$} & \multicolumn{1}{c}{0.89$\pm0.02$} & 0.96$\pm0.01$ & 0.96$\pm0.02$ & 0.97$\pm0.01$ \\
          \multirow{4}{*}{}  & \multicolumn{1}{c}{0.14$\pm0.05$} & \multicolumn{1}{c}{0.12$\pm0.03$} & 0.04$\pm0.01$ & \multicolumn{1}{c}{0.03$\pm0.01$} & \multicolumn{1}{c}{0.92$\pm0.02$} & 0.94$\pm0.03$ & 0.94$\pm0.01$ & 0.93$\pm0.02$ \\
          \multirow{4}{*}{}  & \multicolumn{1}{c}{0.14$\pm0.04$} & \multicolumn{1}{c}{0.12$\pm0.03$} & 0.08$\pm0.04$ & \multicolumn{1}{c}{0.05$\pm0.02$} & \multicolumn{1}{c}{0.86$\pm0.04$} & 0.91$\pm0.05$ & 0.92$\pm0.03$ & 0.90$\pm0.04$ \\
          \multirow{4}{*}{}  & \multicolumn{1}{c}{0.14$\pm0.03$} & \multicolumn{1}{c}{0.16$\pm0.02$} & 0.11$\pm0.04$ & \multicolumn{1}{c}{0.08$\pm0.04$} & \multicolumn{1}{c}{0.90$\pm0.01$} & 0.91$\pm0.06$ & 0.88$\pm0.04$ & 0.89$\pm0.04$ \\ \midrule
          \multirow{4}{*}{$10$}  & \multicolumn{1}{c}{0.06$\pm0.02$} & \multicolumn{1}{c}{0.02$\pm0.01$} & 0.03$\pm0.00$ & \multicolumn{1}{c}{0.05$\pm0.02$} & \multicolumn{1}{c}{0.81$\pm0.02$} & 0.83$\pm0.02$ & 0.79$\pm0.03$ & 0.77$\pm0.03$ \\
          \multirow{4}{*}{}  & \multicolumn{1}{c}{0.17$\pm0.03$} & \multicolumn{1}{c}{0.10$\pm0.03$} & 0.11$\pm0.03$ & \multicolumn{1}{c}{0.08$\pm0.02$} & \multicolumn{1}{c}{0.92$\pm0.01$} & 0.93$\pm0.02$ & 0.92$\pm0.03$ & 0.96$\pm0.01$ \\
          \multirow{4}{*}{}  & \multicolumn{1}{c}{0.17$\pm0.02$} & \multicolumn{1}{c}{0.08$\pm0.01$} & 0.14$\pm0.03$ & \multicolumn{1}{c}{0.09$\pm0.04$} & \multicolumn{1}{c}{0.89$\pm0.02$} & 0.92$\pm0.03$ & 0.95$\pm0.01$ & 0.92$\pm0.03$ \\
          \multirow{4}{*}{}  & \multicolumn{1}{c}{0.17$\pm0.03$} & \multicolumn{1}{c}{0.15$\pm0.02$} & 0.14$\pm0.04$ & \multicolumn{1}{c}{0.09$\pm0.04$} & \multicolumn{1}{c}{0.89$\pm0.02$} & 0.88$\pm0.04$ & 0.94$\pm0.02$ & 0.91$\pm0.04$ \\ \midrule
          \multirow{4}{*}{$12$}  & \multicolumn{1}{c}{0.04$\pm0.00$} & \multicolumn{1}{c}{0.03$\pm0.01$} & 0.02$\pm0.00$ & \multicolumn{1}{c}{0.03$\pm0.01$} & \multicolumn{1}{c}{0.69$\pm0.02$} & 0.79$\pm0.02$ & 0.71$\pm0.01$ & 0.69$\pm0.02$ \\
          \multirow{4}{*}{}  & \multicolumn{1}{c}{0.18$\pm0.01$} & \multicolumn{1}{c}{0.14$\pm0.02$} & 0.10$\pm0.03$ & \multicolumn{1}{c}{0.10$\pm0.03$} & \multicolumn{1}{c}{0.91$\pm0.01$} & 0.93$\pm0.02$ & 0.90$\pm0.01$ & 0.96$\pm0.01$ \\
          \multirow{4}{*}{}  & \multicolumn{1}{c}{0.14$\pm0.02$} & \multicolumn{1}{c}{0.12$\pm0.03$} & 0.09$\pm0.03$ & \multicolumn{1}{c}{0.11$\pm0.03$} & \multicolumn{1}{c}{0.91$\pm0.02$} & 0.93$\pm0.02$ & 0.90$\pm0.08$ & 0.94$\pm0.03$ \\
          \multirow{4}{*}{}  & \multicolumn{1}{c}{0.14$\pm0.03$} & \multicolumn{1}{c}{0.13$\pm0.03$} & 0.09$\pm0.03$ & \multicolumn{1}{c}{0.14$\pm0.03$} & \multicolumn{1}{c}{0.90$\pm0.02$} & 0.90$\pm0.03$ & 0.90$\pm0.01$ & 0.95$\pm0.02$ \\ \midrule
          \multirow{4}{*}{$14$}  & \multicolumn{1}{c}{0.03$\pm0.01$} & \multicolumn{1}{c}{0.03$\pm0.01$} & 0.02$\pm0.00$ & \multicolumn{1}{c}{0.01$\pm0.00$} & \multicolumn{1}{c}{0.68$\pm0.01$} & 0.69$\pm0.03$ & 0.64$\pm0.01$ & 0.58$\pm0.02$ \\
          \multirow{4}{*}{}  & \multicolumn{1}{c}{0.17$\pm0.01$} & \multicolumn{1}{c}{0.15$\pm0.03$} & 0.13$\pm0.02$ & \multicolumn{1}{c}{0.13$\pm0.03$} & \multicolumn{1}{c}{0.91$\pm0.01$} & 0.94$\pm0.01$ & 0.93$\pm0.02$ & 0.90$\pm0.01$ \\
          \multirow{4}{*}{}  & \multicolumn{1}{c}{0.19$\pm0.03$} & \multicolumn{1}{c}{0.16$\pm0.03$} & 0.15$\pm0.03$ & \multicolumn{1}{c}{0.15$\pm0.03$} & \multicolumn{1}{c}{0.91$\pm0.01$} & 0.91$\pm0.02$ & 0.88$\pm0.02$ & 0.87$\pm0.03$ \\
          \multirow{4}{*}{}  & \multicolumn{1}{c}{0.21$\pm0.02$} & \multicolumn{1}{c}{0.17$\pm0.03$} & 0.14$\pm0.02$ & \multicolumn{1}{c}{0.18$\pm0.03$} & \multicolumn{1}{c}{0.91$\pm0.01$} & 0.88$\pm0.02$ & 0.85$\pm0.03$ & 0.84$\pm0.03$ \\\hline\hline
          \multirow{4}{*}{$6$}  & \multicolumn{1}{c}{0.37$\pm0.04$} & \multicolumn{1}{c}{0.33$\pm0.01$} & 0.35$\pm0.01$ & \multicolumn{1}{c}{0.34$\pm0.00$} & \multicolumn{1}{c}{0.95$\pm0.04$} & 0.99$\pm0.00$ & 0.99$\pm0.00$ & 1.00$\pm0.00$ \\
          \multirow{4}{*}{}  & \multicolumn{1}{c}{0.07$\pm0.04$} & \multicolumn{1}{c}{0.10$\pm0.05$} & 0.09$\pm0.06$ & \multicolumn{1}{c}{0.03$\pm0.01$} & \multicolumn{1}{c}{0.94$\pm0.02$} & 0.94$\pm0.02$ & 0.95$\pm0.02$ & 0.90$\pm0.04$ \\
          \multirow{4}{*}{}  & \multicolumn{1}{c}{0.11$\pm0.06$} & \multicolumn{1}{c}{0.09$\pm0.05$} & 0.08$\pm0.06$ & \multicolumn{1}{c}{0.04$\pm0.02$} & \multicolumn{1}{c}{0.94$\pm0.02$} & 0.91$\pm0.04$ & 0.93$\pm0.03$ & 0.89$\pm0.05$ \\
          \multirow{4}{*}{}  & \multicolumn{1}{c}{0.06$\pm0.03$} & \multicolumn{1}{c}{0.09$\pm0.05$} & 0.08$\pm0.05$ & \multicolumn{1}{c}{0.06$\pm0.03$} & \multicolumn{1}{c}{0.92$\pm0.05$} & 0.92$\pm0.03$ & 0.96$\pm0.02$ & 0.91$\pm0.04$ \\ \midrule
          \multirow{4}{*}{$8$}  & \multicolumn{1}{c}{0.17$\pm0.04$} & \multicolumn{1}{c}{0.10$\pm0.04$} & 0.05$\pm0.02$ & \multicolumn{1}{c}{0.11$\pm0.04$} & \multicolumn{1}{c}{0.86$\pm0.06$} & 0.92$\pm0.03$ & 0.95$\pm0.02$ & 0.87$\pm0.05$ \\
          \multirow{4}{*}{}  & \multicolumn{1}{c}{0.13$\pm0.05$} & \multicolumn{1}{c}{0.10$\pm0.04$} & 0.07$\pm0.02$ & \multicolumn{1}{c}{0.09$\pm0.03$} & \multicolumn{1}{c}{0.87$\pm0.04$} & 0.90$\pm0.02$ & 0.84$\pm0.03$ & 0.87$\pm0.03$ \\
          \multirow{4}{*}{}  & \multicolumn{1}{c}{0.16$\pm0.05$} & \multicolumn{1}{c}{0.09$\pm0.04$} & 0.09$\pm0.02$ & \multicolumn{1}{c}{0.08$\pm0.02$} & \multicolumn{1}{c}{0.87$\pm0.03$} & 0.90$\pm0.03$ & 0.87$\pm0.03$ & 0.85$\pm0.03$ \\
          \multirow{4}{*}{}  & \multicolumn{1}{c}{0.21$\pm0.04$} & \multicolumn{1}{c}{0.12$\pm0.05$} & 0.12$\pm0.04$ & \multicolumn{1}{c}{0.13$\pm0.04$} & \multicolumn{1}{c}{0.87$\pm0.02$} & 0.92$\pm0.02$ & 0.88$\pm0.03$ & 0.84$\pm0.05$ \\ \midrule
          \multirow{4}{*}{$10$}  & \multicolumn{1}{c}{0.08$\pm0.03$} & \multicolumn{1}{c}{0.06$\pm0.01$} & 0.10$\pm0.02$ & \multicolumn{1}{c}{0.05$\pm0.02$} & \multicolumn{1}{c}{0.80$\pm0.03$} & 0.74$\pm0.03$ & 0.73$\pm0.02$ & 0.76$\pm0.03$ \\
          \multirow{4}{*}{}  & \multicolumn{1}{c}{0.19$\pm0.03$} & \multicolumn{1}{c}{0.17$\pm0.02$} & 0.16$\pm0.03$ & \multicolumn{1}{c}{0.09$\pm0.02$} & \multicolumn{1}{c}{0.88$\pm0.02$} & 0.82$\pm0.04$ & 0.85$\pm0.03$ & 0.85$\pm0.03$ \\
          \multirow{4}{*}{}  & \multicolumn{1}{c}{0.21$\pm0.02$} & \multicolumn{1}{c}{0.20$\pm0.03$} & 0.18$\pm0.04$ & \multicolumn{1}{c}{0.16$\pm0.03$} & \multicolumn{1}{c}{0.84$\pm0.03$} & 0.83$\pm0.03$ & 0.82$\pm0.04$ & 0.79$\pm0.02$ \\
          \multirow{4}{*}{}  & \multicolumn{1}{c}{0.22$\pm0.02$} & \multicolumn{1}{c}{0.20$\pm0.02$} & 0.16$\pm0.04$ & \multicolumn{1}{c}{0.19$\pm0.02$} & \multicolumn{1}{c}{0.81$\pm0.04$} & 0.81$\pm0.03$ & 0.79$\pm0.04$ & 0.79$\pm0.03$ \\ \midrule
          \multirow{4}{*}{$12$}  & \multicolumn{1}{c}{0.05$\pm0.01$} & \multicolumn{1}{c}{0.06$\pm0.01$} & 0.07$\pm0.02$ & \multicolumn{1}{c}{0.04$\pm0.00$} & \multicolumn{1}{c}{0.65$\pm0.03$} & 0.65$\pm0.03$ & 0.64$\pm0.03$ & 0.67$\pm0.02$ \\
          \multirow{4}{*}{}  & \multicolumn{1}{c}{0.17$\pm0.02$} & \multicolumn{1}{c}{0.18$\pm0.01$} & 0.20$\pm0.01$ & \multicolumn{1}{c}{0.13$\pm0.03$} & \multicolumn{1}{c}{0.84$\pm0.01$} & 0.82$\pm0.02$ & 0.82$\pm0.03$ & 0.83$\pm0.03$ \\
          \multirow{4}{*}{}  & \multicolumn{1}{c}{0.17$\pm0.02$} & \multicolumn{1}{c}{0.18$\pm0.01$} & 0.22$\pm0.01$ & \multicolumn{1}{c}{0.13$\pm0.02$} & \multicolumn{1}{c}{0.81$\pm0.01$} & 0.84$\pm0.02$ & 0.82$\pm0.02$ & 0.80$\pm0.02$ \\
          \multirow{4}{*}{}  & \multicolumn{1}{c}{0.16$\pm0.02$} & \multicolumn{1}{c}{0.19$\pm0.01$} & 0.23$\pm0.01$ & \multicolumn{1}{c}{0.12$\pm0.02$} & \multicolumn{1}{c}{0.80$\pm0.02$} & 0.82$\pm0.02$ & 0.81$\pm0.02$ & 0.80$\pm0.02$ \\ \midrule
          \multirow{4}{*}{$14$}  & \multicolumn{1}{c}{0.06$\pm0.02$} & \multicolumn{1}{c}{0.05$\pm0.01$} & 0.06$\pm0.01$ & \multicolumn{1}{c}{0.08$\pm0.01$} & \multicolumn{1}{c}{0.51$\pm0.07$} & 0.55$\pm0.07$ & 0.53$\pm0.01$ & 0.50$\pm0.03$ \\
          \multirow{4}{*}{}  & \multicolumn{1}{c}{0.18$\pm0.03$} & \multicolumn{1}{c}{0.20$\pm0.01$} & 0.21$\pm0.01$ & \multicolumn{1}{c}{0.18$\pm0.01$} & \multicolumn{1}{c}{0.82$\pm0.03$} & 0.82$\pm0.02$ & 0.83$\pm0.02$ & 0.79$\pm0.01$ \\
          \multirow{4}{*}{}  & \multicolumn{1}{c}{0.19$\pm0.03$} & \multicolumn{1}{c}{0.22$\pm0.02$} & 0.25$\pm0.02$ & \multicolumn{1}{c}{0.23$\pm0.03$} & \multicolumn{1}{c}{0.78$\pm0.03$} & 0.80$\pm0.03$ & 0.78$\pm0.02$ & 0.75$\pm0.02$ \\
          \multirow{4}{*}{}  & \multicolumn{1}{c}{0.21$\pm0.02$} & \multicolumn{1}{c}{0.22$\pm0.02$} & 0.25$\pm0.01$ & \multicolumn{1}{c}{0.25$\pm0.03$} & \multicolumn{1}{c}{0.75$\pm0.03$} & 0.80$\pm0.02$ & 0.75$\pm0.03$ & 0.75$\pm0.02$ \\\hline
        \end{tabular}
      }
  }
\end{table}

\subsection{Balls Dataset}
\label{appendix:balls_dataset}

\subsubsection{Data Generation and Model Architecture}

In Tables~\ref{table: balls_indep_supp}, \ref{table: balls_dynamic_supp} we provide additional results when $z$'s (i.e., balls' coordinates) follow an independent and dynamic SCM. As described in the main body, we observe that  as the number of domains increase we achieve high $R^2_{\mathcal{S}}$ and low  $R^2_{\mathcal{U}}$ and especially under the combination of Min-Max and MMD penalty. When $z$'s follow an independent SCM, $z_{\mathcal{S}}$, the invariant block of $z$ corresponds to the coordinates of the ball that is always sampled in an $m\times n$ rectangle that is at a fix location across all domains. The other ball that accounts for $z_{\mathcal{U}}$ is sampled from an $m'\times n'$ rectangle whose location varies across the $k$ domains. When $z$'s follow a dynamic SCM, we alter each component of $z_{\mathcal{U}}$ with probability 0.5 by adding or subtracting its counterpart in $z_{\mathcal{S}}$, subject to the constraints that $z_{\mathcal{U}}$ remains inside the frame, and that the two balls do not overlap to violate the injectivity assumption. The training and validation splits comprise 50000 and 10000 samples, respectively. We conducted experiments for a varying number of domains $k$. We divide the table into three sections with top block corresponding to Min-Max penalty, the middle block corresponding to the MMD penalty and the bottom block corresponding to the combination of the two denoted Min-Max + MMD. For each penalty we present two rows, the top row corresponding to the $R^2$ scores after training an autoencoder with reconstruction objective only, and right before enforcing any distributional invariances. Again, since we only need an autoencoder that can fully reconstruct the input, there is no need for training multiple perfect autoencoders, hence there is no standard error reported for such entries. We then take the perfectly trained autoencoder and enforce the distributional invariance penalty with 5 seeds, and present the results in the bottom row per each penalty. Our autoencoder architecture comprises a ResNet18 \citep{he2016deep} encoder with standard deconvolutional layers in the decoder. We closely follow the architecture from  \citet{ahuja2022interventional}. In all experiments, the encoder's output is 128-dimensional, and the invariance penalty is enforced on the first 64-dimensional block of encoder's output. For sample reconstructions see Figure \ref{figure: balls_recons_supp}.

\subsection{Unlabeled colored MNIST}
\label{appendix:unlabeled_colored_mnist}

\subsubsection{Data Generation and Model Architecture}

\paragraph{Data Generation}  All of the digit pixels will be coloured according to $z_{\mathcal{U}}$. The background remains untouched (coloured digits on black background). Now we describe the colouring scheme across domains and different SCMs.\\
\textbf{Independent SCM.} For each domain $i\in [k]$ we sample $l_c^i, h_c^i\sim \text{Uniform}[0, 1]$, such that $c\sim \text{Uniform}[l_c^i, h_c^i]$, where $c\in\{r,g,b\}$ denotes the colour channel. In other terms, each of the RGB channels comes from a uniform distribution that is unique to each domain $i$. The digits then are coloured by sampling $z_{\mathcal{U}}=(r,g,b)$ for each domain.\\
\textbf{Dynamic SCM.} To obtain a Dynamic SCM, we introduce a probabilistic relation among digits and $z_{\mathcal{U}}$ as follows. For any domain $i$, we sample each channel $c\sim \text{Uniform}[l_c^i, h_c^i]$ with a probability of 0.2, and with a probability of 0.8 we introduce the following relation among digits and the colours. If the image contains digits from 0-4, the channels are sampled according to $c\sim \text{Uniform}[l_c^i, (l_c^i+h_c^i)/2]$, and if the image contains digits from 5-9, the channels are sampled according to $c\sim \text{Uniform}[(l_c^i+h_c^i)/2, h_c^i]$. In simple words, most of the time we introduce a correlation between the digits and the colours, and for a small portion of the dataset, digits and colours are sampled independently, thus overall, we achieve a Dynamic SCM. 

\paragraph{Model Architecture}
All experiments are carried out in two stages similar to polynomial mixing, and balls image datasets. The architectures of the autoencoders at stages 1,2 are given in Tables \ref{table: mnist_architecture_stage_1},\ref{table: mnist_architecture_stage_2}.

\begin{table}[hpbt]
  \centering{
    {\caption{Autoencoder architecture for stage 1. First section presents the encoder layers, and the second section presents the decoder layers.}
        \label{table: mnist_architecture_stage_1}}%
      {
        \begin{tabular}{cccccc}
          \hline
          Layer      & Input Size & Output Size & Bias  & Activation & BatchNorm \\ \hline
          Linear (1) & 784         & 256          & True  & ReLU   & True    \\
          Linear (2) & 256         & 256          & True  & ReLU   & True   \\
          Linear (3) & 256         & 128          & True & ReLU   & True    \\\hline
          Linear (1) & 128         & 256          & True & ReLU   & True    \\
          Linear (2) & 256         & 256          & True  & ReLU   & True    \\
          Linear (3) & 256         & 784          & True  & -   & False   \\\hline
        \end{tabular}
      }}
\end{table}

\begin{table}[hpbt]
  \centering{
    {\caption{Autoencoder architecture for stage 2. First section presents the encoder layers, and the second section presents the decoder layers.}
        \label{table: mnist_architecture_stage_2}}%
      {
        \begin{tabular}{cccccc}
          \hline
          Layer      & Input Size & Output Size & Bias  & Activation & BatchNorm \\ \hline
          Linear (1) & 128         & 200          & True  & LeakyReLU(0.2)   & True    \\
          Linear (2) & 200         & 200          & True  & LeakyReLU(0.2)   & True   \\
          Linear (3) & 200         & 200          & True & LeakyReLU(0.2)   & True    \\
          Linear (3) & 200         & 128          & True & -   & False    \\\hline
          Linear (1) & 128         & 200          & True & LeakyReLU(0.2)   & True    \\
          Linear (2) & 200         & 200          & True  & LeakyReLU(0.2)   & True   \\
          Linear (3) & 200         & 200          & True & LeakyReLU(0.2)   & True    \\
          Linear (3) & 200         & 128          & True & -   & False    \\\hline
        \end{tabular}
      }}
\end{table}

The results for the Independent and Dynamic SCM are given in Tables \ref{table: mnist_digits_indep_supp},\ref{table: mnist_digits_dynamic_supp}, respectively.

\begin{table}[h]
  \centering{
    {\caption{Balls Dataset $R^2$, \textbf{Independent SCM} DGP. For each penalty, the top row corresponds to the scores after training the autoencoder, and the bottom row denotes the scores after enforcing distributional invariances. The results are averaged over 5 seeds. $\hat{z}\in R^{128}$ and $z_\mathcal{S}, z_{\mathcal{U}}\in R^{64}$. The underlying latent $z\in R^4$. The sections are Min-Max, MMD, and the combination, respectively.}
        \label{table: balls_indep_supp}
        }%
      {
        \begin{tabular}{cccc cccc}
          \hline
            \multicolumn{4}{c}{$R^2_{\mathcal{S}}$}
            & \multicolumn{4}{c}{$R^2_{\mathcal{U}}$}  \\ \cmidrule(lr){1-4} \cmidrule(lr){5-8}
          \multicolumn{1}{c}{$k=2$} & \multicolumn{1}{c}{$k=4$}   & \multicolumn{1}{c}{$k=8$} & $k=16$  & \multicolumn{1}{c}{$k=2$} & \multicolumn{1}{c}{$k=4$} & \multicolumn{1}{c}{$k=8$} & $k=16$ \\ \hline
           \multicolumn{1}{c}{0.99} & \multicolumn{1}{c}{0.99} & \multicolumn{1}{c}{0.99} & 0.99 & \multicolumn{1}{c}{0.99} & \multicolumn{1}{c}{0.98} & \multicolumn{1}{c}{0.98} & 0.97 \\
           \multicolumn{1}{c}{0.94$\pm0.00$} & \multicolumn{1}{c}{0.89$\pm0.01$} & \multicolumn{1}{c}{0.85$\pm0.04$} & 0.65$\pm0.01$ & \multicolumn{1}{c}{0.88$\pm0.00$} & \multicolumn{1}{c}{0.66$\pm0.02$} & \multicolumn{1}{c}{0.56$\pm0.04$} & 0.19$\pm0.01$ \\\hline
           \multicolumn{1}{c}{0.99} & \multicolumn{1}{c}{0.99} & \multicolumn{1}{c}{0.99} & 0.99 & \multicolumn{1}{c}{0.99} & \multicolumn{1}{c}{0.98} & \multicolumn{1}{c}{0.99} & 0.97 \\
           \multicolumn{1}{c}{0.76$\pm0.01$} & \multicolumn{1}{c}{0.83$\pm0.03$} & \multicolumn{1}{c}{0.67$\pm0.05$} & 0.63$\pm0.04$ & \multicolumn{1}{c}{0.65$\pm0.02$} & \multicolumn{1}{c}{0.72$\pm0.04$} & \multicolumn{1}{c}{0.57$\pm0.01$} & 0.27$\pm0.05$ \\\hline
           \multicolumn{1}{c}{0.99} & \multicolumn{1}{c}{0.99} & \multicolumn{1}{c}{0.99}     & 0.99 & \multicolumn{1}{c}{0.99} & \multicolumn{1}{c}{0.98} & \multicolumn{1}{c}{0.99} & 0.97 \\
           \multicolumn{1}{c}{0.77$\pm0.01$} & \multicolumn{1}{c}{0.91$\pm0.01$} & \multicolumn{1}{c}{0.86$\pm0.03$}     & 0.81$\pm0.04$ & \multicolumn{1}{c}{0.30$\pm0.01$} & \multicolumn{1}{c}{0.11$\pm0.01$} & \multicolumn{1}{c}{0.14$\pm0.01$} & 0.18$\pm0.02$ \\\hline
        \end{tabular}
      }
  }
\end{table}

\begin{table}[h]
  \centering{
    {\caption{Balls Dataset $R^2$, \textbf{Dynamic SCM} DGP. For each penalty, the top row corresponds to the scores after training the autoencoder, and the bottom row denotes the scores after enforcing distributional invariances. The results are averaged over 5 seeds. $\hat{z}\in R^{128}$ and $z_\mathcal{S}, z_{\mathcal{U}}\in R^{64}$. The underlying latent $z\in R^4$. The sections are Min-Max, MMD, and the combination, respectively.}
        \label{table: balls_dynamic_supp}
        }%
      {
        \begin{tabular}{cccc cccc}
          \hline
        \multicolumn{4}{c}{$R^2_{\mathcal{S}}$}
& \multicolumn{4}{c}{$R^2_{\mathcal{U}}$}  \\ \cmidrule(lr){1-4} \cmidrule(lr){5-8}
          \multicolumn{1}{c}{$k=2$} & \multicolumn{1}{c}{$k=4$}   & \multicolumn{1}{c}{$k=8$} & $k=16$ & \multicolumn{1}{c}{$k=2$} & \multicolumn{1}{c}{$k=4$} & \multicolumn{1}{c}{$k=8$} & $k=16$ \\ \hline
          \multicolumn{1}{c}{0.99}  & \multicolumn{1}{c}{0.99} & \multicolumn{1}{c}{0.99} & 0.99 & \multicolumn{1}{c}{0.97} & \multicolumn{1}{c}{0.95} & \multicolumn{1}{c}{0.99} & 0.99 \\
          \multicolumn{1}{c}{0.93$\pm0.00$} & \multicolumn{1}{c}{0.77$\pm0.03$} & \multicolumn{1}{c}{0.42$\pm0.01$} & 0.61$\pm0.03$ & \multicolumn{1}{c}{0.92$\pm0.01$} & \multicolumn{1}{c}{0.84$\pm0.00$} & \multicolumn{1}{c}{0.67$\pm0.04$} & 0.22$\pm0.01$ \\\hline
          \multicolumn{1}{c}{0.99} & \multicolumn{1}{c}{0.99} & \multicolumn{1}{c}{0.99} & 0.99  & \multicolumn{1}{c}{0.97} & \multicolumn{1}{c}{0.95} & \multicolumn{1}{c}{0.99} & 0.99 \\
          \multicolumn{1}{c}{0.55$\pm0.01$} & \multicolumn{1}{c}{0.46$\pm0.01$} & \multicolumn{1}{c}{0.31$\pm0.01$} & 0.55$\pm0.12$ & \multicolumn{1}{c}{0.67$\pm0.01$} & \multicolumn{1}{c}{0.46$\pm0.01$} & \multicolumn{1}{c}{0.32$\pm0.01$} & 0.15$\pm0.04$ \\\hline
          \multicolumn{1}{c}{0.99}  & \multicolumn{1}{c}{0.99} & \multicolumn{1}{c}{0.99} & 0.99 & \multicolumn{1}{c}{0.97} & \multicolumn{1}{c}{0.95} & \multicolumn{1}{c}{0.99} & 0.99 \\
           \multicolumn{1}{c}{0.73$\pm0.01$} & \multicolumn{1}{c}{0.71$\pm0.03$} & \multicolumn{1}{c}{0.77$\pm0.02$}     & 0.82$\pm0.02$ & \multicolumn{1}{c}{0.35$\pm0.02$} & \multicolumn{1}{c}{0.22$\pm0.01$} & \multicolumn{1}{c}{0.19$\pm0.01$} & 0.20$\pm0.04$ \\\hline
        \end{tabular}
      }
  }
\end{table}

\begin{figure}[h]
    \centering
    \includegraphics[ width=4.5in]{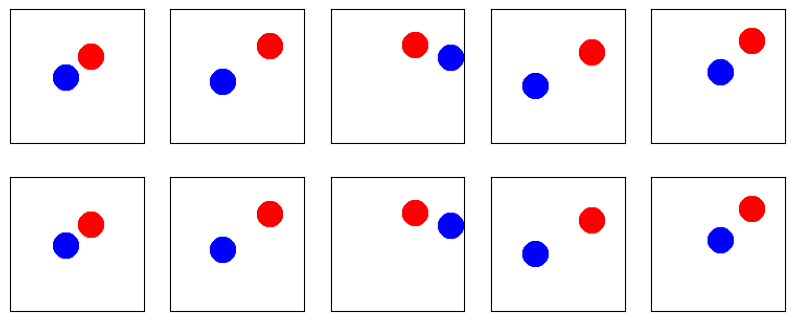}
    \caption{The top row shows the inputs to the image autoencoder, and the bottom row shows model's reconstructions.}
    \label{figure: balls_recons_supp}
\end{figure}

\begin{table}[hpbt]
  \centering{
    {\caption{$\beta$-VAE - Balls Dataset $R^2$, \textbf{Independent SCM} DGP. The top row corresponds to the scores after training the autoencoder and before enforcing $\beta$-VAE's KL divergence constraint, and the rest denote the disentanglement performance after training the $\beta$-VAE for each value of $\beta$. $\beta$-VAE The results are averaged over 5 seeds. $\hat{z}\in R^{128}$ and $z_\mathcal{S}, z_{\mathcal{U}}\in R^{64}$. The underlying latent $z\in R^4$.}
        \label{table:beta_balls_indep_scm}
        }%
      {
        \begin{tabular}{c cccc cccc}
          \hline
            \multicolumn{1}{c}{} & \multicolumn{4}{c}{$R^2_{\mathcal{S}}$}
            & \multicolumn{4}{c}{$R^2_{\mathcal{U}}$}  \\ \cmidrule(lr){2-5} \cmidrule(lr){6-9}
          \multicolumn{1}{c}{$\beta$} & \multicolumn{1}{c}{$k=2$} & \multicolumn{1}{c}{$k=4$}   & \multicolumn{1}{c}{$k=8$} & $k=16$  & \multicolumn{1}{c}{$k=2$} & \multicolumn{1}{c}{$k=4$} & \multicolumn{1}{c}{$k=8$} & $k=16$ \\ \hline
          \multicolumn{1}{c}{} & \multicolumn{1}{c}{0.28$\pm0.02$} & \multicolumn{1}{c}{0.24$\pm0.04$} & \multicolumn{1}{c}{0.25$\pm0.02$} & 0.25$\pm0.03$ & \multicolumn{1}{c}{0.29$\pm0.06$} & \multicolumn{1}{c}{0.25$\pm0.03$} & \multicolumn{1}{c}{0.34$\pm0.04$} & 0.26$\pm0.04$ \\
           \multicolumn{1}{c}{0.1} & \multicolumn{1}{c}{0.97$\pm0.00$} & \multicolumn{1}{c}{0.97$\pm0.00$} & \multicolumn{1}{c}{0.96$\pm0.00$} & 0.96$\pm0.00$ & \multicolumn{1}{c}{0.98$\pm0.00$} & \multicolumn{1}{c}{0.98$\pm0.00$} & \multicolumn{1}{c}{0.97$\pm0.00$} & 0.93$\pm0.01$ \\
           \multicolumn{1}{c}{1.0} & \multicolumn{1}{c}{0.96$\pm0.00$} & \multicolumn{1}{c}{0.94$\pm0.01$} & \multicolumn{1}{c}{0.93$\pm0.00$} & 0.92$\pm0.01$ & \multicolumn{1}{c}{0.95$\pm0.01$} & \multicolumn{1}{c}{0.97$\pm0.00$} & \multicolumn{1}{c}{0.97$\pm0.00$} & 0.93$\pm0.01$ \\
           \multicolumn{1}{c}{10.0} & \multicolumn{1}{c}{0.83$\pm0.03$} & \multicolumn{1}{c}{0.82$\pm0.03$} & \multicolumn{1}{c}{0.77$\pm0.02$} & 0.77$\pm0.02$ & \multicolumn{1}{c}{0.92$\pm0.01$} & \multicolumn{1}{c}{0.93$\pm0.01$} & \multicolumn{1}{c}{0.94$\pm0.00$} & 0.87$\pm0.01$ \\\hline
        \end{tabular}
      }
  }
\end{table}

\begin{table}[hpbt]
  \centering{
    {\caption{$\beta$-VAE - Balls Dataset $R^2$, \textbf{Dynamic SCM} DGP. The top row corresponds to the scores after training the autoencoder and before enforcing $\beta$-VAE's KL divergence constraint, and the rest denote the disentanglement performance after training the $\beta$-VAE for each value of $\beta$. $\beta$-VAE The results are averaged over 5 seeds. $\hat{z}\in R^{128}$ and $z_\mathcal{S}, z_{\mathcal{U}}\in R^{64}$. The underlying latent $z\in R^4$.}
        \label{table:beta_balls_dynamic_scm}
        }%
      {
        \begin{tabular}{c cccc cccc}
          \hline
            \multicolumn{1}{c}{} & \multicolumn{4}{c}{$R^2_{\mathcal{S}}$}
            & \multicolumn{4}{c}{$R^2_{\mathcal{U}}$}  \\ \cmidrule(lr){2-5} \cmidrule(lr){6-9}
          \multicolumn{1}{c}{$\beta$} & \multicolumn{1}{c}{$k=2$} & \multicolumn{1}{c}{$k=4$}   & \multicolumn{1}{c}{$k=8$} & $k=16$  & \multicolumn{1}{c}{$k=2$} & \multicolumn{1}{c}{$k=4$} & \multicolumn{1}{c}{$k=8$} & $k=16$ \\ \hline
          \multicolumn{1}{c}{} & \multicolumn{1}{c}{0.29$\pm0.03$} & \multicolumn{1}{c}{0.32$\pm0.03$} & \multicolumn{1}{c}{0.29$\pm0.04$} & 0.37$\pm0.03$ & \multicolumn{1}{c}{0.34$\pm0.03$} & \multicolumn{1}{c}{0.29$\pm0.03$} & \multicolumn{1}{c}{0.26$\pm0.04$} & 0.37$\pm0.05$ \\
           \multicolumn{1}{c}{0.1} & \multicolumn{1}{c}{0.96$\pm0.00$} & \multicolumn{1}{c}{0.96$\pm0.00$} & \multicolumn{1}{c}{0.95$\pm0.00$} & 0.95$\pm0.00$ & \multicolumn{1}{c}{0.94$\pm0.01$} & \multicolumn{1}{c}{0.93$\pm0.00$} & \multicolumn{1}{c}{0.98$\pm0.00$} & 0.98$\pm0.00$ \\
           \multicolumn{1}{c}{1.0} & \multicolumn{1}{c}{0.94$\pm0.00$} & \multicolumn{1}{c}{0.94$\pm0.01$} & \multicolumn{1}{c}{0.92$\pm0.01$} & 0.90$\pm0.01$ & \multicolumn{1}{c}{0.93$\pm0.00$} & \multicolumn{1}{c}{0.93$\pm0.01$} & \multicolumn{1}{c}{0.97$\pm0.00$} & 0.98$\pm0.00$ \\
           \multicolumn{1}{c}{10.0} & \multicolumn{1}{c}{0.82$\pm0.02$} & \multicolumn{1}{c}{0.81$\pm0.05$} & \multicolumn{1}{c}{0.77$\pm0.02$} & 0.74$\pm0.03$ & \multicolumn{1}{c}{0.88$\pm0.01$} & \multicolumn{1}{c}{0.89$\pm0.01$} & \multicolumn{1}{c}{0.94$\pm0.00$} & 0.94$\pm0.01$ \\\hline
        \end{tabular}
      }
  }
\end{table}

\begin{table}[h]
  \centering{
    {\caption{MNIST, \textbf{Coloured Digits}, \textbf{Independent SCM} DGP. The top row corresponds to the scores after training the autoencoder, and the following rows denote the scores after enforcing distributional invariances through Min-Max penalty, MMD, and the combination, respectively. The results are averaged over 5 seeds. $\hat{z}\in R^{128}$ and $\hat{z}_{\mathcal{\hat{S}}}, \hat{z}_{\mathcal{\hat{U}}}\in R^{64}$.}
        \label{table: mnist_digits_indep_supp}
        }%
      {
        \begin{tabular}{cccc cccc}
          \hline
        \multicolumn{4}{c}{Digits Classification Accuracy}
& \multicolumn{4}{c}{Colours $R^2_{\mathcal{U}}$}  \\ \cmidrule(lr){1-4} \cmidrule(lr){5-8}
          \multicolumn{1}{c}{$k=2$} & \multicolumn{1}{c}{$k=4$}   & \multicolumn{1}{c}{$k=8$} & $k=16$ & \multicolumn{1}{c}{$k=2$} & \multicolumn{1}{c}{$k=4$} & \multicolumn{1}{c}{$k=8$} & $k=16$ \\ \hline
          \multicolumn{1}{c}{0.87}  & \multicolumn{1}{c}{0.33} & \multicolumn{1}{c}{0.33} & 0.32 & \multicolumn{1}{c}{0.76} & \multicolumn{1}{c}{0.67} & \multicolumn{1}{c}{0.73} & 0.74 \\
          \multicolumn{1}{c}{0.71$\pm0.02$} & \multicolumn{1}{c}{0.59$\pm0.01$} & \multicolumn{1}{c}{0.58$\pm0.01$} & 0.66$\pm0.01$ & \multicolumn{1}{c}{0.72$\pm0.02$} & \multicolumn{1}{c}{0.55$\pm0.01$} & \multicolumn{1}{c}{0.51$\pm0.03$} & 0.49$\pm0.02$ \\
          \multicolumn{1}{c}{0.72$\pm0.01$} & \multicolumn{1}{c}{0.70$\pm0.01$} & \multicolumn{1}{c}{0.73$\pm0.01$} & 0.73$\pm0.01$ & \multicolumn{1}{c}{0.77$\pm0.01$} & \multicolumn{1}{c}{0.64$\pm0.01$} & \multicolumn{1}{c}{0.64$\pm0.02$} & 0.63$\pm0.02$ \\
           \multicolumn{1}{c}{0.73$\pm0.02$} & \multicolumn{1}{c}{0.70$\pm0.02$} & \multicolumn{1}{c}{0.74$\pm0.00$}     & 0.74$\pm0.01$ & \multicolumn{1}{c}{0.73$\pm0.02$} & \multicolumn{1}{c}{0.54$\pm0.02$} & \multicolumn{1}{c}{0.38$\pm0.01$} & 0.28$\pm0.01$ \\\hline
        \end{tabular}
      }
  }
\end{table}

\begin{table}[h]
  \centering{
    {\caption{MNIST, \textbf{Coloured Digits}, \textbf{Dynamic SCM} DGP. The top row corresponds to the scores after training the autoencoder, and the following rows denote the scores after enforcing distributional invariances through Min-Max penalty, MMD, and the combination, respectively. The results are averaged over 5 seeds. $\hat{z}\in R^{128}$ and $\hat{z}_{\mathcal{\hat{S}}}, \hat{z}_{\mathcal{\hat{U}}}\in R^{64}$.}
        \label{table: mnist_digits_dynamic_supp}
        }%
      {
        \begin{tabular}{cccc cccc}
          \hline
        \multicolumn{4}{c}{Digits Classification Accuracy}
& \multicolumn{4}{c}{Colours $R^2_{\mathcal{U}}$}  \\ \cmidrule(lr){1-4} \cmidrule(lr){5-8}
          \multicolumn{1}{c}{$k=2$} & \multicolumn{1}{c}{$k=4$}   & \multicolumn{1}{c}{$k=8$} & $k=16$ & \multicolumn{1}{c}{$k=2$} & \multicolumn{1}{c}{$k=4$} & \multicolumn{1}{c}{$k=8$} & $k=16$ \\ \hline
          \multicolumn{1}{c}{0.84}  & \multicolumn{1}{c}{0.90} & \multicolumn{1}{c}{0.70} & 0.75 & \multicolumn{1}{c}{0.81} & \multicolumn{1}{c}{0.55} & \multicolumn{1}{c}{0.74} & 0.77 \\
          \multicolumn{1}{c}{0.56$\pm0.01$} & \multicolumn{1}{c}{0.78$\pm0.01$} & \multicolumn{1}{c}{0.48$\pm0.02$} & 0.53$\pm0.01$ & \multicolumn{1}{c}{0.72$\pm0.02$} & \multicolumn{1}{c}{0.16$\pm0.01$} & \multicolumn{1}{c}{0.56$\pm0.02$} & 0.43$\pm0.02$ \\
          \multicolumn{1}{c}{0.70$\pm0.01$} & \multicolumn{1}{c}{0.80$\pm0.01$} & \multicolumn{1}{c}{0.71$\pm0.01$} & 0.75$\pm0.01$ & \multicolumn{1}{c}{0.80$\pm0.01$} & \multicolumn{1}{c}{0.16$\pm0.01$} & \multicolumn{1}{c}{0.63$\pm0.02$} & 0.65$\pm0.02$ \\
           \multicolumn{1}{c}{0.70$\pm0.02$} & \multicolumn{1}{c}{0.79$\pm0.02$} & \multicolumn{1}{c}{0.64$\pm0.01$}     & 0.72$\pm0.02$ & \multicolumn{1}{c}{0.58$\pm0.03$} & \multicolumn{1}{c}{0.13$\pm0.01$} & \multicolumn{1}{c}{0.46$\pm0.01$} & 0.31$\pm0.03$ \\\hline
        \end{tabular}
      }
  }
\end{table}

\subsection{$\beta-$VAE Baseline}
For all experiments in sections \ref{appendix:linear_mixing}, \ref{appendix:polynomial_mixing}, \ref{appendix:balls_dataset}, we implement a baseline based on $\beta-$VAE \citep{higgins2017betavae} and report the metrics in tables \ref{table:beta_linear_indep_scm}, \ref{table:beta_linear_dynamic_scm} for Linear Mixing, in tables \ref{table:beta_poly_indep_scm}, \ref{table:beta_poly_dynamic_scm} for Polynomial Mixing, and in tables \ref{table:beta_balls_indep_scm}, \ref{table:beta_balls_dynamic_scm} for the Balls image dataset. To obtain the scores for this baseline, we employ a similar 2 stage procedure, where in stage 1, we train an autoencoder with reconstruction objective only. Then at stage 2, we employ the KL divergence constraint from \citet{higgins2017betavae} on the representations obtained from the autoencoder at stage 1, and randomly divide the resulting $\hat{z}$ into two halves to represent $\hat{z}_{\hat{\mathcal{S}}}$,$\hat{z}_{\hat{\mathcal{U}}}$, and compute the $R^2$ scores against $z_{\mathcal{S}}$,$z_{\mathcal{U}}$. Note that unlike our method that directly affects a known subset of $\hat{z}$ to obtain $\hat{z}_{\hat{\mathcal{S}}}$, we have no way of knowing beforehand such subsets with the KL divergence penalty of \citet{higgins2017betavae}, hence the need for randomly selecting such features.

\paragraph{Training Details and Hyperparameter Selection}
\label{appendix: hp_search}
It should be noted that hyperparameter selection in unsupervised scenarios such as this work differs crucially from the conventional setups as in practice one does not have access to the ground-truth latents $z$. Therefore we focus on using default hyperparameters and demonstrate the robustness and versatility of our approach across the different datasets. We train all models with Adam \citep{Adam} optimizer with a learning rate of $10^{-3}$ without weight decay, $\epsilon=10^{-8}, \beta_1=0.9, \beta_2=0.999$. We reduce the learning rate by a factor of 0.5 if the training objective is not improved for 10 epochs. This drop is followed by a cool-down period of 10 epochs, and the learning rate cannot decrease to lower than $10^{-4}$. For all datasets we use a batch size of 1024 and early stop the training at 2000 steps. The weight of invariance penalty is always set to 1.0, regardless of the combination of penalties used. To ensure the robustness of the Min-Max penalty, we enforce the support invariance not just on the minimum and maximum across a batch, rather, we sort the batch and for each component of $z_{\mathcal{S}}$ take the top 10 for computing the penalty. For MMD penalty we always use the standard RBF kernel with a default bandwidth of 1.0, with the only exception of using an adaptive bandwidth for linear mixing experiments.

\end{document}